\documentclass[twoside,11pt]{article} 
\usepackage{jair, theapa, rawfonts}
\usepackage[T1]{fontenc} 
\usepackage{microtype}
\usepackage{latexsym,amsmath,amssymb}
\usepackage{pifont}
\usepackage{stmaryrd}
\usepackage{misc}
\usepackage{xspace}
\usepackage{xcolor}
\newcommand{\url}[1]{\texttt{#1}}
\usepackage{enumerate}
\usepackage{tikz}
\usetikzlibrary{arrows.meta,positioning}

\usepackage{thmtools,thm-restate}

\SetSymbolFont{stmry}{bold}{U}{stmry}{m}{n}

\newtheorem{theorem}{Theorem}
\newtheorem{definition}[theorem]{Definition}
\newtheorem{example}[theorem]{Example}
\newtheorem{proposition}[theorem]{Proposition}
\newtheorem{lemma}[theorem]{Lemma}
\title{%
  Conservative Extensions in \\ Horn Description Logics with Inverse Roles%
}
\author{%
  \name Jean Christoph Jung \email{jeanjung@uni-bremen.de} \\
  \name Carsten Lutz \email{clu@uni-bremen.de} \\
  \name Mauricio Martel \email{mauricio.martel@gmail.com} \\
  \name Thomas Schneider \email{thomas.schneider@uni-bremen.de} \\
  \addr Fachbereich 3 Mathematik/Informatik												\\
        Universit\"at Bremen																			\\
        Postfach 330\,440, 28334 Bremen, Germany
}

\ShortHeadings{Conservative Extensions in Horn DLs with Inverse Roles}
{Jung, Lutz, Martel \& Schneider}

\begin{document}

\maketitle

\begin{abstract}
  We investigate the decidability and computational complexity of
  conservative extensions and the related notions of inseparability
  and entailment in Horn description logics (DLs) with inverse roles.
  We consider both query conservative extensions, defined by requiring
  that the answers to all conjunctive queries are left unchanged, and
  deductive conservative extensions, which require that the entailed
  concept inclusions, role inclusions, and functionality assertions do
  not change. Upper bounds for query conservative extensions are
  particularly challenging because characterizations in terms of unbounded
  homomorphisms between universal models, which are the foundation of the standard
  approach to establishing decidability, fail in the presence of
  inverse roles. We resort to a characterization that carefully mixes
  unbounded and bounded homomorphisms and enables a decision procedure
  that combines tree automata and a mosaic technique. Our main results
  are that query conservative extensions are \TwoExpTime-complete in
  all DLs between \ELI and Horn-\ALCHIF and between Horn-\ALC and
  Horn-\ALCHIF, and that deductive conservative extensions are 
   \TwoExpTime-complete in
  all DLs between \ELI and $\ELHIFbot$. The same results hold
  for inseparability and entailment.
\end{abstract}

\section{Introduction}

When accessing incomplete data, it can be beneficial to add an ontology
formulated in a decription logic (DL) to specify relevant domain
knowledge, to assign a semantics to the data, and to enrich and unify
the vocabulary available for querying. The resulting framework is
known as ontology-mediated querying
\cite{DBLP:journals/tods/BienvenuCLW14,BO15} and, in a data
integration context, as ontology-based data access, OBDA
\cite{PLCD+08}.  Significant research activity on ontology-mediated
query evaluation has resulted in a thorough understanding of
computational complexity trade-offs and in various tools for
evaluating queries in practice, for a wide range of DLs such as
DL-Lite
\cite{CDL+07,DBLP:journals/jair/ArtaleCKZ09,DBLP:conf/kr/Rodriguez-MuroC12},
expressive Horn DLs such as $\mathcal{ELI}$ and Horn-$\mathcal{SHIQ}$
\cite{HMS07,EGOS08,EOS+12,TSCS15}, and expressive ``full Boolean'' DLs
such as \ALC and $\mathcal{SHIQ}$
\cite{DBLP:journals/jair/GlimmLHS08,DBLP:conf/cade/Lutz08,ZGN+15,DBLP:conf/kr/NgoOS16}.

While query evaluation is by now rather well-understood, there is a
need to develop reasoning services that aim at engineering ontologies
for ontology-mediated querying and support tasks such as ontology
versioning, ontology import, and module extraction, which means to
extract a subset from an ontology that is sufficient for the
application at hand
\cite{KLWW12,KLWW09,DBLP:journals/jair/GrauHKS08}. In all these
applications, it is important to relate different ontologies. In
versioning, for example, one would like to know whether replacing an
ontology with a new version has an effect on evaluating the relevant
queries. In ontology import, one wants to control the effect on query
evaluation of importing an existing ontology. And in module
extraction, one wants to know whether the module is really sufficient
to evaluate the queries from the application. All this can be
formalized by requiring that when exchanging an existing ontology with
a new one, the answers to the relevant queries do not change, over all
possible data sets \cite{KWZ10}. One arrives at notions of
``equivalence'' between ontologies that are different from logical
equivalence.

We now make this more precise. In DLs, ontologies are represented as a
\emph{TBox} while data sets are stored in an \emph{ABox}. A
\emph{signature} is a set of concept names and role names. We say that
a TBox $\Tmc_2 \supseteq \Tmc_1$ is an
\emph{$(\sigmaabox,\sigmaquery)$-query conservative extension} of a
TBox $\Tmc_1$, where $\sigmaabox$ and $\sigmaquery$ are signatures
relevant for the data and queries, respectively, if all
$\sigmaquery$-queries give the same answers w.r.t.\ $\Tmc_1$ and
w.r.t.\ $\Tmc_2$, for every $\sigmaabox$-ABox. We thus identify the
relevant queries by signature.  Note that the subset relationship
$\Tmc_2 \supseteq \Tmc_1$ is natural in some applications such as
ontology import and module extraction. It
is not natural in other applications such as versioning. In the
general case, when $\Tmc_1$ need not be a subset of $\Tmc_2$ and the
above condition is satisfied, we call $\Tmc_1$ and $\Tmc_2$
\emph{$(\sigmaabox,\sigmaquery)$-query inseparable}.  We also consider
the notion of query entailment: $\Tmc_1$
\emph{$(\sigmaabox,\sigmaquery)$-query entails} $\Tmc_2$ if all
$\sigmaquery$-queries give \emph{at least} the answers w.r.t.\
$\Tmc_1$ that they give w.r.t.\ $\Tmc_2$, over every $\sigmaabox$-ABox.
Clearly, query inseparability and conservative extensions are special
cases of query entailment: inseparability is bidirectional entailment and
conservative extensions are entailment under the promise that
\mbox{$\Tmc_1 \subseteq \Tmc_2$}.
When studying decidability or computational complexity, it thus
suffices to prove upper bounds for query entailment and lower bounds
for conservative extensions.  For the query language, we concentrate on
conjunctive queries (CQs); since we work with Horn DLs and quantify
over the queries, this is equivalent to using unions of CQs (UCQs) and
positive existential queries (PEQs).  Conservative extensions,
inseparability, query entailment of \emph{TBoxes}, as defined above,
are useful when knowledge is considered static and data changes
frequently.  Variants of these notions for \emph{knowledge bases
  (KBs)}, which consist of a TBox and an ABox, can be used for
applications with static data \cite{WWT+14,ABCR16}.

CQ entailment has been studied for various DLs
\cite{KPS+09,LW10,KLWW12,DBLP:journals/ai/BotoevaLRWZ19}, also in the KB version
\cite{BKR+16,DBLP:journals/ai/BotoevaLRWZ19}, and also for OBDA specifications that involve
mappings between data sources and the ABox \cite{BR15}, see also
the survey by~\citeA{BKL+16}.
Nevertheless, there is still a notable gap in our understanding of
this notion: query entailment between TBoxes is poorly understood in
Horn DLs with inverse roles, which are considered a crucial feature in
many applications. There is in fact a reason for this: as has already
been observed by Botoeva et al.~\citeyear{BKL+16,BKR+16}, standard techniques for Horn
DLs without inverse roles fail when inverse roles are added.  More
precisely, for Horn DLs without inverse roles query entailment can be
characterized by the existence of homomorphisms between universal
models \cite{LW10,BKL+16}. The resulting characterizations provide an
important foundation for decision procedures, often based on tree
automata emptiness \cite{BKL+16}. In the presence of inverse roles,
however, such characterizations are only correct if one resorts to
\emph{bounded} homomorphisms, that is, if one requires the existence
of an $n$-bounded homomorphism, \emph{for any} $n$
\cite{BKL+16,BKR+16}. The unbounded $n$ in $n$-bounded homomorphisms
corresponds to CQs of unbounded size that can be used for separating
the two TBoxes.  Unbounded homomorphisms, in contrast, correspond to
infinitary CQs, and while the (implicit) transition to infinitary CQs
poses no problems in DLs that do not admit inverse roles, it
compromises correctness in the presence of inverse roles. The
`tighter' characterization in terms of bounded homomorphism is
problematic because it is not obvious how the existence of such
(infinite families of) homomorphisms can be verified using tree
automata or related techniques and, consequently, decidability results
for query conservative extensions in Horn DLs with inverse roles are
difficult to obtain. In fact, the only result of which we are aware
concerns inseparability \emph{of KBs}, and it is proved using
intricate game-theoretic techniques~\cite{BKR+16}.

The aim of this article is to develop decision procedures for and study
the complexity of query conservative extensions, query inseparability,
and query entailment in Horn DLs with inverse roles such as \ELI and
Horn-\ALCHIF.  The main idea for establishing decidability is to provide a
very careful characterization that mixes unbounded and bounded
homomorphisms, pushing the use of bounded homomorphisms to only those
places where they cannot possibly be avoided. We can then deal with the
part of the characterization that uses unbounded homomorphisms using
tree automata while the part that uses bounded homomorphisms is
addressed up-front by precomputing relevant information using a mosaic
technique.  In this way, we establish decidability and a \TwoExpTime
upper bound for query entailment in the expressive Horn DL
Horn-\ALCHIF, and thus also for query inseparability and query
conservative extensions.  Together with known lower bounds~\cite{DBLP:journals/ai/BotoevaLRWZ19}, this yields \TwoExpTime-completeness for all DLs
between \ELI and Horn-\ALCHIF as well as between Horn-\ALC and
Horn-\ALCHIF.  To be more precise, the complexity is single
exponential in the size of $\Tmc_1$ and double exponential only
in the size of $\Tmc_2$. 

We additionally study the \emph{deductive} version of query
entailment, query inseparability, and query conservative extensions.
Here, instead of asking whether the answers to all queries are
preserved, the question is whether $\Tmc_1$ entails every concept
inclusion, role inclusion, and functionality assertion that $\Tmc_2$
entails, over a given signature \sig. This problem, too, has received
considerable interest in the literature, but has not previously been
studied for Horn DLs with inverse roles. It is more appropriate than
the query-based notions in applications that require conceptual
reasoning rather than querying data. It was historically even the
first notion of conservative extension studied for DLs~\cite{GLW06};
we again refer to the survey~\cite{BKL+16} for a more detailed
discussion.  We show that deductive entailment, deductive
inseparability, and deductive conservative extensions are
\TwoExpTime-complete for all DLs between \ELI and $\ELHIFbot$; again,
the runtime of our algorithm is single exponential in the size of
$\Tmc_1$ and double exponential in the size of $\Tmc_2$.  For the
upper bound, we first show that deductive entailment is essentially
identical to query entailment when the queries are tree-shaped CQs
with a single answer variable. We then characterize this version of
query entailment using simulations in place of homomorphisms; it
is not necessary to resort to a bounded version of
simulations. This again enables a decision procedure based on tree
automata emptiness. The lower bound is proved by reduction
from a certain homomorphism problem between universal models
for ABoxes w.r.t.\ \ELI TBoxes, studied in the context of querying
by example~\cite{GuJuSa-IJCAI18}.

This article is structured as follows:
in Section~\ref{sec:prelims} we define the notions used throughout the text;
in Section~\ref{sec:characterization} we establish the model-theoretic
characterizations;
in Section~\ref{sec:automata} we develop the automata-based decision procedure
and prove the exact complexity for query entailment, inseparability,
and conservative extensions;
Section~\ref{sec:lower_bounds} deals with the case of tree-shaped queries
and deductive query entailment;
Section~\ref{sec:concl} discusses possible future work.
Proofs of some auxiliary lemmas can be found in the appendix.


\section{Preliminaries}
\label{sec:prelims}

We define basic notions and lemmas that are needed in the remainder of the article.

\subsection{Description Logics}

The main DL considered in this article is Horn-\ALCHIF, a member of
the Horn-\SHIQ family of DLs whose reasoning problems have been widely
studied~\cite{HMS07,DBLP:journals/tocl/KrotzschRH13,EGOS08,Kaz09,ILS14,DBLP:journals/lmcs/LutzW17}. We
introduce Horn-\ALCHIF and several of its fragments.
%
 Let $\mn{N_C},\mn{N_R},\mn{N_I}$ be countably infinite sets of
 \emph{concept names}, \emph{role names}, and \emph{individual names}.
 A \emph{role} is either a role name $r$ or an \emph{inverse role}
 $r^-$ where $r$ is a role name.  As usual, we identify $(r^-)^-$ and $r$, allowing us to
 switch between role names and their inverses easily.  A
 \emph{concept inclusion (CI)} is of the form $L \sqsubseteq R$, where
 $L$ and $R$ are concepts defined by the syntax rules
%
\begin{align*} 
    R,R' &::= \top \mid \bot \mid A \mid \neg A \mid R \sqcap R' \mid 
\neg L \sqcup R \mid \exists r . R \mid
   \forall r . R \\[1mm]
    L,L' &::= \top \mid \bot \mid A \mid L \sqcap L' \mid L \sqcup L' \mid 
    \exists r . L 
\end{align*}
with $A$ ranging over concept names and $r$ over roles. 
%
A \emph{role inclusion (RI)} is of the form $r \sqsubseteq s$ with
$r,s$ roles, and a \emph{functionality assertion (FA)} is of the form
$\mn{func}(r)$ with $r$ a role. A \emph{Horn-\ALCHIF TBox} \Tmc is a
finite set of CIs, RIs, and FAs.
%
%
To avoid dealing with rather messy
technicalities that do neither seem to be very illuminating from a
theoretical viewpoint nor too useful from a practical one,
we generally assume that functional roles
cannot have subroles, that is, $r \sqsubseteq s \in \Tmc$ implies
$\mn{func}(s) \notin \Tmc$.\footnote{%
  Concerning the usefulness of allowing functional roles to have subroles,
  we found that only
  21 ($\leq{}$4.8\%) out of 439 available ontologies in BioPortal \cite{MP17}
  contain subroles of functional roles;
  many of these occurrences appear to be due to
  modeling mistakes.%
}
We conjecture that our main results also
hold without that restriction.

In this article, we consider
the following fragments of Horn-\ALCHIF. A
\emph{Horn-\ALC} TBox \Tmc is a finite set of CIs that do not use
inverse roles. An \emph{\ELIbot concept} is an expression that is
built according to the syntax rule for $L$ above, but does not
use~``$\sqcup$''.
An \emph{\ELIbot CI} is a CI of the form $L \sqsubseteq R$ where both $L$
and $R$ are \ELIbot concepts.  An \emph{\ELHIFbot TBox} is a finite
set of \ELIbot CIs, RIs, and FAs. 

An \emph{ABox} \Amc is a finite set of \emph{concept and role
  assertions} of the form $A(a)$ and $r(a,b)$, where $A \in \mn{N_C}$,
$r \in \mn{N_R}$ and $a,b \in \mn{N_I}$. We write $\mn{ind}(\Amc)$ for
the set of individuals in \Amc.  An ABox $\Amc$ is \emph{tree-shaped}
if (i) $\Amc$ does not contain an assertion of the form $r(a,a)$, (ii) the
undirected graph $G_\Amc = (\mn{ind}(\Amc),\,\{\{a,b\} \mid r(a,b) \in
\Amc\})$ is a tree, and (iii) there are no multi-edges, that is, for any
$a,b \in \mn{ind}(\Amc)$, \Amc contains at most one role assertion
that involves both $a$ and $b$.  

The semantics of Horn-\ALCHIF is defined in the usual
way~\cite{DL-Textbook}. An \emph{interpretation} is a pair $\Imc =
(\Delta^\Imc,\cdot^\Imc)$, where $\Delta^\Imc$ is a non-empty set, the
\emph{domain}, and $\cdot^\Imc$ is the
\emph{interpretation function}, with $A^\Imc \subseteq \Delta^\Imc$
for every $A \in \mn{N_C}$, $r^\Imc \subseteq \Delta^\Imc \times
\Delta^\Imc$ for every $r \in \mn{N_R}$, and $a^\Imc \in \Delta^\Imc$
for every $a \in \mn{N_I}$.  We make the standard name assumption,
that is, $a^\Imc=a$ for all $a \in \mn{N_I}$.  The interpretation
function is extended to inverse roles and arbitrary concepts $C$
in the usual way:
\[
  \begin{array}{@{}r@{~~}c@{~~}l@{\qquad}r@{~~}c@{~~}l@{\qquad}r@{~~}c@{~~}l@{}}
    (r^-)^\Imc         & = & \multicolumn{7}{@{}l}{\{(e,d) \mid (d,e) \in r^\Imc\}} \\[6pt]
    \top^\Imc          & = & \Delta^\Imc                                        \\[2pt]
    \bot^\Imc          & = & \emptyset                                          \\[2pt]
    (\lnot C)^\Imc     & = & \Delta^\Imc \setminus C^\Imc                       \\[2pt]
    (C \sqcap D)^\Imc  & = & C^\Imc \cap D^\Imc                                 \\[2pt]
    (C \sqcup D)^\Imc  & = & C^\Imc \cup D^\Imc                                 \\[2pt]
    (\exists r.C)^\Imc & = & \multicolumn{7}{@{}l}{\{d \mid \text{there is~} e \in \Delta^\Imc \text{~with~} (d,e) \in r^\Imc \text{~and~} e \in C^\Imc\}} \\[2pt]
    (\forall r.C)^\Imc & = & \multicolumn{7}{@{}l}{\{d \mid \text{for
        all~} e \in \Delta^\Imc, \text{~if~} (d,e) \in r^\Imc,
      \text{~then~} e \in C^\Imc\}}.
  \end{array}
\]
An interpretation \Imc satisfies
\begin{itemize}
  \item
    a CI $C \sqsubseteq D$, written $\Imc \models C \sqsubseteq D$, if $C^\Imc \subseteq D^\Imc$;
  \item
    an RI $r \sqsubseteq s$, written $\Imc \models r \sqsubseteq s$, if $r^\Imc \subseteq s^\Imc$;
  \item 
    an FA $\mn{func}(r)$, written $\Imc \models \mn{func}(r)$, if $r^\Imc$ is functional;
  \item
    a concept assertion $A(a)$, written $\Imc \models A(a)$, if $a \in A^\Imc$;
  \item
    a role assertion $r(a,b)$, written $\Imc \models r(a,b)$, if $(a,b) \in r^\Imc$.
\end{itemize}
$\Imc$ is 
a \emph{model} of a TBox \Tmc, written $\Imc \models \Tmc$,
if it satisfies all inclusions and
assertions in it, likewise for ABoxes.
 \Amc is \emph{consistent} with \Tmc
if \Tmc and \Amc have a common model.

Tree-shaped interpretations are defined by analogy with tree-shaped
ABoxes. We additionally need a weaker variant that permits
multi-edges: an interpretation $\Imc$ is \emph{weakly tree-shaped} if
there are no $d,r$ with $(d,d) \in r^\Imc$ and the undirected graph
$(\Delta^\Imc,\,\{\{d,e\} \mid (d,e) \in r^\Imc\})$ is a tree. 
Thus, an interpretation $\Imc$ is \emph{tree-shaped} if it is
weakly tree-shaped and has no multi-edges, i.e., $r^\Imc \cap s^\Imc =
\emptyset$ for any two distinct roles $r,s$.

A \emph{signature} $\sig$ is a set of concept and role names. An
\emph{$\sig$-ABox} is an ABox that uses only concept and role names from
$\sig$, and \emph{$\sig$-\ELIbot concepts} and other
syntactic objects are defined analogously.

Generally and without further notice, we work with Horn-\ALCHIF
TBoxes that are in a certain nesting-free \emph{normal form}, that is, they
contain only CIs of the form
\[
  \top \sqsubseteq A
  \qquad
  A \sqsubseteq \bot
  \qquad
  A_1 \sqcap A_2 \sqsubseteq B
  \qquad
  A \sqsubseteq \exists r.B
  \qquad
  A \sqsubseteq \forall r.B,
\]
where $A,B,A_1,A_2$ are concept names and $r,s$ are roles. It is
well-known that every Horn-\ALCHIF TBox $\Tmc$ can be converted
into a TBox $\Tmc'$ in normal form (introducing additional concept
names) such that $\Tmc$ is a logical consequence of $\Tmc'$ and every
model of $\Tmc$ can be extended to one of $\Tmc'$ by interpreting the
additional concept names \cite<see, e.g.,>{BHLW16}. As a consequence,
all results obtained in this article for TBoxes in normal form lift to
the general case. Note that a Horn-\ALCHIF TBox in normal form is 
essentially an \ELIbot TBox since CIs of the form $A\sqsubseteq \forall r.B$
can be equivalently rewritten as $\exists r^-.A \sqsubseteq B$.


\subsection{Query Conservative Extensions and Entailment}
\label{sec:query_entailment}

A \emph{conjunctive query (CQ)} is of the form
$
  q(\xbf) = \exists \ybf\,\varphi(\xbf,\ybf),
$
where \xbf and \ybf are tuples of variables and
$\varphi(\xbf,\ybf)$ is a conjunction of \emph{atoms} of the form
$A(z)$ or $r(z,z')$ with
$A \in \mn{N_C}$, $r \in \mn{N_R}$, and
$z,z' \in \xbf \cup \ybf$.
We call \xbf \emph{answer variables}
and \ybf \emph{quantified variables} of $q$.
Tree-shaped and weakly tree-shaped CQs are defined by analogy with (weakly) tree-shaped ABoxes and interpretations:
a CQ $q$ 
is \emph{weakly tree-shaped} if it does not contain atoms of the
  form $r(z,z)$ and the undirected graph $(\xbf \cup \ybf,~\{\{z,z'\}
\mid r(z,z') \text{~is an atom in~} q\})$ is a tree.
A weakly tree-shaped CQ $q$
is \emph{tree-shaped} or a \emph{tCQ} if $q$ has no multi-edges,
that is, for any two distinct variables $z,z'$ in $q$, there is
at most one atom of the form $r(z,z')$ or $r(z',z)$ in~$q$.
A \emph{\onetCQ} is a tCQ with exactly one answer variable.
We sometimes write $r^-(z,z') \in q$ to mean $r(z',z) \in q$.

A \emph{match} of $q$ in an interpretation \Imc is a function
$\pi : \xbf \cup \ybf \to \Delta^\Imc$ such that $\pi(z) \in A^\Imc$
for every atom $A(z)$ of $q$ and $(\pi(z),\pi(z')) \in r^\Imc$ for every
atom $r(z,z')$ of $q$
.  We write
$\Imc \models q(a_1,\dots,a_n)$ if there is a match of $q$ in $\Imc$
with $\pi(x_i) = a_i$ for all $i < n$.  A tuple
$\abf$ of elements from $\mn{N_I}$ is a
\emph{certain answer} to $q$ over an ABox \Amc given a TBox \Tmc,
written $\Tmc,\Amc \models q(\abf)$, if $\Imc \models q(\abf)$ for
all models of \Tmc and \Amc.
\begin{definition}
  \label{def:entailment}
  Let $\sigmaabox,\sigmaquery$ be signatures
  and $\Tmc_1,\Tmc_2$ Horn-\ALCHIF TBoxes.
  We say that \emph{$\Tmc_1$ $(\sigmaabox,\sigmaquery)$-CQ entails $\Tmc_2$},
  written $\Tmc_1 \models_{\sigmaabox,\sigmaquery}^{\textup{CQ}} \Tmc_2$,
  if for all $\sigmaabox$-ABoxes \Amc
  consistent with $\Tmc_1$ and $\Tmc_2$,
  all $\sigmaquery$-CQs $q(\xbf)$
  and all tuples $\abf \subseteq \mn{ind}(\Amc)$,
  $\Tmc_2,\Amc \models q(\abf)$ implies
  $\Tmc_1,\Amc \models q(\abf)$.  If in addition
  $\Tmc_1 \subseteq \Tmc_2$, we say that \emph{$\Tmc_2$ is an
    $(\sigmaabox,\sigmaquery)$-CQ conservative extension of
    $\Tmc_1$}. If
  $\Tmc_1 \models_{\sigmaabox,\sigmaquery}^{\textup{CQ}} \Tmc_2$ and
  vice versa, then $\Tmc_1$ and $\Tmc_2$ are
  \emph{$(\sigmaabox,\sigmaquery)$-CQ inseparable}.
\end{definition}
%
We also consider
\emph{$(\sigmaabox,\sigmaquery)$-\onetCQ entailment}, denoted
$\models_{\sigmaabox,\sigmaquery}^{\textup{\onetCQ}}$,
\emph{$(\sigmaabox,\sigmaquery)$-\onetCQ conservative extensions},
and \emph{$(\sigmaabox,\sigmaquery)$-\onetCQ inseparability},
defined in the
obvious way by replacing CQs with \onetCQs.

If $\Tmc_1 \not\models_{\sigmaabox,\sigmaquery}^{\text{CQ}} \Tmc_2$
because $\Tmc_2,\Amc \models q(\abf)$ but
$\Tmc_1,\Amc \not\models q(\abf)$ for some $\sigmaabox$-ABox \Amc
consistent with both $\Tmc_i$, $\sigmaquery$-CQ $q(\xbf)$ and
tuple $\abf$ over $\mn{ind}(\Amc)$, we call the triple $(\Amc, q, \abf)$ a \emph{witness} to
non-entailment.
\begin{example}
  Let $\Tmc_1=\{\text{PhDStud}\sqsubseteq \exists
  \text{advBy}.\text{Prof}, \text{adv}\sqsubseteq\text{advBy}^-\}$ and
  $\Tmc_2=\Tmc_1\cup\{\mn{func}(\text{advBy})\}$,
  $\sigmaabox=\{\text{PhDStud},\text{adv}\}$ and $\sigmaquery=\{\text{Prof}\}$. Then
  we have
   $\Tmc_1\not\models^{\textup{CQ}}_{\sigmaabox,\sigmaquery} \Tmc_2$ because of
  the witness
  $(\{\text{PhDStud}(\text{john}),\text{adv}(\text{mary},\text{john})\},
  \text{Prof}(x), \text{mary})$.
\end{example}
If we drop from Definition~\ref{def:entailment} the condition that
$\Amc$ must be consistent with both $\Tmc_1$ and $\Tmc_2$, then we
obtain an alternative notion of CQ entailment that we call
\emph{CQ entailment with inconsistent ABoxes}.  While CQ
entailment with inconsistent ABoxes
trivially implies CQ entailment in the original sense, the
converse fails. 
\begin{example}
  \label{exa:incons}
  Let $\Tmc_1 = \emptyset$,
  $\Tmc_2 = \{A_1 \sqcap A_2 \sqsubseteq \bot\}$ and $\sigmaabox = \{
  A_1,A_2\}$, $\sigmaquery = \{B\}$. Then
  $\Tmc_1 \models_{\sigmaabox,\sigmaquery}^{\textup{CQ}} \Tmc_2$
  because none of the $\Tmc_i$ uses the concept name $B$;
  however,
  $\Tmc_1$ does not $(\sigmaabox,\sigmaquery)$-CQ entail $\Tmc_2$
  with inconsistent ABoxes because of the witness
  $(\{A_1(a),A_2(a)\},\,B(x),\,a)$.
\end{example}
%
%
%
The following lemma relates the two notions of CQ entailment.
\emph{CQ evaluation} is the problem to decide, given a TBox \Tmc, an
ABox \Amc, a CQ~$q$, and a tuple $\abf \in \mn{ind}(\Amc)$, whether
$\Tmc,\Amc \models q(\abf)$.
\begin{proposition}
  \label{lem:inconsistent_ABoxes}
  $(\sigmaabox,\sigmaquery)$-CQ entailment with inconsistent ABoxes
  can be decided in polynomial time
  given access to oracles deciding
  $(\sigmaabox,\sigmaquery)$-CQ entailment 
and CQ evaluation.
\end{proposition}
Consequently and since CQ evaluation is in \ExpTime in Horn-\ALCHIF
\cite{EGOS08}, all complexity results obtained in this article also
apply to CQ entailment with inconsistent ABoxes.

To prove Proposition~\ref{lem:inconsistent_ABoxes}, we introduce a
notion of entailment that refers to the (in)con\-sis\-ten\-cy of ABoxes:
\emph{$\Tmc_1$ $\sigmaabox$-inconsistency entails $\Tmc_2$}, written
$\Tmc_1 \models_{\sigmaabox}^{\bot} \Tmc_2$, if for all
$\sigmaabox$-ABoxes \Amc: if $\Amc$ is inconsistent with $\Tmc_2$,
then $\Amc$ is inconsistent with $\Tmc_1$.  Also, given a TBox \Tmc
and signatures $\sigmaabox,\sigmaquery$, we say that \Tmc is
\emph{$(\sigmaabox, \sigmaquery)$-universal} if $\Tmc,\Amc \models
q(\abf)$ for all $\sigmaabox$-ABoxes $\Amc$, $\sigmaquery$-CQs
$q(\xbf)$, and tuples $\abf$ over $\mn{ind}(\Amc)$ that are of the
same length as $\xbf$. The following is proved in
Appendix~\ref{appx:inconsistent_ABoxes}.
\begin{restatable}{lemma}{leminconsistentABoxesaux}
  \label{lem:inconsistent_ABoxes_aux}
  Let $\Tmc_1$ and $\Tmc_2$ be $\HornALCHIF$ TBoxes and let
  $\sigmaabox,\sigmaquery$ be signatures.
  Then 
  $\Tmc_1$ $(\sigmaabox,\sigmaquery)$-CQ entails $\Tmc_2$ with inconsistent ABoxes
  iff one of the two following conditions holds.
  \begin{enumerate}
    \item[(1)]
      $\Tmc_1 \models_{\sigmaabox,\sigmaquery}^{\textup{CQ}} \Tmc_2$
      ~and~ $\Tmc_1 \models_{\sigmaabox}^{\bot} \Tmc_2$;
    \item[(2)]
      $\Tmc_1$ is $(\sigmaabox, \sigmaquery)$-universal.
  \end{enumerate}
\end{restatable}
\par\noindent
Lemma~\ref{lem:inconsistent_ABoxes_aux} is the core ingredient to
proving Proposition~\ref{lem:inconsistent_ABoxes}. We first note that
whether $\Tmc_1$ is $(\sigmaabox, \sigmaquery)$-universal can be
decided by polynomially many CQ evaluation checks. In fact, it is not
hard to verify that $\Tmc_1$ is $(\sigmaabox, \sigmaquery)$-universal
iff \Qbf contains no role names\footnote{If there is an $r \in \Qbf$,
  then any disconnected ABox \Amc witnesses that $\Tmc_1$ is not
  $(\sigmaabox, \sigmaquery)$-universal since
  $\Tmc_1,\Amc \not\models r(a,b)$ whenever $a$ and $b$ are from
  different maximal connected components of \Amc.}  and
$\Tmc_1,\Amc \models B(a)$ for all $\sigmaabox$-ABoxes \Amc that
contain a single assertion, all $\sigmaquery$-CQs $B(x)$, and all
$a \in \mn{ind}(\Amc)$. Moreover,
$\Tmc_1 \models_{\sigmaabox}^{\bot} \Tmc_2$ iff the following
conditions are satisfied:
\begin{enumerate}

\item 
$\Tmc_1^A \models_{\sigmaabox,\{A\}}^{\textup{CQ}}
\Tmc_2^A$, where $A$ is a fresh concept name and $\Tmc_i^A$ is
obtained from $\Tmc_i$ by replacing each occurrence of $\bot$ with $A$
and adding the axioms $A \sqsubseteq \forall s.A$ and $A\sqsubseteq
\forall s^-.A$ for every role $s$ that occurs in $\Tmc_i$, for
$i=1,2$;

\item every $\sigmaabox$-ABox $\Amc = \{r(a,b),r(a,c)\}$
  inconsistent with $\Tmc_2$ is also inconsistent with $\Tmc_1$.

\end{enumerate}
This is proved in detail in Appendix~\ref{appx:inconsistentcy_entailment}.
In summary, we obtain Proposition~\ref{lem:inconsistent_ABoxes}.

\smallskip

We close this section with a general remark on the impact of
role inclusions on $(\Abf,\Qbf)$-entailment.
It is easy to see that
$\Tmc_1 \not\models_{\sigmaabox,\sigmaquery}^{\textup{CQ}} \Tmc_2$ if
there is an $\sigmaabox$-role $r$ and a $\sigmaquery$-role $s$ with
$\Tmc_2 \models r \sqsubseteq s$ but
$\Tmc_1 \not\models r \sqsubseteq s$.  We write
$\Tmc_1 \models_{\sigmaabox,\sigmaquery}^{\text{RI}} \Tmc_2$ if
there are no such $r$ and $s$. Clearly,
$\Tmc_1 \models_{\sigmaabox,\sigmaquery}^{\text{RI}} \Tmc_2$ can be
decided in exponential time: for all $\Omc(|\sigmaabox| \cdot |\sigmaquery|)$
many pairs of roles, subsumption w.r.t.\ both $\Tmc_i$ needs to be tested,
and each such test can be reduced to concept subsumption \cite{HP04},
which is in \ExpTime for the extension \SHIQ of \HornALCHIF
\cite{Tob01}.
It is thus safe to assume
$\Tmc_1 \models_{\sigmaabox,\sigmaquery}^{\text{RI}} \Tmc_2$ when
deciding
CQ entailment, which we will generally do from now on to 
avoid
dealing
with special cases. 

\subsection{Deductive Conservative Extensions}
\label{sec:dCEs}

Another natural notion of entailment that has received significant
attention in the literature is deductive entailment, which generalizes
the notion of deductive conservative extensions, 
and which separates two TBoxes in terms of concept and role inclusions
and functionality assertions, instead of
ABoxes and queries \cite{GLW06,LWW07,KLWW09,LW10}.
\begin{definition}
  \label{def:deductive_entailment}
  Let $\sig$ be a signature and let $\Tmc_1$ and $\Tmc_2$ be \ELHIFbot TBoxes.
  We say that \emph{$\Tmc_1$ $\sig$-deductively entails $\Tmc_2$},
  written $\Tmc_1 \models_{\sig}^{\textup{TBox}} \Tmc_2$, if for all
  $\sig$-\ELIbot-concept inclusions $\alpha$ and all
  $\sig$-RIs and $\sig$-FAs $\alpha$:
  $\Tmc_2 \models \alpha$ implies $\Tmc_1 \models \alpha$. 
  If additionally $\Tmc_1 \subseteq \Tmc_2$,
  then we say that \emph{$\Tmc_2$ is an
  $\sig$-deductive conservative extension of $\Tmc_1$}.
  If $\Tmc_1 \models_{\sig}^{\textup{TBox}} \Tmc_2$
  and vice versa, then $\Tmc_1$ and $\Tmc_2$ are
  \emph{$\sig$-deductively inseparable}.
\end{definition}
Deductive conservative extensions in other DLs is defined accordingly,
that is, separation takes place in terms of the TBox statements that
are admitted in the DL under consideration.  In contrast to
$(\sigmaabox,\sigmaquery)$-query entailment, only one signature is
relevant for $\sig$-deductive entailment (unless one would want to
distinguish the signatures of the two sides in CIs or RIs, which seems
unintuitive).  Although $\sig$-deductive entailment and
$(\sig,\sig)$-query entailment are closely related, it is not
difficult to see that they are orthogonal.
%
%
%
%
\begin{example}
  \label{exa:deductive_entailment_1}
  Let $\Tmc_1$, $\Tmc_2$ be as in Example~\ref{exa:incons} and
  $\sig = \{A_1,A_2,B\}$.
  Then $\Tmc_1 \not\models_{\sig}^{\textup{TBox}} \Tmc_2$,
  witnessed by the \ELI CI $A_1 \sqcap A_2 \sqsubseteq B$.
  However, 
  $\Tmc_1 \models_{\sig,\sig}^{\textup{CQ}} \Tmc_2$
  because $\sig$-CQs cannot detect
  the disjointness of $A_1$ and $A_2$ based on 
  ABoxes that are consistent with $\Tmc_2$.
  
  \par\smallskip
  For the converse direction, let $\Tmc_1 = \emptyset$ and
  $\Tmc_2 = \{ A \sqsubseteq \exists r . B\}$, and $\sig=\{A,B\}$.
  Then $\Tmc_1 \not\models_{\sig,\sig}^{\textup{CQ}} \Tmc_2$ is
  witnessed by $(\{A(a)\},\exists x \, B(x),a)$.  However, since
  \ELIbot CIs cannot jump to an unreachable point like the 
  existential quantifier in the CQ $\exists x \, B(x)$, we have 
  $\Tmc_1 \models_{\sig}^{\textup{TBox}} \Tmc_2$.
\end{example}
For $\sig$-deductive entailment and $(\sig,\sig)$-\emph{\onetCQ}
entailment, the connection is even more intimate. In fact, we have the
following.
%
%
\begin{proposition}
  \label{lem:deductive_versus_query_entailment}
  In \ELHIFbot, $\sig$-deductive entailment can be decided
  in polynomial time given access to oracles for $(\sig,\sig)$-\onetCQ
  entailment and \onetCQ evaluation.
%
\end{proposition}
In fact, Proposition~\ref{lem:deductive_versus_query_entailment} is a
consequence of the following lemma and the observations about 
$\Tmc_1 \models_\sig^\bot \Tmc_2$ made in
Section~\ref{sec:query_entailment}.
A proof is in Appendix~\ref{appx:deductive_versus_query_entailment}.
\begin{restatable}{lemma}{myRest}
  \label{lem:deductive_versus_query_entailment_aux}
  Let $\sig$ be a signature and
  $\Tmc_1,\Tmc_2$ \ELHIFbot TBoxes
  such that $\Tmc_1 \models_{\sig,\sig}^{\textup{RI}} \Tmc_2$.
  Then
  \[
    \Tmc_1 \models_{\sig}^{\textup{TBox}} \Tmc_2
    \quad\text{iff}\quad
    \Tmc_1 \models_{\sig,\sig}^{\textup{\onetCQ}} \Tmc_2
    \quad\text{and}\quad
    \Tmc_1 \models_\sig^\bot \Tmc_2.
  \]
\end{restatable}
Note that Definition~\ref{def:deductive_entailment} assumes the TBoxes
$\Tmc_1$ and $\Tmc_2$ to be formulated in $\ELHIFbot$ rather than in
\HornALCHIF. This is because the main theme of this article are query
conservative extensions and, as demonstrated by
Proposition~\ref{lem:deductive_versus_query_entailment}, separation in
terms of \ELIbot concept inclusions is closely related to separation
by (tree-shaped) ABoxes and 1tCQs. Now, separation in terms of \ELIbot
concept inclusions is the natural choice for \ELHIFbot TBoxes as
\ELHIFbot is based on the concept language \ELIbot while this is not
true for \HornALCHIF. All results in this article that concern deductive
conservative extensions (and related notions) actually remain valid
also when $\Tmc_1$ and~$\Tmc_2$ are \HornALCHIF TBoxes, while
separation is still in terms of \ELIbot CIs. Separating \HornALCHIF
TBoxes in terms of \HornALCHIF is an interesting topic that is outside
the scope of this article, see Section~\ref{sec:concl} for some additional
discussion.

\subsection{Homomorphisms and Simulations}

Homomorphisms, both in their (standard) unbounded form and in bounded
form, are an elementary tool for dealing with CQs. Likewise,
(unbounded) simulations are closely linked to {\onetCQ}s. We will make
extensive use of these notions throughout the article.

For interpretations $\Imc_1,\Imc_2$ and a signature $\sig$, an
\emph{$\sig$-homomorphism} from $\Imc_1$ to $\Imc_2$ is a total
function $h : \Delta^{\Imc_1} \to \Delta^{\Imc_2}$ that
satisfies the following conditions.
\begin{enumerate}
  \item[(1)]
$h(a) = a$ for all $a \in \mn{N_I}$;
 \item[(2)]
$h(d) \in A^{\Imc_2}$ for all $d \in A^{\Imc_1}$, $A \in \mn{N_C}
\cap \sig$;
  \item[(3)]
$(h(d),h(d')) \in r^{\Imc_2}$ for all $(d,d') \in r^{\Imc_1}$, 
$r \in \mn{N_R} \cap \sig$.
%
 \end{enumerate}
%
We write 
$\Imc_1 \to_\sig \Imc_2$ to denote the existence of an $\sig$-homomorphism from $\Imc_1$ to $\Imc_2$.
If $\sig = \mn{N_C}\cup\mn{N_R}$,
we write $\Imc\to\Jmc$.

Let $\Imc_1,\Imc_2$ be interpretations, $d \in \Delta^{\Imc_1}$, and
$n \geq 0$.
We say that there is an \emph{$n$-bounded $\sig$-homomorphism} from $\Imc_1$ to
$\Imc_2$, written $\Imc_1 \rightarrow^n_\sig \Imc_2$, if for any
subinterpretation\footnote{We mean an interpretation $\Imc'_1$ such
  that $\Delta^{\Imc'_1} \subseteq \Delta^\Imc$, $A^{\Imc'_1} = A^\Imc
  \cap
\Delta^{\Imc'_1}$ for all concept names $A$, and $r^{\Imc'_1} = r^\Imc
  \cap
(\Delta^{\Imc'_1} \times \Delta^{\Imc'_1})$ for all role names $r$.}
$\Imc'_1$ of $\Imc_1$ with
$|\Delta^{\Imc'_1}| \leq n$, we have
$\Imc'_1 \rightarrow_\sig \Imc_2$. Moreover, we write
$\Imc_1 \rightarrow^\mn{fin}_\sig \Imc_2$ if
$\Imc_1 \rightarrow^n_\sig \Imc_2$ for every $n$. 
If $\sig = \mn{N_C}\cup\mn{N_R}$,
we write
$\Imc\to^\mn{fin}\Jmc$.

For $d\in\Delta^\Imc$ and $n\geq 0$, we denote with $\Imc|^d_n$ the
restriction of interpretation $\Imc$ to elements that can be reached
by starting at $d$ and traveling along at most $n$ role edges
(forwards or backwards).  The following is standard to
prove. 
\begin{lemma}
  \label{lem:alternative_def_bounded_hmph} 
  Let $\sig$ be a signature and $\Imc_1,\Imc_2$ be interpretations.
  \begin{enumerate}
    \item
      If $\Imc_1 \rightarrow_\sig \Imc_2$, then for all $\sig$-CQs $q$
      and tuples $\abf$, $\Imc_1 \models q(\abf)$ implies $\Imc_2 \models q(\abf)$.
    \item
      For every $n \geq 0$, if $\Imc_1 \rightarrow_\sig^n \Imc_2$, then for all $\sig$-CQs $q$
      with at most $n$ variables and all tuples $\abf$,
      $\Imc_1 \models q(\abf)$ implies $\Imc_2 \models q(\abf)$.
    \item
      If $\Imc_1$ is finitely branching, then the following are equivalent:
      \begin{enumerate}
        \item 
          $\Imc_1 \rightarrow_\sig^{\mn{fin}} \Imc_2$;
        \item
          $\Imc_1|^d_i \rightarrow_\sig \Imc_2$
          for every $d \in \Delta^{\Imc_1}$ and $i \geq 0$.
      \end{enumerate}       
  \end{enumerate}
\end{lemma} 
%
Given a signature $\sig$ and two interpretations $\Imc,\Jmc$, an
\emph{$\sig$-simulation of $\Imc$ in $\Jmc$} is a relation
$\sigma \subseteq \Delta^{\Imc} \times \Delta^{\Jmc}$ that satisfies the following conditions.
%
%
\begin{itemize}

\item[(1)] $(a,a) \in \sigma$ for all $a \in \mn{N_I}$;

\item[(2)] if $d \in A^{\Imc}$ with $A \in \sig$ and $(d,e) \in \sigma$,
    then $e \in A^{\Jmc}$;

\item[(3)] if $(d,d') \in r^{\Imc}$ with $r$ an $\sig$-role and $(d,e) \in \sigma$,
    then there is some $e'$ with $(e,e') \in r^{\Jmc}$ and $(d',e') \in \sigma$.

\end{itemize}
We write $\Imc \preceq_\sig \Jmc$ if there is an $\sig$-simulation of
$\Imc$ in $\Jmc$. Point~1 of
Lemma~\ref{lem:alternative_def_bounded_hmph} still holds when 
homomorphisms are replaced with simulations and CQs with
\onetCQs.

\subsection{The Universal Model}
\label{sec:universal_model}

We next introduce 
universal
models, whose existence is a distinguishing feature of Horn logics.
More precisely, for every \HornALCHIF TBox \Tmc and ABox \Amc, there
is a model $\Imc_{\Tmc,\Amc}$ of \Tmc and \Amc that homomorphically
embeds into every model of \Tmc and \Amc.  This last property,
together with Point~1 of Lemma~\ref{lem:alternative_def_bounded_hmph},
ensures that $\Imc_{\Tmc,\Amc}$ is universal in the sense that the
certain answers to any CQ  $q$ over \Amc given \Tmc can be obtained
by evaluating $q$ over~$\Imc_{\Tmc,\Amc}$.  We define the universal
model by extending the ABox in a forward chaining way akin to the
chase in database theory.

Let \Tmc be a $\HornALCHIF$ TBox in normal form and \Amc an ABox that
is consistent with~\Tmc. 
For a set $t$ of concept names, we write $\bigsqcap t$
  as a shorthand for $\bigsqcap_{A\in t}A$.
A \Tmc-\emph{type} is a set $t$ of concept
names that occur in \Tmc such that
$\Tmc \models \bigsqcap t \sqsubseteq A$ implies $A \in t$ for all
concept names~$A$.  For $a \in \mn{ind}(\Amc)$ and with \Tmc
understood, let $\mn{tp}_\Tmc(a)$ denote the $\Tmc$-type
$\{A \mid \Tmc,\Amc \models A(a)\}$.  When $t,t'$ are \Tmc-types, $r$
is a role, and $a \in \mn{ind}(\Amc)$, we write
\begin{itemize} 
  
  \item $t \rightsquigarrow^\Tmc_r t'$ if $\Tmc \models \bigsqcap t
    \sqsubseteq \exists r. \bigsqcap t'$ and $t'$ is maximal with
    this condition;

  \item $a \rightsquigarrow^{\Tmc,\Amc}_r t$ if $\Tmc,\Amc \models
    \exists r . \bigsqcap t(a)$ and $t$ is maximal with this condition.

\end{itemize}
The relations $\rightsquigarrow^\Tmc_r$ and
$\rightsquigarrow_{r}^{\Tmc,\Amc}$ can be computed in exponential
time: for every $t,t'$ (resp.\ $a,t$), the above entailment test
requires at most exponential time since subsumption (resp.\ instance
checking) in \SHIQ is \ExpTime-complete~\cite{Tob01}.  A \emph{path
  for \Amc and \Tmc} is a finite sequence $\pi=a r_0 t_1
\cdots\linebreak[2] t_{n-1} r_{n-1} t_n$, $n \geq 0$, with $a \in
\mn{ind}(\Amc)$, $r_0,\dots,r_{n-1}$ roles, and $t_1,\dots,t_n$
\Tmc-types such that
\begin{enumerate}

  \item[(i)] $a\rightsquigarrow^{\Tmc,\Amc}_{r_0} t_1$ and, if
    $\mn{func}(r_0)\in\Tmc$, then there is no $b\in\mn{ind}(\Amc)$ such
    that $\Tmc,\Amc \models r_0(a,b)$;

  \item[(ii)] $t_i
    \rightsquigarrow^\Tmc_{r_i}
    t_{i+1}$ and, if $\mn{func}(r_i)\in\Tmc$, then $r_{i-1}\neq r_i^-$,
    for $1 \leq i < n$.
    
\end{enumerate}
When $n > 0$, we use $\mn{tail}(\pi)$ to denote
$t_n$. Let $\mn{Paths}$ be the set of all paths for \Amc and \Tmc.
The \emph{universal model} $\Imc_{\Tmc,\Amc}$ of \Tmc and \Amc
is defined as follows:
\begin{align*}
  \Delta^{\Imc_{\Tmc,\Amc}} & = \mn{Paths} \\
  A^{\Imc_{\Tmc,\Amc}} & = \{ a \in \mn{ind}(\Amc) \mid \Tmc,\Amc
  \models A(a) \} \; \cup \;
  \{ \pi \in \Delta^\Imc \setminus \mn{ind}(\Amc) \mid \Tmc \models
  \bigsqcap \mn{tail}(\pi) \sqsubseteq A \} \displaybreak[2]\\
  r^{\Imc_{\Tmc,\Amc}} & = \{ (a,b) \in \mn{ind}(\Amc) \times \mn{ind}(\Amc) \mid
  s(a,b)\in \Amc \text{~and~} \Tmc\models s\sqsubseteq r\} \; \cup \\
  &\phantom{{}={}} \{ (\pi,\pi s t)  \mid \pi s t \in \mn{Paths}
  \text{~and~} \Tmc\models s\sqsubseteq r \} \; \cup \\[-1pt]
  &\phantom{{}={}} \{ (\pi s t,\pi) \mid \pi s t \in \mn{Paths}
  \text{~and~} \Tmc\models s^-\sqsubseteq r \}.
\end{align*}
We also need a universal model $\Imc_{\Tmc,t}$ of a TBox \Tmc and a type $t$,
instead of an ABox. More precisely,
we define $\Imc_{\Tmc,t} = \Imc_{\Tmc,\Amc_t}$
where $\Amc_t = \{A(a) \mid A \in t\}$ for a fixed $a \in \mn{N_I}$.

The following lemma summarizes the main properties of universal
models. The proof is standard and omitted \cite<see, e.g.,>{BKR+16}.
\begin{lemma} \label{lem:universal_model}
  For every $\HornALCHIF$ TBox \Tmc in normal form and ABox \Amc consistent with~\Tmc, the
  following hold: 

  \begin{enumerate}[(1)]

    \item
      \label{it:univ_model_is_a_model}
      $\Imc_{\Tmc,\Amc}$ is a model of $\Tmc$  and $\Amc$;

    \item
      \label{it:univ_model_embeds}
      $\Imc_{\Tmc,\Amc} \rightarrow \Imc$
      for all models $\Imc$ of $\Tmc$ and $\Amc$;

    \item
      \label{it:univ_model_queries}
      $\Tmc,\Amc \models q(\mathbf{a})$ iff $\Imc_{\Tmc,\Amc} \models q(\mathbf{a})$,
      for all CQs $q(\xbf)$ and tuples $\mathbf{a}$ of individuals;

    \item
      \label{it:univ_model_regularity_supertypes}
      $\Imc_{\Tmc,t} \to \Imc_{\Tmc,t'}$ for all \Tmc-types $t,t'$ with $t \subseteq t'$.

  \end{enumerate}
\end{lemma}
We shall sometimes refer to subinterpretations of universal models. 
Given a TBox \Tmc, an ABox \Amc, and $a \in \mn{ind}(\Amc)$, we
use $\Imc_{\Tmc,\Amc}|_a$ to denote the subinterpretation of $\Imc_{\Tmc,\Amc}$
rooted at~$a$, i.e., the restriction of $\Imc_{\Tmc,\Amc}$ to
all paths $\pi$ beginning with $a$.
Clearly, $\Imc_{\Tmc,\Amc}|_a$ is weakly tree-shaped.

\section{Model-Theoretic Characterization} 
\label{sec:characterization}

We aim to provide a model-theoretic characterization of CQ-entailment
that will be the basis for our decision procedure later on.  We
proceed in two steps. We first show that non-entailment between
$\Tmc_1$ and $\Tmc_2$ is always witnessed by a tree-shaped ABox and
a weakly tree-shaped CQ with at most one answer variable.  Second, we
provide the desired characterization, using a careful mix of bounded
and unbounded homomorphisms. We also establish a version for
1tCQ-entailment that uses (unbounded) simulations.


\subsection{Tree-Shaped Witnesses}
\label{sec:tree_shaped_witnesses}


We show that non-entailment between $\Tmc_1$ and $\Tmc_2$ is always
witnessed by a tree-shaped ABox and a weakly tree-shaped CQ with at
most one answer variable. The idea is to first manipulate the
witnessing CQ so that it takes the desired form and to then unravel
the ABox.

We start with defining ABox unraveling \cite<see
also>{DBLP:journals/lmcs/LutzW17}.  The \emph{unraveling} $U_\Amc^a$
of an ABox \Amc at an individual $a \in \mn{ind}(\Amc)$ is the
following (usually infinite) ABox:
\begin{itemize}

  \item $\mn{ind}(U_\Amc^a)$ is the set of sequences $b_0r_0b_1\cdots
    r_{n-1}b_n$ with $n \geq 0$, where $b_0=a$, $b_i \in
    \mn{ind}(\Amc)$ for all $0\leq i \leq n$, $r_i(b_i,b_{i+1}) \in
    \Amc$ for all $0\leq i < n$, and $(b_{i-1},r^-_{i-1}) \neq
    (b_{i+1},r_i)$ (the latter inequality is needed to ensure
    preservation of functionality);

  \item the concept assertions in $U_\Amc^a$ are all assertions of the
    shape $C(\alpha)$ such that
    $\alpha = b_0\cdots b_n \in \mn{ind}(U_\Amc^a)$ and
    $C(b_n) \in \Amc$, and the role assertions in $U_\Amc^a$ are all
    assertions of the shape
    $r(b_0\cdots b_{n-1},\alpha)$ such that $\alpha = b_0\cdots
    b_{n-1}rb_n \in \mn{ind}(U_\Amc^a)$.

\end{itemize}
%
%
The following is not hard to prove \cite<see>{DBLP:journals/lmcs/LutzW17}.
\begin{lemma} \label{prop:unraveling_preserves_consistency} Let
  $\Tmc$ be a $\HornALCHIF$ TBox, \Amc an ABox that is consistent with \Tmc, and $a \in
  \mn{ind}(\Amc)$. Then
  \begin{enumerate}[(a)]

    \item\label{it:unraveling-consistent} $U_\Amc^a$ is consistent
      with $\Tmc$;

    \item\label{it:unraveling-hom} $\Imc_{\Tmc,U_\Amc^a}\to\Imc_{\Tmc,\Amc}$;

    \item\label{it:unraveling-isomorph-subtree} $\Imc_{\Tmc,U_\Amc^a}|_a$ is isomorphic to
      $\Imc_{\Tmc,\Amc}|_a$.

  \end{enumerate}
  
\end{lemma}
We now prove the main result of this section.
\begin{proposition} \label{lem:tree_shaped_witnesses}
  Let $\Tmc_1$ and $\Tmc_2$ be $\HornALCHIF$ TBoxes with $\Tmc_1
  \models_{\sigmaabox,\sigmaquery}^{\textup{RI}}\Tmc_2$.  If $\Tmc_1
  \not\models_{\sigmaabox,\sigmaquery}^{\textup{CQ}} \Tmc_2$, then
  there is a witness $(\Amc,q,\abf)$ where $\Amc$ is tree-shaped and
  $q$ is weakly tree-shaped and has at most one answer variable.
\end{proposition}
\begin{proof}
  Assume $\Tmc_1 \not\models_{\sigmaabox,\sigmaquery}^{\text{CQ}}
  \Tmc_2$, i.e., $\Tmc_2,\Amc \models q(\mathbf{a})$ and $\Tmc_1,\Amc
  \not\models q(\mathbf{a})$, for some $\sigmaabox$-ABox \Amc
  consistent with both $\Tmc_i$, some $\sigmaquery$-CQ $q(\xbf)$ and
  some tuple $\mathbf{a}$.
  Lemma~\ref{lem:universal_model}~(\ref{it:univ_model_queries}) yields
  $\Imc_{\Tmc_2,\Amc}\models q(\abf)$ and
  $\Imc_{\Tmc_1,\Amc}\not\models q(\abf)$.  We first show that the
  following properties of $q$ and $\abf$ are without loss of
  generality.
  \par\smallskip\noindent \begin{enumerate}[(a)] \item
	\label{it:tree-q-quantified-anonymous} Every match $\pi$ of
	$q(\xbf)$ in $\Imc_{\Tmc_2,\Amc}$ with $\pi(\xbf)=\abf$ maps
	every quantified variable into the anonymous part;

    \item\label{it:tree-q-quantified-distinct}
      every match $\pi$ of $q(\xbf)$ in $\Imc_{\Tmc_2,\Amc}$ with
      $\pi(\xbf)=\abf$ maps the quantified variables to pairwise distinct
	  elements;

    \item \label{it:tree-q-no-binary-answer-atoms} $q(\xbf)$ does not
      contain atoms of the form $r(x_1,x_2)$ with $x_1,x_2$ answer
      variables;

    \item \label{it:tree-q-connected} $q(\xbf)$ is connected.  
  \end{enumerate}
  
  \par\smallskip\noindent
  \emph{For~(\ref{it:tree-q-quantified-anonymous})}, take a match
  $\pi$ of $q$ in $\Imc_{\Tmc_2,\Amc}$ with $\pi(\xbf)=\abf$. If there
  is a quantified variable $y$ such that $\pi(y)=b\in\mn{ind}(\Amc)$,
  obtain $q'(\xbf,y)$ from $q(\xbf)$ by removing the quantification
  over $y$, thus making $y$ an answer variable.  Clearly, we have
  $\Imc_{\Tmc_2,\Amc}\models q'(\abf,b)$ and $\Imc_{\Tmc_1,\Amc}\not
  \models q'(\abf,b)$, and thus $\Tmc_2,\Amc\models q'(\abf,b)$ and
  $\Tmc_1,\Amc\not \models q'(\abf,b)$. 
  
  \par\smallskip\emph{For~(\ref{it:tree-q-quantified-distinct})}, take
  a match $\pi$ of $q$ in $\Imc_{\Tmc_2,\Amc}$ with $\pi(\xbf)=\abf$.
  If there are quantified variables $y,y'$ such that $y\neq y'$ and
  $\pi(y)=\pi(y')$, obtain $q'(\xbf)$ from $q(\xbf)$ by replacing all
  occurrences of $y'$ with $y$ and removing quantification over $y'$. Clearly, we have $\Imc_{\Tmc_2,\Amc}\models q'(\abf)$ and
  $\Imc_{\Tmc_1,\Amc}\not \models q'(\abf)$, and thus
  $\Tmc_2,\Amc\models q'(\abf)$ and $\Tmc_1,\Amc\not \models
  q'(\abf)$.
%
  \par\smallskip \emph{For~(\ref{it:tree-q-no-binary-answer-atoms})},
  let $q(\xbf) = \exists \ybf\,(r(x_1,x_2) \land \varphi(\xbf',\ybf))$
  with $x_1,x_2 \in \xbf$, and let $\Imc_{\Tmc_2,\Amc} \models
  q(\abf)$ be witnessed by the match $\pi$ with $\pi(x_i) = a_i$,
  $i=1,2$.  Construct the CQ $q(\xbf') = \exists
  \ybf\,\varphi(\xbf',\ybf)$ by dropping the atom $r(x_1,x_2)$ (and
  thus possibly removing $x_1$ and/or $x_2$ from the free variables).
  It is clear that $\Imc_{\Tmc_2,\Amc}\models q'(\abf')$ for the
  corresponding restriction $\abf'$ of the tuple $\abf$; thus it
  suffices to show that $\Imc_{\Tmc_1,\Amc}\not\models q'(\abf')$.

  From $\Imc_{\Tmc_2,\Amc}\models q(\abf)$ we can conclude
  that $(a_1,a_2) \in r^{\Imc_{\Tmc_2,\Amc}}$.
  By construction of $\Imc_{\Tmc_2,\Amc}$
  there is some $\sigmaabox$-role $r'$
  with $r'(a_1,a_2) \in \Amc$ and $\Tmc_2 \models r' \sqsubseteq r$
  (which includes the possibility $r'=r$, i.e., $r(a_1,a_2) \in \Amc$).
  Due to $\Tmc_1 \models_{\sigmaabox,\sigmaquery}^{\text{RI}} \Tmc_2$,
  we also have $\Tmc_1 \models r' \sqsubseteq r$
  and hence $(a_1,a_2) \in r^{\Imc_{\Tmc_1,\Amc}}$.
  This implies the desired $\Imc_{\Tmc_1,\Amc}\not\models q'(\abf')$ because,
  otherwise,
  any match $\pi$ of $q'$ in $\Imc_{\Tmc_1,\Amc}$ with $\pi(x_i)=a_i$, $i=1,2$,
  could be extended to a match of $q$.


%

  \par\smallskip
  \emph{For~(\ref{it:tree-q-connected})},
  observe that $\Tmc_2,\Amc \models q(\mathbf{a})$ and
  $\Tmc_1,\Amc \not\models q(\mathbf{a})$ implies
  $\Tmc_2,\Amc\models q'(\abf)$ and $\Tmc_1,\Amc\not \models
  q'(\abf)$ for some connected component $q'$ of $q$.

  \par\medskip\noindent Thus, as long as $q$ violates any of the above
  properties, we apply the corresponding modification; it is routine
  to verify that after finitely many steps $q$ satisfies
  properties~(\ref{it:tree-q-quantified-anonymous})
  to~(\ref{it:tree-q-connected}). We verify that then $q$ is weakly
  tree-shaped and has at most one answer variable. 

  \par\smallskip To show the latter, assume that $\xbf$ contains more
  than one answer variable, say $x \neq x'$.
  By~(\ref{it:tree-q-connected}), $q$ is connected, and thus, there is
  a path from $x$ to $x'$ in $q$, that is, a sequence of atoms
  $r_1(z_1,z_2), r_2(z_2,z_3), \dots, r_n(z_n,z_{n+1})$ with
  $z_1=x,z_{n+1}=x'$ and with roles $r_i$ such that $z_{i+1} \neq
  z_{i-1}$ for every $1 < i \leq n$. Let $\pi$ be a match of $q$ in
  $\Imc_{\Tmc_2,\Amc}$ with $\pi(\xbf)=\abf$ and consider $z_2$.
  Because of property~(\ref{it:tree-q-no-binary-answer-atoms}), $z_2$
  cannot be an answer variable. By
  property~(\ref{it:tree-q-quantified-anonymous}), $\pi(z_2)$ is a
  successor of $\pi(z_1)$ in $\Imc_{\Tmc_2,\Amc}|_{\pi(z_1)}$.
  Since $\pi$ is a match
  with $\pi(x')\in\mn{ind}(\Amc)$ and since
  $\Imc_{\Tmc_2,\Amc}|_{\pi(z_1)}$ is weakly tree-shaped,
  there has to be another
  element $z_\ell\neq z_2$ such that $\pi(z_\ell)=\pi(z_2)$, in
  contradiction with property~(\ref{it:tree-q-quantified-distinct}).
 
  \par\smallskip For the former, assume that $q$ is not weakly
  tree-shaped, that is, it contains a cycle $r_1(z_1,z_2),
  r_2(z_2,z_3), \dots, r_n(z_n,z_{n+1})$ such that $z_1=z_{n+1}$,
  $z_2\neq z_n$, and $z_{i+1} \neq z_{i-1}$ for every $1 < i \leq n$.
  Let $\pi$ be a match of $q(\xbf)$ in $\Imc_{\Tmc_2,\Amc}$ with
  $\pi(\xbf)=\abf$. Since $q$ has at most one answer variable and
  satisfies properties~(\ref{it:tree-q-quantified-anonymous})
  and~(\ref{it:tree-q-connected}), there is some ABox individual
  $a\in\mn{ind}(\Amc)$ such that $\pi$ maps all variables
  $z_1,\ldots,z_{n+1}$ into $\Imc_{\Tmc_2,\Amc}|_a$, which is weakly
  tree-shaped by definition.
  Since $\pi$ is a match for the cycle in this part,
  there have to be $i\neq j$ such that $z_i,z_j$ are quantified
  variables with $z_i\neq z_j$ and $\pi(z_i)=\pi(z_j)$, in
  contradiction to property~(\ref{it:tree-q-quantified-distinct}).

  \par\smallskip It remains to transform $\Amc$ into a tree-shaped ABox.
  Due to properties~(\ref{it:tree-q-quantified-anonymous})
  and~(\ref{it:tree-q-connected}), we can fix an ABox
  individual $a$ such that there is a match $\pi$ of $q$ in
  $\Imc_{\Tmc_2,\Amc}|_a$. 
  Note that, if $q$ has a
  single answer variable $x$, then $\pi(x)=a$. Consider the unraveling $U_\Amc^a$ of $\Amc$ at $a$. 
  Due to
  Lemma~\ref{prop:unraveling_preserves_consistency}~(\ref{it:unraveling-consistent}),
  $U_\Amc^a$ is still consistent with both $\Tmc_i$.
  Lemma~\ref{prop:unraveling_preserves_consistency}~(\ref{it:unraveling-isomorph-subtree})
  and the fact that the range of $\pi$ is contained in the domain of $\Imc_{\Tmc_2,\Amc}|_a$
  imply $\Tmc_2,U_{\Amc}^a\models q(a)$. Moreover, 
  Lemma~\ref{prop:unraveling_preserves_consistency}~(\ref{it:unraveling-hom})
  and $\Imc_{\Tmc_1,\Amc}\not\models q(a)$ imply $\Tmc_1,U_{\Amc}^a\not
  \models q(a)$. 
  By compactness, there is a finite subset
  $\Bmc\subseteq U_{\Amc}^a$ with $\Tmc_2,\Bmc\models q(a)$ and
  $\Tmc_1,\Bmc\not \models q(a)$. Clearly, we can also assume that
  $\Bmc$ is connected.
\end{proof}
We also observe an analogous result for 1tCQ entailment.
\begin{proposition}
  \label{lem:tree_shaped_witnesses_tCQs}
  Let $\Tmc_1$ and $\Tmc_2$ be $\HornALCHIF$ TBoxes
  with $\Tmc_1 \models_{\sigmaabox,\sigmaquery}^{\textup{RI}}\Tmc_2$.
  If $\Tmc_1 \not\models_{\sigmaabox,\sigmaquery}^{\textup{\onetCQ}} \Tmc_2$,
  then there is a witness $(\Amc,q,a)$ where $\Amc$ is tree-shaped and
  $q$ is a \onetCQ.
\end{proposition}
\begin{proof}
  Let $(\Amc,a,q)$ be a witness for
  $\Tmc_1\not\not\models_{\sigmaabox,\sigmaquery}^{\textup{\onetCQ}}
  \Tmc_2$, that is, $\Tmc_2,\Amc\models q(a)$ and
  $\Tmc_1,\Amc\not\models q(a)$. As $q$ is already a \onetCQ, it suffices
  to modify the ABox. Consider the unraveling $U_\Amc^a$ of
  $\Amc$ at $a$. We can argue as in the proof of
  Proposition~\ref{lem:tree_shaped_witnesses}
  that there is a tree-shaped (and finite) ABox $\Bmc\subseteq U_\Amc^a$ such
  that $\Tmc_2,\Bmc\models q(a)$ and
  $\Tmc_1,\Bmc\not \models q(a)$.
%
\end{proof}

\subsection{Characterization for CQ Entailment}
\label{sec:bounded_hmphs}

We now develop the announced characterization of CQ entailment. Known
characterizations of entailment in Horn DLs without inverse roles
\cite{LW10,DBLP:journals/ai/BotoevaLRWZ19} suggest that a first natural candidate
characterization would be:
$\Tmc_1 \models_{\sigmaabox,\sigmaquery}^{\text{CQ}} \Tmc_2$ iff
$\Imc_{\Tmc_2,\Amc} \rightarrow_{\sigmaquery} \Imc_{\Tmc_1,\Amc}$
 for all
tree-shaped $\sigmaabox$-ABoxes \Amc that are consistent with $\Tmc_1$ and
$\Tmc_2$.
However, such a characterization fails in the presence of inverse roles, in
fact already for \ELI.
\begin{example}
  \label{exa:homomorphisms_too_strong}
  Let $\Tmc_1 = \{A \sqsubseteq \exists s.B,~ B \sqsubseteq \exists r^-.B\}$,
  $\Tmc_2 = \{A \sqsubseteq \exists s.B,~ B \sqsubseteq \exists r.B\}$,
  $\sigmaabox = \{A\}$, and $\sigmaquery = \{r\}$.
  Then both $\Imc_{\Tmc_1,\Amc}$ and $\Imc_{\Tmc_2,\Amc}$
  contain an infinite $r$-path;
  the $r$-path in $\Imc_{\Tmc_1,\Amc}$
  has a final element while the one in $\Imc_{\Tmc_2,\Amc}$ does not.
  Hence, we have $\Imc_{\Tmc_2,\Amc} \not\rightarrow_{\sigmaquery}
  \Imc_{\Tmc_1,\Amc}$. However, it can be verified that
  $\Tmc_1 \models_{\sigmaabox,\sigmaquery}^{\textup{CQ}}\Tmc_2$, see Theorem~\ref{thm:homs} below.
\end{example}
It is easy to turn the above candidate characterization into a correct
one by replacing homomorphisms with bounded homomorphisms, cf.\
Points~1 and~2 of Lemma~\ref{lem:alternative_def_bounded_hmph}. Since
we are interested in \emph{all} CQs over signature \Qbf, however, we
have to replace unbounded homomorphisms with $n$-bounded
homomorphisms, for \emph{any} bound $n$.  In fact, the following
characterization follows from the definition of CQ entailment,
Proposition~\ref{lem:tree_shaped_witnesses}, and
Lemma~\ref{lem:alternative_def_bounded_hmph}.
\begin{lemma}
\label{lem:ini}
  Let $\Tmc_1$ and $\Tmc_2$ be $\HornALCHIF$ TBoxes with
  $\Tmc_1 \models_{\sigmaabox,\sigmaquery}^{\textup{RI}}\Tmc_2$.  Then
  $\Tmc_1 \models_{\sigmaabox,\sigmaquery}^{\textup{CQ}} \Tmc_2$ iff
  for all tree-shaped $\sigmaabox$-ABoxes \Amc consistent with
  $\Tmc_1$ and $\Tmc_2$, we have
  $\Imc_{\Tmc_2,\Amc}
  \rightarrow^{\mn{fin}}_{\sigmaquery} \Imc_{\Tmc_1,\Amc}$.
\end{lemma}
\begin{proof}
  We prove both implications via contraposition.
  
  \par\smallskip\noindent
  \textbf{{\boldmath ``$\Leftarrow$''.}}~
  Assume $\Tmc_1 \not\models_{\sigmaabox,\sigmaquery}^{\textup{CQ}} \Tmc_2$
  and consider a witness $(\Amc,q,\abf)$.
  By Proposition~\ref{lem:tree_shaped_witnesses},
  we can assume that $\Amc$ is tree-shaped.
  From Lemma~\ref{lem:universal_model}~(\ref{it:univ_model_queries}),
  we get $\Imc_{\Tmc_2,\Amc} \models q(\abf)$
  and $\Imc_{\Tmc_1,\Amc} \not\models q(\abf)$. Let $\pi$ be a match
  for $q(\xbf)$ in $\Imc_{\Tmc_2,\Amc}$ with $\pi(\xbf)=(\abf)$, and
  let $\Imc$ be the finite subinterpretation of $\Imc_{\Tmc_2,\Amc}$
  whose domain $\Delta^\Imc$ is the range of $\pi$. 
  We have $\Imc \not\rightarrow_\sigmaquery \Imc_{\Tmc_1,\Amc}$
  because of $\Imc_{\Tmc_1,\Amc} \not\models q(\abf)$ and $\Imc\models
  q(\abf)$.
  Hence $\Imc_{\Tmc_2,\Amc} \not\rightarrow^{\mn{fin}}_\sigmaquery \Imc_{\Tmc_1,\Amc}$.

  \par\smallskip\noindent
  \textbf{{\boldmath ``$\Rightarrow$''.}}~
  Assume $\Imc_{\Tmc_2,\Amc} \not\rightarrow^{\mn{fin}}_\sigmaquery \Imc_{\Tmc_1,\Amc}$,
  that is, there is a finite subinterpretation $\Imc$ of $\Imc_{\Tmc_2,\Amc}$
  with $\Imc \not\rightarrow_\sigmaquery \Imc_{\Tmc_1,\Amc}$. We 
  associate a query $q_\Imc$ with $\Imc$ as follows. The 
  variables in $q_\Imc$ take the shape $x_a$ with $a\in\Delta^\Imc$,
  and $q_\Imc$ contains atoms $A(x_a)$ for all $a\in A^\Imc$ and
  $r(x_a,x_b)$ for all $(a,b)\in r^{\Imc}$. Answer variables of
  $q_\Imc$ are all $x_a$ with $a\in\abf$, the remaining variables are
  quantified. 
  Then it can be verified that $\Imc_{\Tmc_2,\Amc} \models
  q_\Imc(\abf)$ and $\Imc_{\Tmc_1,\Amc} \not\models q_\Imc(\abf)$.
\end{proof}
%
%
%
We now show that it is possible to refine Lemma~\ref{lem:ini} so that
it makes a much more careful statement in which bounded homomorphisms
and unbounded ones are mixed. It is then possible to
check the unbounded homomorphism part of the characterization using
tree automata as desired, and to deal with bounded homomorphisms using
a mosaic technique that ``precompiles'' relevant information to be
used in the automaton construction.

\smallskip

We begin with a useful lemma which shows that, under certain
additional conditions, we can extract an unbounded homomorphism from a
suitable family of bounded ones. Let \Imc be an interpretation.  For a
signature \sig, we say that the interpretation $\Imc$ is
\emph{$\sig$-connected} if between any two domain elements there is a
path using only $\sig$-role edges, and \emph{finitely branching} if
every domain element has only finitely many directly connected
elements.
\begin{lemma} \label{lem:constructhoms_ELHIF_bot} 
  Let $\Imc_1,\Imc_2$ be finitely branching interpretations and let
  $\Imc_1$ be $\sig$-connected, for a signature \sig. If there are $d_0 \in
  \Delta^{\Imc_1}$ and $e_0 \in \Delta^{\Imc_2}$ such that for each $i
  \geq 0$ there is an $\sig$-homomorphism $h_i$ from
  $\Imc_1|^{d_0}_i$ to $\Imc_2$ with $h_i(d_0)=e_0$, then $\Imc_1
  \rightarrow_{\sig} \Imc_2$.
\end{lemma}
\begin{proof} We are going to construct an $\sig$-homomorphism
  $h$ from $\Imc_1$ to $\Imc_2$ step by step, obtaining the desired
  homomorphism in the limit.  We will take care that, at all times,
  the domain of $h$ is finite and
  \begin{itemize}

  \item[($*$)] there is a sequence $h_0,h_1,\dots$ with $h_i$ 
    an $\sig$-homo\-morphism from $\Imc_1|^{d_0}_i$ to
    $\Imc_2$ such that whenever $h(d)$ is already defined, then
    $h_i(d)=h(d)$ for all $i \geq 0$.

  \end{itemize}
  Start with setting $h(d_0)=e_0$. The original sequence
  $h_0,h_1,\ldots$
  from the lemma witnesses~($*$). Now consider the set $\Lambda$ that
  consists of all elements $d \in \Delta^{\Imc_1}$ such that $h(d)$ is
  undefined and there is an $e \in \Delta^{\Imc_1}$ with $h(e)$
  defined and $(e,d)\in r^{\Imc_1}$, for some $\sig$-role $r$.
  Since the domain of $h$ is finite and $\Imc_1$ is finitely
  branching, $\Lambda$ is finite. By ($*$) and since $\Imc_2$ is finitely branching, for each $d \in
  \Lambda$, there are only finitely many $e'$ such that $h_i(d)=e'$ for
  some $i$. Thus, there must be a function $\delta:\Lambda \rightarrow
  \Delta^{\Imc_2}$ such that, for infinitely many $i$, we have
  $h_i(d)=\delta(d)$, for all $d \in \Lambda$. Extend $h$ accordingly,
  that is, set $h(d)=\delta(d)$ for all $d \in \Lambda$. Clearly, the
  sequence $h_0,h_1,\dots$ from ($*$) before the extension is no
  longer sufficient to witness ($*$) after the extension. We fix this
  by skipping homomorphisms that do not respect $\delta$, that is,
  define a new sequence $h'_0,h'_1,\dots$ by using as $h'_i$ the
  restriction of $h_j$ to the domain of $\Imc_1|^{d_0}_i$ where $j\geq
  i$ is smallest such that $h_j(d)=\delta(d)$ for all $d \in \Lambda$.
  This finishes the construction. The lemma follows from the fact
  that, due to $\sig$-connectedness of $\Imc_1$, every element
  is eventually reached. Note that we automatically have
  $h(a)=a$ for all individual names $a$ (as required), no matter
  whether $d_0$ is an individual name or not.
\end{proof}
We continue with introducing notation relevant for the refinement of
Lemma~\ref{lem:ini}. For a
signature~$\sig$, we use $\Imc|^\mn{con}_\sig$ to denote the
restriction of the interpretation \Imc to those elements that can be
reached from an ABox individual by traveling along $\sig$-roles
(forwards or backwards).   An \emph{$\sig$-subtree in} $\Imc_{\Tmc,\Amc}$ is a
maximal weakly tree-shaped, $\sig$-connected sub-interpretation \Imc of
$\Imc_{\Tmc,\Amc}$ that does not comprise any ABox individuals. The
\emph{root} of \Imc is the (unique) element of $\Delta^\Imc$ that can
be reached from an ABox individual on a shortest path among all
elements of~$\Delta^\Imc$. 
\begin{theorem}
\label{thm:homs}
  Let $\Tmc_1$ and $\Tmc_2$ be $\HornALCHIF$ TBoxes with
  $\Tmc_1 \models_{\sigmaabox,\sigmaquery}^{\textup{RI}}\Tmc_2$.  Then
  $\Tmc_1 \models_{\sigmaabox,\sigmaquery}^{\textup{CQ}} \Tmc_2$ iff
  for all tree-shaped $\sigmaabox$-ABoxes \Amc consistent with
  $\Tmc_1$ and $\Tmc_2$, and for all weakly tree-shaped and finitely
  branching models $\Imc_1$ of $\Amc$ and $\Tmc_1$, the following
  hold:
  \begin{enumerate}

  \item[(1)] $\Imc_{\Tmc_2,\Amc}|^\mn{con}_{\sigmaquery} \rightarrow_{\sigmaquery}
    \Imc_1$;

  \item[(2)] for all $\sigmaquery$-subtrees \Imc in $\Imc_{\Tmc_2,\Amc}$, one
    of the following holds:
    \begin{enumerate}

      \item[(a)] $\Imc \rightarrow_{\sigmaquery} \Imc_1$;

      \item[(b)] $\Imc \rightarrow^\mn{fin}_{\sigmaquery}
	\Imc_{\Tmc_1,\mn{tp}_{\Imc_1}(a)}$ for some $a \in
	\mn{ind}(\Amc)$.

    \end{enumerate}
  \end{enumerate}
\end{theorem}
\begin{proof}
  \textbf{{\boldmath ``$\Leftarrow$''.}}~
  Suppose that $\Tmc_1
  \not\models_{\sigmaabox,\sigmaquery}^{\text{CQ}} \Tmc_2$. By
  Proposition~\ref{lem:tree_shaped_witnesses},
  there is a tree-shaped $\sigmaabox$-ABox $\Amc$
  consistent with both $\Tmc_i$,
  and a weakly tree-shaped $\sigmaquery$-CQ $q$ such that either
  \begin{enumerate}

    \item[(1$'$)] $q$ has a single answer variable
      and there is an element $a \in \mn{ind}(\Amc)$ such that
      $\Tmc_2,\Amc \models q(a)$
      but $\Tmc_1,\Amc \not\models q(a)$, or

    \item[(2$'$)] $q$ is Boolean and $\Tmc_2,\Amc \models q$
      but $\Tmc_1,\Amc \not\models q$.

  \end{enumerate}
  In case (1$'$), let $\pi$ be a match of $q$ in $\Imc_{\Tmc_2,\Amc}$
  with $\pi(x) = a$.  Since $q$ contains an answer variable, we must have
  $\Imc_{\Tmc_2,\Amc}|^\mn{con}_{\sigmaquery} \not
  \rightarrow_{\sigmaquery} \Imc_{\Tmc_1,\Amc}$ as otherwise the
  composition of $\pi$ and the witnessing homomorphism shows
  $\Imc_{\Tmc_1,\Amc} \models q(a)$, which is not the case.
  Thus, Condition~(1) is violated for $\Imc_1=\Imc_{\Tmc_1,\Amc}$.

  In case (2$'$), consider a match $\pi$ of $q$ in
  $\Imc_{\Tmc_2,\Amc}$, and let $\Imc$ be the restriction of
  $\Imc_{\Tmc_2,\Amc}$ to the elements in the range of $\pi$. Clearly,
  we have $\Imc \not\rightarrow_\sigmaquery \Imc_{\Tmc_1,\Amc}$.
  Consequently, we also have $\Imc\not\rightarrow^n_\sigmaquery
  \Imc_{\Tmc_1,\Amc}$ where $n$ is the number of variables in $q$,
  implying that Conditions~(2a) and~(2b) are both false.


%
%

%

  \smallskip\noindent
  \textbf{{\boldmath ``$\Rightarrow$''.}}~
  Assume that $\Tmc_1 \models_{\sigmaabox,\sigmaquery}^{\text{CQ}} \Tmc_2$ and let $\Amc$
  be a tree-shaped $\sigmaabox$-ABox consistent with both $\Tmc_i$.
  Moreover, let $\Imc_1$ be a weakly tree-shaped, finitely branching model of $\Amc$
  and~$\Tmc_1$.
  
  \smallskip\noindent\par For Condition~(1), and let
  $\Imc_{\Tmc_2,\Amc}|^\mn{con}_{\sigmaquery}$ be the disjoint union
  of the connected interpretations $\Jmc_1,\dots,\Jmc_k$. In each
  $\Jmc_i$, we find at least one individual $a_i$ from
  $\mn{ind}(\Amc)$. Let $\ell \in \{1,\dots,k\}$. By $\Tmc_1
  \models_{\sigmaabox,\sigmaquery}^{\text{CQ}} \Tmc_2$ and
  Lemma~\ref{lem:ini}, we have
  $\Imc_{\Tmc_2,\Amc}\to^\mn{fin}_\sigmaquery \Imc_{\Tmc_1,\Amc}$ and
  hence, by
  Lemma~\ref{lem:universal_model}~(\ref{it:univ_model_embeds}),
  $\Imc_{\Tmc_2,\Amc}\to^\mn{fin}_\sigmaquery \Imc_1$. By 
  Lemma~\ref{lem:alternative_def_bounded_hmph}~(3), we find a sequence $h_0,h_1,\dots$ such that
  $h_i$ is a $\sigmaquery$-homomorphism from $\Jmc_\ell|^{a_\ell}_i$
  to $\Imc_1$.  Note that we must have $h_i(a_\ell)=a_\ell$ for all
  $i$. Thus, Lemma~\ref{lem:constructhoms_ELHIF_bot} yields $\Jmc_\ell
  \rightarrow_\sigmaquery \Imc_1$ and, in summary,
  $\Imc_{\Tmc_2,\Amc}|^\mn{con}_{\sigmaquery} \rightarrow_\sigmaquery
  \Imc_1$.

\smallskip\noindent\par For Condition~(2), let \Imc be a
$\sigmaquery$-subtree in $\Imc_{\Tmc_2,\Amc}$ with root $d_0$. By
Lemma~\ref{lem:ini}, we have $\Imc\to^\mn{fin}_\sigmaquery
\Imc_{\Tmc_1,\Amc}$. By Lemma~\ref{lem:alternative_def_bounded_hmph}~(3), there is a
sequence $h_0,h_1,\dots$ such that $h_i$ is a
$\sigmaquery$-homomorphism from $\Imc|^{d_0}_i$ to
$\Imc_{\Tmc_1,\Amc}$. We distinguish two cases, leading
to~(2)(a) and~(2)(b), respectively.

First assume that there is an $e_0 \in \Delta^{\Imc_{\Tmc_1,\Amc}}$
such that $h_i(d_0)=e_0$ for infinitely many~$i$. Construct a new
sequence $h'_0,h'_1,\dots$ with $h'_i$ a $\sigmaquery$-homomorphism
from $\Imc|^{d_0}_i$ to $\Imc_{\Tmc_1,\Amc}$ by skipping homomorphisms
that do not map $d_0$ to $e_0$, that is, $h'_i$ is the restriction of
$h_j$ to the domain of $\Imc|^{d_0}_i$ where $j\geq i$ is smallest
such that $h_j(d_0)=e_0$. Clearly, $h'_i(d_0)=e_0$ for all~$i$. Thus,
Lemma~\ref{lem:constructhoms_ELHIF_bot} yields $\Imc \rightarrow_{\sigmaquery}
\Imc_{\Tmc_1,\Amc}$ and
thus, by Lemma~\ref{lem:universal_model}~(\ref{it:univ_model_embeds})
$\Imc \rightarrow_{\sigmaquery} \Imc_1$.

Otherwise, there is no $e_0 \in \Delta^{\Imc_{\Tmc_1,\Amc}}$ such that
$h_i(d_0)=e_0$ for infinitely many $i$. We can assume that there is an
$a_0 \in \mn{ind}(\Amc)$ such that $h_i(d_0) \in
\Delta^{\Imc_{\Tmc_1,\Amc}|_{a_0}}$ for all $i$; in fact, there must be
an $a_0$ such that $h_i(d_0) \in \Delta^{\Imc_{\Tmc_1,\Amc}|_{a_0}}$
for infinitely many $i$ and we can again skip homomorphisms to achieve
this for all $i$.  It is important to note that the remaining
homomorphisms do not necessarily all map domain elements of $\Imc$ to
elements in $\Imc_{\Tmc_1,\Amc}|_{a_0}$ due to the presence of inverse
roles.  Now, since $\Imc_{\Tmc_1,\Amc}$ is finitely branching, for
all $i,n \geq 0$ we must find a $j\geq i$ such that $h_j(d_0)$ is a
domain element whose distance from $a_0$ exceeds~$n$ (otherwise the
previous case would apply). Based on this observation, we construct a
sequence of homomorphisms $h'_0,h'_1,\dots$ as follows.  For all
$i\geq 0$, let $h'_i$ be the restriction of some $h'_j$ to the domain
of $\Imc|^{d_0}_i$ where $j \geq i$ is smallest such that the distance
of $h_j(d_0)$ from $a_0$ exceeds $i$. It should be clear that each
$h_i'$ is a $\sigmaquery$-homomorphism from $\Imc|^{d_0}_i$ to
$\Imc_{\Tmc_1,\Amc}|_{a_0}$, hence $\Imc
\rightarrow^\mn{fin}_{\sigmaquery} \Imc_{\Tmc_1,\Amc}|_{a_0}$, by 
Lemma~\ref{lem:alternative_def_bounded_hmph}~(3).

Now Lemma~\ref{lem:universal_model}~(\ref{it:univ_model_embeds})
and Condition~(1) of homomorphisms
imply
$\mn{tp}_{\Imc_{\Tmc_1,\Amc}}(a_0) \subseteq \mn{tp}_{\Imc_1}(a_0)$;
thus by Lemma~\ref{lem:universal_model}~(\ref{it:univ_model_regularity_supertypes})
we have
$\Imc_{\Tmc_1,\Amc}|_{a_0}
\rightarrow \Imc_{\Tmc_1,\mn{tp}_{\Imc_1}(a_0)}$. Altogether we get
$\Imc \rightarrow^\mn{fin}_{\sigmaquery}
\Imc_{\Tmc_1,\mn{tp}_{\Imc_1}(a_0)}$ as required.
%
\end{proof}

%
\subsection{Characterization for \onetCQ Entailment} 

We now give a \onetCQ version of Theorem~\ref{thm:homs}, that is, 
a characterization of \onetCQ-entailment. We use (unbounded)
simulations in place of homomorphisms. 
\begin{theorem}
\label{thm:homs_tCQ}
  Let $\Tmc_1$ and $\Tmc_2$ be $\HornALCHIF$ TBoxes with
  $\Tmc_1 \models_{\sigmaabox,\sigmaquery}^{\textup{RI}}\Tmc_2$.  Then
  the following three conditions are equivalent:
  \begin{enumerate}
    \item[(a)]
      $\Tmc_1 \models_{\sigmaabox,\sigmaquery}^{\textup{\onetCQ}} \Tmc_2$;
    \item[(b)]
      $\Imc_{\Tmc_2,\Amc} \preceq_\sigmaquery
      \Imc_1$
      for all tree-shaped $\sigmaabox$-ABoxes \Amc that are consistent with
      $\Tmc_1$ and $\Tmc_2$, and for all weakly tree-shaped, finitely branching
      models $\Imc_1$ of $\Amc$  and $\Tmc_1$.
      
  \end{enumerate}
\end{theorem}
%
\begin{proof}

  \smallskip\noindent \textbf{{\boldmath (b) $\Rightarrow$ (a).}}
  This implication is analogous to the ``$\Rightarrow$'' direction in
  the proof of Theorem~\ref{thm:homs}, but using
  Proposition~\ref{lem:tree_shaped_witnesses_tCQs} in place of
  Proposition~\ref{lem:tree_shaped_witnesses}. This rules out
  Case~(2$'$) and thus Condition~(2) from Theorem~\ref{thm:homs}.  In
  Condition~(1), we can replace homomorphisms with simulations since
  the witnessing CQ is a  \onetCQ.

  \smallskip\noindent
  \textbf{{\boldmath (a) $\Rightarrow$ (b).}}
  For the proof of this direction, we need a bounded variant of
  simulations, analogously to bounded homomorphisms. We write
$\Imc_1 \preceq^n_\sigmaquery \Imc_2$ if for any subinterpretation
$\Imc'_1$ of $\Imc_1$ with $|\Delta^{\Imc'_1}| \leq n$, we have
$\Imc'_1 \preceq_\sigmaquery \Imc_2$. Moreover, we write
$\Imc_1 \preceq^\mn{fin}_\sigmaquery \Imc_2$ if
$\Imc_1 \preceq^n_\sigmaquery \Imc_2$, for any~$n$.

  Assume that $\Tmc_1
  \models_{\sigmaabox,\sigmaquery}^{\text{\onetCQ}} \Tmc_2$ and let $\Amc$
  be a tree-shaped $\sigmaquery$-ABox that is consistent with both $\Tmc_i$.
  Let $\Imc_1$ be a weakly tree-shaped, finitely branching model of $\Tmc_1$ and $\Amc$.
  We first show the following.
  \\[2mm]
  {\bf Claim.}
  $\Imc_{\Tmc_2,\Amc}|_a \preceq^\mn{fin}_\sigmaquery \Imc_1$ for all
  $a \in \mn{ind}(\Amc)$.
  \\[2mm]
  Assume to the contrary that
  $\Imc_{\Tmc_2,\Amc}|_a \npreceq^\mn{fin}_\sigmaquery \Imc_1$.
  Then $\Imc_{\Tmc_2,\Amc}|_a \npreceq^n_\sigmaquery \Imc_1$
  for some $n$, that is,
  there is a subinterpretation \Imc of $\Imc_{\Tmc_2,\Amc}|_a$ with
  $|\Delta^\Imc| \leq n$ such that $\Imc \npreceq_\sigmaquery
  \Imc_1$.
  We can assume w.l.o.g.\ that \Imc is connected and contains $a$
  (otherwise we just extend \Imc and increase $n$ accordingly).
  Let $q_\Imc$ be \Imc viewed as a 
  weakly tree-shaped CQ with domain elements viewed as variables and
  $a$ viewed as the only answer variable.
  Clearly, $\Imc \models q(a)$ and thus $\Imc_{\Tmc_2,\Amc}|_a \models q(a)$.
  Let $\pi$ be a witnessing match of $q$ in $\Imc$.
  To transform $q$ into a \onetCQ, perform the following operations:
  \begin{itemize}
    \item
      remove all atoms of the form $r(a,a)$;

    \item
      Split multi-edges in $q$ by duplicating subtree queries: if
      $r_1(z,z'),\dots,r_n(z,z')$ are in $q$ where $r_1,\dots,r_n$ are
      potentially inverse roles and $z'$ is a successor of $z$ when
      the answer variable is viewed as the root of $q$, then introduce
      $n$ copies of the subtree rooted at $z'$ and connect $z$ to
      each copy with one of the roles $r_1,\dots,r_n$. Remove
      the subtree rooted at $z'$.

  \end{itemize}
  The result of this transformation is a \onetCQ $q'$,
  which still satisfies $\Imc_{\Tmc_2,\Amc} \models q'(a)$.
  On the other hand, $\Imc_1 \not\models q'(a)$
  because, otherwise, a match $\pi'$ of $q'(x)$ in $\Imc_1$
  would give rise to a simulation of $\Imc$ in $\Imc_1$.
  This finishes the proof of the claim.

  \par\medskip\noindent
  
  The claim implies that for all $a \in \mn{ind}(\Amc)$ and all
  $i \geq 0$, there is a $\Qbf$-simulation from
  $(\Imc_{\Tmc_2,\Amc}|_a)|^a_i$ to $\Imc_1$: for every $a$,
  the situation parallels
  Condition~(3a) of
  Lemma~\ref{lem:alternative_def_bounded_hmph} with homomorphisms
  replaced by simulations. We can use exactly the same arguments as in
  the proof of Lemma~\ref{lem:alternative_def_bounded_hmph}~(3) to show
  that there is a $\Qbf$-simulation from $(\Imc_{\Tmc_2,\Amc}|_a)|^a_i$
  to $\Imc_1$.  The union of these simulations, for all
  $a \in \mn{ind}(\Amc)$, is the desired $\Qbf$-simulation from
  $\Imc_{\Tmc_2,\Amc}$ to $\Imc_1$.
%
\end{proof}

\newcommand{\ata}{\text{2ATA$_c$}\xspace}

\section{Decidability and Complexity} \label{sec:automata}

Our aim is to prove the following result.
\begin{theorem}
  \label{thm:complexity}
  In \HornALCHIF and any of its fragments that contains \ELI or
  Horn-\ALC, the following problems
  are \TwoExpTime-complete:
%
  $(\sigmaabox,\sigmaquery)$-CQ entailment,
  $(\sigmaabox,\sigmaquery)$-CQ inseparability, and
  $(\sigmaabox,\sigmaquery)$-CQ conservative extensions. This holds
  even when $\sigmaabox=\sigmaquery$. Moreover, these problems can be
  solved in time $2^{2^{p(|\Tmc_2| \mn{log}|\Tmc_1|)}}$, where $p$ is a
  polynomial.
%
%
%
\end{theorem}
The lower bounds have been established
by~\citeauthor{DBLP:journals/ai/BotoevaLRWZ19}~\citeyear{DBLP:journals/ai/BotoevaLRWZ19}.  To obtain the upper
bounds, we use an approach based on the characterization provided by
Theorem~\ref{thm:homs} that combines a mosaic technique with tree
automata.  More precisely, we first observe that Condition~(2b) is the
only part of Theorem~\ref{thm:homs} concerned with bounded
homomorphisms, and that addressing this condition requires us to check
for a given $\Qbf$-subtree \Imc in $\Imc_{\Tmc_2,\Amc}$ and a
$\Tmc_1$-type $t_1$, whether $\Imc \rightarrow^\mn{fin}_{\Qbf}
\Imc_{\Tmc_1,t_1}$. In fact, it suffices to consider subtrees \Imc
whose root is connected only by non-\Qbf roles as all \Imc that do not
satisfy this condition are not maximal or treated by Condition~(1) of
Theorem~\ref{thm:homs}. An analysis of universal models reveals that
such \Imc are completely determined by the $\Tmc_2$-type $t_2$ of
their root and are in fact identical to
$\Imc_{\Tmc_2,t_2}|^{\mn{con}}_\Qbf$.\footnote{This depends on the
  assumption that the root is connected only by non-\Qbf roles, c.f.\
  Condition~(ii) from the definition of universal models. } This gives
rise to the problem of deciding whether
$\Imc_{\Tmc_2,t_2}|^{\mn{con}}_\Qbf\rightarrow^\mn{fin}_{\Qbf}
\Imc_{\Tmc_1,t_1}$, for given $\Tmc_1,\Tmc_2,t_1,t_2$, and \Qbf.
%
We show that this
problem can be solved in \TwoExpTime using a mosaic technique and then
use that result as a black box in an automata based approach to prove
the upper bound.

\subsection{Mosaics}
\label{sec:mosaics}

The aim of this section is to prove the following.
%
\begin{theorem}
  \label{cor:bounded_hmph_in_2EXPTIME}
  Given two Horn-\ALCHIF TBoxes $\Tmc_1$ and $\Tmc_2$, a $\Tmc_1$-type
  $t_1$, a $\Tmc_2$-type $t_2$, and a signature $\sig$, it can be
  decided in time $2^{2^{p(|\Tmc_2| \mn{log}|\Tmc_1|)}}$ whether
  $\Imc_{\Tmc_2,t_2}|^{\mn{con}}_\sig \rightarrow^\mn{fin}_{\sig}
  \Imc_{\Tmc_1,t_1}$, where $p$ is a polynomial.
\end{theorem}
We prove Theorem~\ref{cor:bounded_hmph_in_2EXPTIME} by replacing
bounded homomorphisms with unbounded ones and, to compensate for this,
further replacing the target interpretation $\Imc_{\Tmc_1,t_1}$ with a
suitably constructed class of interpretations.

To illustrate, consider Example~\ref{exa:homomorphisms_too_strong} and
let $t_1=t_2=\{B\}$. The difference between
$\Imc_{\Tmc_2,t_2} \rightarrow^\mn{fin}_{\sig} \Imc_{\Tmc_1,t_1}$ and
$\Imc_{\Tmc_2,t_2} \rightarrow_{\sig} \Imc_{\Tmc_1,t_1}$ is that
unbounded homomorphisms fail once they ``reach the root'' of
$\Imc_{\Tmc_1,t_1}$ while bounded homomorphisms can, depending on the
bound, map the root of $\Imc_{\Tmc_2,t_2}$ deeper and deeper into
$\Imc_{\Tmc_1,t_1}$, thus never reaching its root. The latter is
possible because $\Imc_{\Tmc_1,t_1}$ is regular in the sense that any
two elements which have the same type are the root of isomorphic
subtree interpretations.  This is of course not only true in this
example, but 
in any universal model. To transition back from bounded to unbounded
homomorphisms, we replace $\Imc_{\Tmc_1,t_1}$ with a class of
interpretations that can be seen as a ``backwards regularization'' of
$\Imc_{\Tmc_1,t_1}$. In our concrete example, $\Imc_{\Tmc_1,t_1}$
contains an element that has the same type as the root of
$\Imc_{\Tmc_1,t_1}$ and that has a predecessor $d$.  This would lead
us to consider (among others) the interpretation obtained from
$\Imc_{\Tmc_1,t_1}$ by adding a predecessor of the same type as $d$
and possibly adding further predecessors, even ad infinitum. We now
make this idea precise.

A \emph{rooted weakly tree-shaped interpretation} is a pair
$(\Imc,d_0)$ with \Imc a weakly tree-shaped interpretation and $d_0
\in \Delta^\Imc$ a \emph{root}. The root imposes a direction on \Imc
and thus allows us to speak about successors and predecessors, that
is, when $(d,e) \in r^\Imc$, then $e$ is a \emph{successor} of $d$ if
the distance of $e$ to $d_0$ in the undirected graph
$(\Delta^\Imc,\{\{d,e\} \mid (d,e) \in r^\Imc \text{ for some } r \in
\NC \})$ is larger than the distance of $d$ to $d_0$, and a
\emph{predecessor} otherwise. Note that the direction does not
correspond to the distinction between roles and inverse roles, but
rather reflects the directedness of universal models $\Imc_{\Tmc,t}$,
which can be viewed as being rooted weakly tree-shaped:
$\Delta^{\Imc_{\Tmc,t}}$ is a set of paths and the root of
$\Imc_{\Tmc,t}$ is the unique path in $\Delta^{\Imc_{\Tmc,t}}$ of
minimum length. 

Throughout the section, we work with sets $\rho$ of (possibly inverse) roles; in
particular, we say that $b$ is a \emph{$\rho$-successor} of
$a$ if $\rho=\{r\mid (a,b)\in r^\Imc\}$. Let \Tmc be a \HornALCHIF TBox
and let $\mn{tp}(\Tmc)$ be the set of all \Tmc-types consistent with~$\Tmc$.
For every $t_0 \in \mn{tp}(\Tmc)$, we use $\mn{tp}(\Tmc,t_0)$ to denote the
set of all $t \in \mn{tp}(\Tmc)$ that occur in the universal model
$\Imc_{\Tmc,t_0}$ of $t_0$ and \Tmc.  
Furthermore, given a rooted weakly tree-shaped interpretation $(\Imc,d_0)$
and an element $d \in \Delta^\Imc$,
the \emph{1-neighborhood of $d$ in \Imc}
is a tuple $n_1^{\Imc}(d) = (t^-,\rho,t,S)$
such that
\begin{enumerate}

\item[(a)] $t = \mn{tp}_\Imc(d)$;


\item[(b)] if $d$ is the $\rho'$-successor of some $d^- \in \Delta^\Imc$, then
$t^-=\mn{tp}_\Imc(d^-)$ and $\rho=\rho'$, otherwise $\rho=t^-=\bot$;


\item[(c)] $S$ is the set of all pairs $(\rho',t')$
such that there is a $\rho'$-successor $d'$ of $d$ with
$t' = \mn{tp}_\Imc(d')$.

\end{enumerate}
We further write $(t^-_1,\rho_1,t_1,S_1) \sqsubseteq (t^-_2,\rho_2,t_2,S_2)$
if $t_1=t_2$, $S_1 \subseteq S_2$ and, if $\rho_1 \neq \bot$, then $\rho_1=\rho_2$ and $t^-_1=t^-_2$.

In the following, we define the class $\MOD_\omega(\Tmc,t_0)$ of
models of \Tmc that we consider as homomorphism targets.  Each model
from this class is obtained as the limit of a (finite or infinite)
sequence of rooted weakly tree-shaped interpretations. Start to
construct the sequence by choosing a type $t\in\mn{tp}(\Tmc,t_0)$ and
defining $\Imc=(\{d_0\},\cdot^{\Imc})$ such that
$\mn{tp}_{\Imc}(d_0)=t$. The remaining sequence is obtained by
repeatedly applying the following rule:
\begin{enumerate}

\item[{\bf (R)}] Let $d\in\Delta^\Imc$. Choose some
  $e\in \Delta^{\Imc_{\Tmc,t_0}}$ such that
  $n_1^{\Imc}(d) \sqsubseteq n_1^{\Imc_{\Tmc,t_0}}(e)$, and add to $d$
  the predecessor and/or successors required to achieve
  $n_1^{\Imc}(d) = n_1^{\Imc_{\Tmc,t_0}}(e)$. If a new predecessor is
  added, it becomes the new root.

\end{enumerate}
Note that rule application does not need to be fair nor exhaustive. Now,
$\MOD_\omega(\Tmc,t_0)$ is the set of all (weakly tree-shaped but not
rooted) interpretations $\Imc$ that can be obtained as the limit of a
sequence constructed in the described way.
\begin{lemma} \label{lem:bounded_hmph_via_can_omega} Let \Tmc be a
  \HornALCHIF TBox, $t_0 \in \mn{tp}(\Tmc)$, and \Imc a finitely
  branching, weakly tree-shaped interpretation. Then $\Imc
  \rightarrow^\mn{fin}_\sig \Imc_{\Tmc,t_0}$ iff there is a $\Jmc \in
  \MOD_\omega(\Tmc,t_0)$ with $\Imc \rightarrow_\sig \Jmc$.
\end{lemma}
\begin{proof} \textbf{{\boldmath ``$\Rightarrow$''.}}~Suppose
  $\Imc\rightarrow^\mn{fin}_\sig \Imc_{\Tmc,t_0}$ for some finitely
  branching, weakly tree-shaped interpretation \Imc, and choose a root
  $d_0$ of $\Imc$. By Lemma~\ref{lem:alternative_def_bounded_hmph}~(3),
  there is a sequence $h_0,h_1,\ldots$ such that $h_i$ is an
  $\sig$-homomorphism from $\Imc|^{d_0}_i$ to $\Imc_{\Tmc,t_0}$.  Note
  that both $\Imc$ and $\Imc_{\Tmc,t_0}$ are finitely branching. By
  skipping homomorphisms (similar to the proof of
  Lemma~\ref{lem:constructhoms_ELHIF_bot}), we can thus construct a
  new sequence $h_0',h_1', \ldots$ such that $h_i'$ is an
  $\sig$-homomorphism from $\Imc|_{i}^{d_0}$ to $\Imc_{\Tmc,t_0}$ and,
  additionally, for every $i\geq 0$ and every $d\in
  \Delta^{\Imc|^{d_0}_i}$ the following properties hold:
  \begin{enumerate}

    \item[(i)]
      $n_1^{\Imc_{\Tmc,t_0}}(h_i'(d))=n_1^{\Imc_{\Tmc,t_0}}(h_j'(d))$
      for all $j$ with $j\geq i$;

    \item[(ii)] if $e$ is a successor of $d$ in \Imc, then one of the following
      is the case:
      \begin{itemize}

	\item $h_j'(e)$ is the predecessor of $h_j'(d)$ in
	  $\Imc_{\Tmc,t_0}$, for all $j$ with $j\geq i$, or

	\item $h_j'(e)$ is a successor of $h_j'(d)$ in
	  $\Imc_{\Tmc,t_0}$, and there is some $(\rho,t)$ in component
	  $S$ of $n_1^{\Imc_{\Tmc,t_0}}(h_i'(d))$ such that for
	  all $j$ with $j\geq i$, we have 
	  $\rho=\{r\mid (h_j'(d),h_j'(e))\in r^{\Imc_{\Tmc,t_0}}\}$
	  and $h_j'(e)=t$.

      \end{itemize}

  \end{enumerate}
  Guided by $h_i'$, we construct a sequence of 
  rooted weakly tree-shaped interpretations $(\Jmc_0,e_0),$
  $(\Jmc_1,e_1),\ldots$ and a sequence $g_0, g_1,\ldots$ with 
  $g_i$ an $\sig$-homomorphism from $\Imc|^{d_0}_i$ to $\Jmc_i$ such 
  that for every $i,j$ with $0\leq i\leq j$ and every $d\in
  \Delta^{\Imc|^{d_0}_i}$, we have $g_i(d) = g_j(d)$. The required
  interpretation $\Jmc\in\mn{mod}_\omega(\Tmc,t_0)$ is obtained in the
  limit.  Throughout the construction, we maintain the invariant 
  \[
    \tag{$*$}
    n_1^{\Jmc_i}(g_i(d)) \sqsubseteq n_1^{\Imc_{\Tmc,t_0}}(h_i'(d))
  \]
  for all $i,d$ such that $g_i(d)$ is defined.

  We start with $\Jmc_0 = (\{e_0\},\cdot^{\Jmc_0})$ such that
  $\mn{tp}_{\Jmc_0}(e_0)=\mn{tp}_{\Imc_{\Tmc,t_0}}(h_0'(d_0))$, choose
  $e_0$ as the root and set $g_0(d_0)=e_0$. Clearly $(*)$ is
  satisfied.  Assuming that $(\Jmc_i,e_i)$ and $g_i$ are already
  defined, we extend them to $(\Jmc_{i+1},e_{i+1})$ and $g_{i+1}$ as
  follows. Choose some $d\in \Delta^{\Imc|^{d_0}_i}$ and
  $d'\notin\Delta^{\Imc|^{d_0}_i}$ such that $(d,d') \in r^\Imc$ for
  some role $r$. By invariant~$(*)$ and Point~(i), 
  we have $n_1^{\Jmc_i}(g_i(d)) \sqsubseteq n_1^{\Imc_{\Tmc,t_0}}(h_j'(d))$
  for all $j \geq i$. Thus, we can apply~\textbf{(R)} to
  $g_i(d)$ in $\Jmc_i$ and $h'_i(d)$ in $\Delta^{\Imc_{\Tmc,t_0}}$. More precisely, we 
  obtain $\Jmc_{i+1}$ from $\Jmc_i$ by adding a predecessor and/or successors
  to achieve 
  \[
    \tag{$**$}
    n_1^{\Jmc_{i+1}}(g_i(d)) = n_1^{\Imc_{\Tmc,t_0}}(h_i'(d)).
  \]
  Moreover, $e_{i+1}$ is the root of $\Jmc_{i+1}$, updated according
  to~\textbf{(R)}. To define $g_{i+1}$, we extend $g_i$ to $d'$ by 
  distinguishing two cases according to Point~(ii): 

  \begin{itemize}

    \item Suppose $h'_{j}(d')$ is the predecessor of $h'_{j}(d)$ for all $j
      \geq i$. We set $g_{i+1}(d')$ to the predecessor of $g_i(d)$
      (exists due to~($**$)). Clearly, $(*)$ is satisfied also for $g_{i+1}(d')$.


    \item Suppose $h'_{j}(d')$ is a successor of $h'_{j}(d)$ for all
      $j \geq i$, and let $(\rho,t)$ be as described. Let $e\in
      \Delta^{\Jmc_{i+1}}$ be an element with
      $\mn{tp}_{\Jmc_{i+1}}(e)=t$ and $\rho=\{r\mid (g_i(d),e)\in
      r^{\Jmc_{i+1}}\}$. Note that such an element exists due
      to~($**$). Set $g_{i+1}(d')=e$. Clearly, $(*)$ is satisfied also
      for $g_{i+1}(d')$.


  \end{itemize}
  The construction of $\Jmc$ and $h$ is finished by setting
  $h = \bigcup_{i\geqslant 0} g_i$ and
  $\Jmc' = \bigcup_{i\geqslant 0}\Jmc_i$. 

  \par\medskip\noindent \textbf{{\boldmath ``$\Leftarrow$''.}}~ 
  Suppose there is a $\Jmc\in\MOD_{\omega}(\Tmc,t_0)$ with $\Imc\to_\sig\Jmc$.
  It suffices to show $\Jmc\to^{\mn{fin}} \Imc_{\Tmc,t_0}$. Let
  $(\Jmc_0,d_0),(\Jmc_1,d_1),\ldots$ be the sequence of rooted weakly
  tree-shaped interpretations whose limit is $\Jmc$.  We verify
  the following claim, which implies $\Jmc\rightarrow^{\mn{fin}}
  \Imc_{\Tmc,t_0}$.
  
  \par\smallskip\noindent
  \textbf{Claim.}~ For all $i \geq 0$, we have:
  \begin{enumerate}[(i)]

    \item there is an $e_0\in\Delta^{\Imc_{\Tmc,t_0}}$ with
      $\mn{tp}_{\Jmc_i}(d_i) = \mn{tp}_{\Imc_{\Tmc,t_0}}(e_0)$;

    \item for all $e_0 \in \Delta^{\Imc_{\Tmc,t_0}}$ with
      $\mn{tp}_{\Jmc_i}(d_i)\subseteq \mn{tp}_{\Imc_{\Tmc,t_0}}(e_0)$, we have
      $(\Jmc_i, d_i) \rightarrow (\Imc_{\Tmc,t_0}, e_0)$.

  \end{enumerate}
  
  \par\smallskip\noindent We prove the claim by induction on $i$.  For
  $i=0$, Points~(i) and~(ii) follow from the definition of~$\Jmc_0$. 
  For the inductive step, consider $\Jmc_{i+1}$ and suppose
  \textbf{(R)} has been applied to some $d\in \Delta^{\Jmc_{i}}$ and
  $e\in \Delta^{\Imc_{\Tmc,t_0}}$. 
  
  For $i>0$, observe first that Point~(i) is trivially preserved when $d_{i+1}=d_i$.
  In case $d_{i+1}$ is the predecessor of $d_i$, it is preserved by
  the condition on the choice of $e$ in~\textbf{(R)}: $e$ has the same
  type as $d_i$ and the predecessor $e'$ of $e$ has the same type as
  $d_{i+1}$.

  For Point~(ii), we distinguish two cases.

  \smallskip\textit{Case 1.} Suppose the application of \textbf{(R)}
  has not added any predecessors to $d$.  In particular, we then have
  $d_{i+1}=d_i$. For Point~(ii), take any $e_0$ with
  $\mn{tp}_{\Jmc_i}(d_{i+1})\subseteq \mn{tp}_{\Imc_{\Tmc,t_0}}(e_0)
  $. As $d_i=d_{i+1}$, induction hypothesis implies that there is a
  homomorphism $h : (\Jmc_i,d_{i+1})\to(\Imc_{\Tmc,t_0},e_0)$.  We
  extend $h$ to the domain of $\Jmc_{i+1}$ by doing the following for
  each newly added successor $d'$ of $d$.

  Let $\mn{tp}_{\Jmc_{i+1}}(d)=t$ and $\mn{tp}_{\Jmc_{i+1}}(d')=t'$
  and $\rho=\{r\mid (d,d')\in r^{\Jmc_{i+1}}\}$. By the
  choice of $e$ in \textbf{(R)}, $e$ is of type $t$ and has a
  $\rho$-successor of type $t'$. By definition of the universal
  model, there is some $r\in\rho$ with $t\rightsquigarrow_r^\Tmc t'$
  and $\rho=\{s\mid \Tmc\models r\sqsubseteq s\}$. Denote with $\hat
  t=\mn{tp}_{\Imc_{\Tmc,t_0}}(h(d))$.  The definition of a
  homomorphism yields $t\subseteq \hat t$. Thus, there is $\hat
  t'\supseteq t'$ such that $\hat t\rightsquigarrow_{r}^\Tmc \hat t'$.
  By definition of the universal model, $h(d)$ has a $\rho$-successor
  of type $\hat t'$ or a $\rho$-predecessor of type $\hat t''$, for
  $\hat t''\supseteq \hat t'$. We extend $h$ by setting $h(d')$ to
  that predecessor or successor, respectively. 
  
  After the extension, $h$ witnesses $(\Jmc_{i+1},d_{i+1})\to (\Imc_{\Tmc,t_0},e_0)$.
 
  \smallskip\textit{Case 2.} Application of \textbf{(R)} has added a
  predecessor $d'$ to $d$. Then $d_i=d$ and $d_{i+1}=d'$. Let
  $t=\mn{tp}_{\Jmc_{i+1}}(d)$, $t'=\mn{tp}_{\Jmc_{i+1}}(d')$ and
  $\rho=\{r\mid (d',d)\in r^\Jmc\}$.  By construction of the universal
  model, there is $r\in\rho$ with $t'\rightsquigarrow_{r}^\Tmc t$ and
  $\rho=\{s\mid \Tmc\models r\sqsubseteq s\}$.  Let $e_0$ be as
  in~(ii), that is, $\hat t' := \mn{tp}_{\Imc_{\Tmc,t_0}}(e_0)
  \supseteq t'$.  We then have that $\hat t'\rightsquigarrow_{r} \hat
  t$ for some $\hat t\supseteq t$.  By definition of the universal
  model, $e_0$ has a $\rho$-successor of type $\hat t$ or a
  $\rho$-predecessor of type $\hat t''\supseteq t$. Let this element
  be $\overline e_0$.  By induction hypothesis, there is a
  homomorphism $h : (\Jmc_i,d)\to(\Imc_{\Tmc,t_0},\overline{e}_0)$.
  We extend $h$ by first setting $h(d')=e_0$ and then extending $h$ to
  all successors of $d$ as in Case~1.  
  
  It is not difficult to verify that $h$ witnesses
  $(\Jmc_{i+1},d_{i+1})\to (\Imc_{\Tmc,t_0},e_0)$.
\end{proof}
We now use Lemma~\ref{lem:bounded_hmph_via_can_omega} to develop the
mosaic based decision procedure that underlies the proof of
Theorem~\ref{cor:bounded_hmph_in_2EXPTIME}. The main idea is that
since we cannot compute $\mn{mod}_\omega(\Tmc,t_1)$ explicitly, we
decompose the interpretations in this class into $1$-neighborhoods, of
which there are only finitely many. Such a 1-neighborhood is then
represented by a mosaic, together with a decoration with sets of
$\Tmc_2$-types that can be homomorphically embedded into that
neighborhood. 

Let $\Tmc_1,\Tmc_2$ be
$\HornALCHIF$ TBoxes, 
$t_1$ a $\Tmc_1$-type, and \sig a signature.
We denote with $\mn{rol}(\Tmc_i)$ the set of all roles $r,r^-$ such
that the role name $r$ occurs in $\Tmc_i$ (possibly as an inverse
role). Moreover, for a set of roles $\rho$ and a signature \sig,
denote with $\rho|_\sig$ the restriction of $\rho$ to $\sig$-roles,
and with $\rho^-$ the set $\{r^-\mid r\in\rho\}$.

%

Formally, a \emph{mosaic} is a tuple
$M=(t^-,\rho,t,S,\ell)$ such that
$(t^-,\rho,t,S)=n_1^{\Imc_{\Tmc_1,t_1}}(d)$ for some $d\in
\Delta^{\Imc_{\Tmc_1,t_1}}$ 
and $\ell:\{t^-,t\}\cup S\to 2^{\mn{tp}(\Tmc_2)}$
satisfies the following condition:
\begin{enumerate}[(i)]

  \item[\textbf{(M)}] For all $\widehat t \in \ell(t)$ we have
    $\widehat t \cap \sig \subseteq t$ and, for all~$\widehat
    t'\in\mn{tp}(\Tmc_2)$ and $r\in\mn{rol}(\Tmc_2)$ with $\widehat
    t\rightsquigarrow_r^{\Tmc_2} \widehat{t'}$, one of the following
    holds for $\sigma=\{s\in\mn{rol}(\Tmc_2)\mid \Tmc_2\models r\sqsubseteq s\}$:
%
  %
  \begin{enumerate}

    \item $\sigma|_\sig=\emptyset$;

    \item $t^-\neq \bot$, for every $s\in \sigma|_\sig$ we have
      $s^-\in \rho$, and $\widehat{t'}\in \ell(t^-)$;

    \item there is $(\rho',t') \in S$ with $\widehat{t'}\in \ell(\rho',t')$ and
      \mbox{$\sigma|_\sig\subseteq \rho'$}.

  \end{enumerate}
\end{enumerate}
To ease notation, we use $t^-_M$ to denote $t^-$, $\rho_M$ to denote
$\rho$, and likewise for the other components of a mosaic $M$.
We further use $\Mmc$ to denote the set of all mosaics.

We now describe an elimination algorithm in the style of Pratt's type
elimination for PDL~\cite{Pratt79}, but working on mosaics rather than
on types; the existence of this algorithm will establish
Theorem~\ref{cor:bounded_hmph_in_2EXPTIME}.  Let $\Mmc' \subseteq
\Mmc$ be a set of mosaics.  An $M \in \Mmc'$ is \emph{good in} $\Mmc'$
if the following conditions are satisfied:
\begin{enumerate}[1.]

  \item for each $(\rho,t)\in S_M$, there is an 
    $N \in \Mmc'$ such that $(t_M,\rho,t)=(t^-_N,\rho_N,t_N)$,
    $\ell_M(\rho,t)=\ell_{N}(t_{N})$, and
    $\ell_M(t_M)=\ell_{N}(t^-_{N})$;

  \item if $t^-_M \neq \bot$, there is $N \in \Mmc'$
    with
    $(\rho_M,t_M)\in S_N$, $t^-_M=t_N$, $\ell_M(t^-_M) = \ell_N(t_N)$, and
    $\ell_M(t_M)=\ell_{N}(\rho_{M},t_M)$.

\end{enumerate}
Our algorithm computes the sequence $\Mmc_0,\Mmc_1,\dots$ that starts
with $\Mmc_0=\Mmc$ and where $\Mmc_{i+1}$ is obtained from $\Mmc_i$ by
removing all mosaics that are not good in $\Mmc_i$.  This sequence
eventually stabilizes, say at $\Mmc_p$. The following lemma
establishes the central property of the elimination algorithm.
\begin{lemma}
  \label{lem:mosaics}
  Let $t_2 \in \mn{tp}(\Tmc_2)$.
  Then the following two statements are equivalent:
  \begin{enumerate}
    \item
      there is a $\Jmc \in \MOD_\omega(\Tmc_1,t_1)$ such that
      $\Imc_{\Tmc_2,t_2}|^{\mn{con}}_\sig \rightarrow_\sig \Jmc$;
    \item
      $\Mmc_p$ contains a mosaic $M$ with $t_2 \in \ell_M(t_M)$.
  \end{enumerate}
\end{lemma}
\begin{proof}
  \textbf{{\boldmath $1 \Rightarrow 2$.}}~ Let $h$ be an $\sig$-homomorphism
  from $\Imc_{\Tmc_2,t_2}|_\sig^{\mn{con}}$ to some $\Jmc \in
  \MOD_\omega(\Tmc_1,t_1)$. For every $d\in \Delta^{\Jmc}$, denote
  with $T_h(d)$ the set of all types mapped to $d$ by $h$, that is,
  $$T_h(d) = \big\{\mn{tp}_{\Imc_{\Tmc_2,t_2}}(e)
  \mid h(e) = d,~ e \in
  \Delta^{\Imc_{\Tmc_2,t_2}|^\mn{con}_\sig}\:\!\big\}.$$
  For every element $d\in \Delta^\Jmc$, we define a tuple
  $M(d)=(t^-,\rho,t,S,\ell)$ as follows:
  \begin{itemize} 

    \item $(t^-,\rho,t,S)=n_1^\Jmc(d)$;

    \item $\ell(t)=T_h(d)$;
    
    \item if there is a predecessor $d'$ of $d$, then
      $\ell(t^-)=T_h(d')$; otherwise, set $\ell(t^-)=\emptyset$
      (or any other value);

    \item for every $(\rho,t')\in S$, let $d'$ be the $\rho$-successor
      of $d$ with $\mn{tp}_\Jmc(d')=t'$, and set $\ell(\rho',t')=T_h(d')$.

      
  \end{itemize}
  It is easy to verify that every $M(d)$ obtained
  in this way is indeed a mosaic. In particular, it follows from the
  definition of $\Jmc$ that
  $(t^-,\rho,t,S)=n_1^{\Imc_{\Tmc_1,t_1}}(d')$ for some $d'\in
  \Delta^{\Imc_{\Tmc_1,t_1}}$. Moreover, by the fact that $h$ is a
  homomorphism, Condition~\textbf{(M)} is satisfied.

  Let $\Mmc(\Jmc) = \{M(d) \mid d \in \Delta^\Jmc\}$.  It follows from
  the construction that all mosaics in $\Mmc(\Jmc)$ are good in
  $\Mmc(\Jmc)$; hence $\Mmc(\Jmc)\subseteq\Mmc_p$.  Finally, let $d_0$
  be the root of $\Imc_{\Tmc_2,t_2}$ and let $M:=M(h(d_0))$. Then
  $t_2\in \ell_{M}(t_{M})$ and thus Point~2 of Lemma~\ref{lem:mosaics}
  is satisfied.

  \par\medskip\noindent \textbf{{\boldmath $2 \Leftarrow 1$.}}~ Assume
  that $\Mmc_p$ contains a mosaic $M$ with $t_2 \in \ell_M(t_{M})$. We
  construct an interpretation $\Jmc$ by stitching together mosaics.
  Throughout the construction, we maintain a partial function $q :
  \Delta^{\Jmc} \to \Mmc_p$ that records for each domain element of
  $\Jmc$ the mosaic that gave rise to it. We also maintain the
  following invariant:
  \begin{equation}
    \text{If $q(d)=(t^-,\rho,t,S,\ell)$, then
    }n_1^\Jmc(d)=(t^-,\rho,t,S).\tag{$\ast$}
    \label{eq:inv-mosaics}
  \end{equation}
  We start with defining $\Jmc$ as the interpretation corresponding to
  the 1-neighborhood represented by $M$ (in the obvious way), and
  define $q(e_0)=M$, where $e_0$ is the domain element that
  corresponds to the center of 
  that 1-neighborhood.  By definition, the
  invariant~\eqref{eq:inv-mosaics} is satisfied.  Then extend \Jmc by
  doing the following exhaustively in a fair way: 
  Choose some $d\in \Jmc$ such that $q(d)$ is undefined, and
  \begin{itemize}


  \item[(b)] If $d$ has a predecessor $d'$ such that $q(d')=M'$ then, due
      to~\eqref{eq:inv-mosaics}, there is $(\rho,t)\in S_{M'}$ such
      that $d$ is the $\rho$-successor of $d'$ in \Jmc and $\mn{tp}_\Jmc(d)=t$. Let
      $N\in\Mmc_p$ be the mosaic that exists according to Condition~1
      of being good for $(\rho,t)\in S_{M'}$. Then extend $\Jmc$ such
      that $n_1^\Jmc(d)=(t^-_N,\rho_N,t_N,S_N)$ and set $q(d)=N$. 

    \item[(c)] If $d$ has a successor $d'$ such that $q(d')=M'$ then, due
      to~\eqref{eq:inv-mosaics}, we know that
      $t^-_{M'}=\mn{tp}_\Jmc(d)\neq \bot$. Let $N\in\Mmc_p$
      be the mosaic that exists according to Condition~2 of being
      good.  Then extend $\Jmc$ such that
      $n_1^\Jmc(d)=(t^-_N,\rho_N,t_N,S_N)$ and set $q(d)=N$. 

  \end{itemize}
  It is immediate from the construction that this
  preserves~\eqref{eq:inv-mosaics}, and that one of~(a) and~(b)
  always applies. Moreover, by construction, any interpretation $\Jmc$
  obtained in the limit of this process is an element of
  $\MOD_\omega(\Tmc_1,t_1)$.  

  It thus remains to construct an $\sig$-homomorphism $h$ that
  witnesses $\Imc_{\Tmc_2, t_2}|^{\mn{con}}_\sig \rightarrow_\sig
  \Jmc$. We construct $h$ step by step, essentially following the
  construction of $\Imc_{\Tmc_2, t_2}$, maintaining the invariant:
  \begin{equation}
    \text{If $h(d)$ is defined, then
      $\mn{tp}_{\Imc_{\Tmc_2,t_2}}(d)\in
      \ell_{q(h(d))}(t_{q(h(d))})$.}\tag{$\dagger$}
      \label{eq:inv-mosaics2}
  \end{equation}
  Let $d_0$ be the root of $\Imc_{\Tmc_2,t_2}$. We start with setting
  $h(d_0)=e_0$, where $e_0$ is as above. By the assumption that
  $t_2\in \ell_M(t_M)$, invariant~\eqref{eq:inv-mosaics2} is
  satisfied. Now, exhaustively apply the following step:

  Choose $d\in\Delta^{\Imc_{\Tmc_2,t_2}|^\mn{con}_\sig}$ such that
  $h(d)$ is not defined but $h(d')=e$ is defined for the predecessor
  $d'$ of $d$. Let $t=\mn{tp}_{\Imc_{\Tmc_2,t_2}}(d)$,
  $t'=\mn{tp}_{\Imc_{\Tmc_2,t_2}}(d')$, and $M'=q(d')$. By definition
  of $\Imc_{\Tmc_2,t_2}$, $t'\rightsquigarrow_r^{\Tmc_2} t$ for some
  $r\in\mn{rol}(\Tmc_2)$. Let $\sigma=\{s\mid \Tmc\models r\sqsubseteq
  s\}$. By Invariant~\eqref{eq:inv-mosaics2}, $t'\in
  \ell_{M'}(t_{M'})$. Thus, one of~(a) to~(c) in Condition~\textbf{(M)}
  applies.  Condition~(a) in impossible since
  $\sigma|_\sig\neq\emptyset$ as
  $d,d'\in\Delta^{\Imc_{\Tmc_2,t_2}|^{\mn{con}}_\sig}$.  In case
  of~(b), we extend $h$ by setting $h(d)$ to the predecessor of
  $h(d')$. In case of~(c), we extend $h$ by setting $h(d)$ to the
  according successor of $h(d')$.

  Note that $h$ extended like this satisfies the homomorphism
  conditions and preserves~\eqref{eq:inv-mosaics2} due to the
  conditions in~(b) and~(c).
\end{proof}
%
%
%
We now finish the proof of Theorem~\ref{cor:bounded_hmph_in_2EXPTIME}.
Due to Lemma~\ref{lem:bounded_hmph_via_can_omega}, we can decide
$\Imc_{\Tmc_2,t_2}|^{\mn{con}}_\sig \rightarrow^\mn{fin}_{\sig}
\Imc_{\Tmc_1,t_1}$ by checking whether there is a $\Jmc \in
\MOD_\omega(\Tmc_1,t_1)$ with $\Imc_{\Tmc_2,t_2}|^{\mn{con}}_\sig
\rightarrow_\sig \Jmc$.  By Lemma~\ref{lem:mosaics}, this can be done
by constructing the described sequence of mosaics
$\Mmc_0,\Mmc_1,\dots\Mmc_p$ and checking whether $\Mmc_p$ contains a
mosaic $M$ with $t_2 \in T_M$.

The time bound stated in
Theorem~\ref{cor:bounded_hmph_in_2EXPTIME} is a consequence of the
following observations. We can compute $\Mmc_0$ as follows. First,
enumerate all possible tuples $(t^-,\rho,t,S)$ with the number of
elements in $S$ bounded by $|\Tmc_1|$ (since the outdegree of
$\Imc_{\Tmc_1,t_1}$ is bounded by $|\Tmc_1|$). Then remove those
which do not correspond to $1$-neighborhoods in $\Imc_{\Tmc_1,t_1}$.
Note that this requires computing $\rightsquigarrow_r$ for every $r$,
which can be done in time $2^{p(|\Tmc_1|)}$, $p$ a polynomial.
Moreover, we have to perform reachability checks to verify that the
tuple $(t^-,\rho,t,S)$ indeed occurs as the $1$-neighborhood in
$\Imc_{\Tmc_1,t_1}$. Based on this set of all tuples, we consider all
possible combinations with labelings $\ell$ and remove those tuples
$(t^-,\rho,t,S,\ell)$ which do not satisfy~\textbf{(M)}. It is routine
to verify that there are at most $2^{2^{q(|\Tmc_2|\log|\Tmc_1|)}}$
many mosaics, for some polynomial $q$. It remains to note that
Conditions~1 and~2 of a mosaic being good can be checked in time
polynomial in the size of the current set of mosaics $\Mmc_i$ and that
the maximal number of iterations is $|\Mmc_0|$.

\subsection{Automata Construction}
\label{sec:automata_construction}

We now prove the upper bound on CQ entailment stated in
Theorem~\ref{thm:complexity} based on the characterization provided by
Theorem~\ref{thm:homs}, using the decision procedure asserted by
Theorem~\ref{cor:bounded_hmph_in_2EXPTIME} as a black box.  Our main
tool are alternating two-way tree automata with counting (\ata) which
extend alternating automata on unranked infinite trees as used for
example by~\citeauthor{GradelW99}~\citeyear{GradelW99} with the ability to count.

A \emph{tree} is a
non-empty (potentially infinite) set of words $T \subseteq (\Nbbm
\setminus 0)^*$ closed under prefixes.  
We assume that trees
are finitely branching, that is, for every $w \in T$, the set $\{ i >0
\mid w \cdot i \in T \}$ is finite.
For $w \in (\Nbbm \setminus 0)^*$, set $w \cdot
0 := w$. For $w=n_0n_1 \cdots n_k$, we set $w
\cdot -1 := n_0 \cdots n_{k-1}$.  For an alphabet $\Theta$, a
\emph{$\Theta$-labeled tree} is a pair $(T,L)$ with $T$ a tree and
$L:T \rightarrow \Theta$ a node labeling function.

A \emph{\ata} is a tuple $\Amf = (Q,\Theta,q_0,\delta,\Omega)$ where
$Q$ is a finite set of {\em states}, $\Theta$ is the {\em input
  alphabet}, $q_0\in Q$ is the {\em initial state}, $\delta$ is a {\em
  transition function}, and $\Omega:Q\to \mathbb{N}$ is a {\em
  priority function}.  The transition function $\delta$ maps every
state $q$ and input letter $a \in \Theta$ to a positive Boolean
formula $\delta(q,a)$ over the truth constants $\mn{true}$ and
$\mn{false}$ and \emph{transition atoms} of the form $q$,
$\Diamond^- q$, $\Box^- q$, $\Diamond_n q$ and $\Box_n q$.
A
transition $q$ expresses that a copy of \Amf is sent to the current
node in state $q$; $\Diamond^- q$ means that a copy is sent in
state $q$ to the predecessor node, which is required to exist; $\Box^- q$
means the same except that the predecessor node is not required to
exist; $\Diamond_n q$ means that a copy of $q$ is sent to $n$
successors and $\Box_n q$ means that a copy of $q$ is sent to
all but $n$ successors.
We use $\Diamond q$ and $\Box q$ to abbreviate $\Diamond_{1} q$ and $\Box_0
q$, respectively.
%
%
The semantics of \ata is given in terms of runs as usual, details are
in Appendix~\ref{appx:automata}. We use $L(\Amf)$ to denote the set of
$\Theta$-labeled trees accepted by $\Amf$.
It is not hard to show that \ata are closed under
intersection and that the intersection automaton can be 
constructed in polynomial time \cite<using techniques from, e.g.,>{tata2007}.
The \emph{emptiness
  problem} for \ata means to decide, given a \ata \Amf, whether
$L(\Amf)=\emptyset$. We assume here that the numbers in transitions of
the form $\Diamond_n q$ and $\Box_n q$ are encoded in
unary.\footnote{For binary encoding, Theorem~\ref{thm:nonemptiness}
below includes additionally an exponential dependence in the
\emph{size} of the largest number appearing in $\delta$.}
The following result is obtained via reduction to the
emptiness problem of the more standard alternating parity tree
automata on ranked trees \cite{Vardi98}, see
  Appendix~\ref{appx:emptiness} for details.
\begin{restatable}{theorem}{nonemptiness}\label{thm:nonemptiness}
  The emptiness problem for \ata is in \ExpTime. More precisely, it
  can be solved in time single exponential in
  the number of states and the maximal priority, and polynomial in all
  other components.
\end{restatable}
To prove the \TwoExpTime upper bound for CQ entailment stated in
Theorem~\ref{thm:complexity}, it thus suffices to show that given
\HornALCHIF TBoxes $\Tmc_1$ and $\Tmc_2$ and signatures \Abf and \Qbf,
one can construct in time $2^{2^{p(|\Tmc_2| \mn{log}|\Tmc_1|)}}$ a
\ata \Amf with single exponentially many states and both maximal
priority and maximal occurring number one such that $L(\Amf) \neq
\emptyset$ iff $\Tmc_1
\not\models_{\sigmaabox,\sigmaquery}^{\textup{CQ}} \Tmc_2$. This is
what we do in the following.

The desired {\ata} runs over $\Theta$-labeled trees with
$\Theta=2^{\Theta_0}\times 2^{\Theta_1}\times 2^{\Theta_2}$ where
\begin{itemize}
\item 
$\Theta_0 = \sigmaabox\cup \{r^-\mid r\in\sigmaabox\}$;
\item $\Theta_i = \mn{sig}(\Tmc_i)\cup \{r^-\mid r\in \mn{sig}(\Tmc_i)\}$
for $i=1,2$.
\end{itemize}
 For a $\Theta$-labeled tree $(T,L)$, we use $L_i$, $i\in \{0,1,2\}$ to
refer to the $i$-th component of $L$, that is,
$L(n)=(L_0(n),L_1(n),L_2(n))$, for all $n\in T$.
The component $L_0$ represents a (possibly infinite) ABox
$$\Amc_L = \{ A(n) \mid A\in L_0(n)\} \cup \{ r(n\cdot -1,n)\mid n\neq
\varepsilon, r\in L_0(n)\},
$$
where $r^-(a,b)$ is identified with $r(b,a)$. Note that even when
$\Amc_L$ is finite it is not necessarily a tree-shaped ABox as it
might have multi-edges and be disconnected. Components $L_1$ and $L_2$
represent interpretations $\Imc_{L,1}=(T,\cdot^{\Imc_{L,1}})$ and
$\Imc_{L,2}=(\mn{ind}(\Amc_L),\cdot^{\Imc_{L,2}})$, where for $i\in \{1,2\}$:
\begin{align*}
  A^{\Imc_{L,i}} &\!\!\;=\!\!\; \{n \mid A\in L_i(n)\}\\
  r^{\Imc_{L,i}} &\!\!\;=\!\!\; \{ (n,n\!\!\;\cdot\!\!\; -1) \mid
  r^-\!\in\! L_i(n)\} \cup \{(n\!\!\;\cdot\!\!\;-1,n)\mid r\!\in\!
  L_i(n)\}
\end{align*}
%

Now, let $\Tmc_1$ and $\Tmc_2$ be \HornALCHIF TBoxes and let \Abf, \Qbf, and 
\Sbf be signatures. The claimed \ata \Amf is constructed as the
intersection of the four {\ata}s $\Amf_1,\Amf_2,\Amf_3,\Amf_4$
provided by the following lemma.
\begin{lemma}
  \label{lem:automata}
  There are {\ata}s $\Amf_1,\Amf_2,\Amf_3,\Amf_4$ such that:
  \begin{itemize}

    \item[--] $\Amf_1$ accepts $(T,L)$ iff $\Amc_L$ is finite,
      tree-shaped, and $\varepsilon \in \mn{ind}(\Amc_L)$;

    \item[--] $\Amf_2$ accepts $(T,L)$ iff $\Imc_{L,1}$ is a model of
      $\Amc_L$ and $\Tmc_1$;

    \item[--] $\Amf_3$ accepts $(T,L)$ iff $\Amc_L$ is consistent with
      $\Tmc_2$, and $\Imc_{L,2}$ is
      $\Imc_{\Tmc_2,\Amc_L}$ restricted to $\mn{ind}(\Amc_L)$;
      
    \item[--] $\Amf_4$ accepts $(T,L)$ iff Conditions~(1) and (2) from
      Theorem~\ref{thm:homs} are not both satisfied, when
      $\Imc_{\Tmc_2,\Amc_L}$ is replaced with $\Imc_{L,2}$.

  \end{itemize}
  The number of states of $\Amf_1$ and $\Amf_2$ is polynomial in
  $|\Tmc_1|$ (and independent of $\Tmc_2$); the number of states of
  $\Amf_3$ is polynomial in $|\Tmc_2|$ (and independent of $\Tmc_1$),
  and the number of states of $\Amf_4,$ is exponential in
  $|\Tmc_2|$ (and independent of $\Tmc_1$). All automata can be
  constructed in time $2^{2^{p(|\Tmc_2| \mn{log}|\Tmc_1|)}}$, $p$ a
  polynomial and have maximal priority of one.
\end{lemma}

It can be verified that indeed $L(\Amf)\neq\emptyset$ iff $\Tmc_1
\not\models_{\sigmaabox,\sigmaquery}^{\textup{CQ}} \Tmc_2$.
The rest of this section is devoted to proving
Lemma~\ref{lem:automata}.  

\subsubsection*{Automata $\Amf_1$ and $\Amf_2$.}

The construction of the automaton $\Amf_1$ is straightforward and left
to the reader. Also $\Amf_2$ is easy to construct. It checks that
$\Imc_{L,1}$ is a model of $\Amc_L$ by synchronizing the $L_0$ and
$L_1$ components of the input tree and that all statements in $\Tmc_1$
are satisfied by $\Imc_{L,1}$ by imposing constraints on the $L_1$
component. The latter is particularly simple since $\Tmc_1$ is in
normal form. For the sake of completeness and as a warm up, we present
the details.  Define $\Amf_2=(Q_2,\Theta,q_0,\delta_2,\Omega_2)$ where
\begin{align*}
  Q_2 ={} & \{q_0,q_\Amc\}\cup \{q_\alpha\mid \alpha\in
  \Tmc_1\}\cup\{q_\rho,\overline q_{\rho}\mid \rho\in\Theta_1\}\cup {}
  \\
  & \{q_{r,B},q_{r,B}^\downarrow, \overline q_{r,B}, \overline
  q_{r,B}^\downarrow \mid \exists r.B \text{ occurs in } \Tmc_1\},
\end{align*}
and $\Omega_2$ assigns $0$ to all states. 
Here, the transition function
$\delta_2$ is given as follows. For $\sigma=(L_0,L_1,L_2)$, set
$$ \begin{array}{rcll}
  \delta_2(q_0,\sigma) & = & \Box q_0 \wedge q_\Amc \wedge
  \bigwedge_{\alpha\in\Tmc_1}q_\alpha \\[1mm]
  \delta_2(q_\Amc,\sigma)&=& \displaystyle\bigwedge_{\rho\in L_0}
  q_\rho \\
  \delta_2(q_{\mn{func}(r)},\sigma) & = & (q_{r^-}\wedge \Box\overline
  q_r)\vee(\overline q_{r^-}\wedge \Box_1 \overline q_r)\\[1mm]
  %
  \delta_2(q_{r\sqsubseteq s},\sigma) & = & \overline q_r\vee
  q_s\\[1mm]
  %
  \delta_2(q_{A_1\sqcap A_2\sqsubseteq B},\sigma) &=& \overline
  q_{A_1}\vee \overline q_{ A_2 }\vee q_B\\[1mm]
  %
  \delta_2(q_{A\sqsubseteq \bot},\sigma) &=& \overline q_A\\[1mm]
  \delta_2(q_{\top\sqsubseteq A},\sigma) &=& q_A\\[1mm]
  %
  \delta_2(q_{A\sqsubseteq \exists r.B},\sigma) &=& \overline q_A\vee
  q_{r,B}\\[1mm]
  %
  \delta_2(q_{A\sqsubseteq \forall r.B},\sigma) &=& 
  q_{B} \vee \overline q_{r^-,A}\\[1mm]
  %
  \delta_2(q_{r,B},\sigma) &=&  \Diamond q_{r,B}^\downarrow \vee
  (q_{r^-}\wedge \Diamond^- q_B) \\[1mm]
  %
  \delta_2(\overline q_{r,B}, \sigma) &=& \Box \overline
  q_{r,B}^\downarrow \wedge (\overline q_{r^-} \vee {\Box^-} \overline
  q_B) \\[1mm]
  \delta_2(q_{r,B}^\downarrow, \sigma) &=& q_r\wedge q_B \\[1mm]
  %
  %
  \delta_2(\overline q_{r,B}^\downarrow, \sigma) &=&\overline q_r\vee
  \overline q_B.
  %
\end{array}$$
We further set for all $\rho\in\Theta_1$: $$\begin{array}{rcll}
  \delta_2(q_\rho, \sigma) &=&  \left\{\begin{array}{ll} \mn{true} &
    \text{if } \rho\in L_1 \\ \mn{false} & \text{if }\rho\notin L_1
  \end{array} \right. \\[4mm]
  \delta_2(\overline q_\rho, \sigma) &=&  \left\{\begin{array}{ll}
    \mn{true} & \text{if } \rho\notin L_1 \\ \mn{false} & \text{if
    $\rho\in L_1$.} \end{array} \right. 
  %
\end{array}$$

\subsubsection*{Automaton $\Amf_4$.}

We next consider the automaton $\Amf_4$ as it is the
most interesting and crucial ingredient to the construction of \Amf.
Recall that $\Amf_4$ has to make sure that Conditions~(1) and (2) from
Theorem~\ref{thm:homs} are not both satisfied. It achieves this by
verifying that its input $(T,L)$ is such that there is an
$n\in\mn{ind}(\Amc_L)$ for which one of the following conditions
holds,
where $\Imc_{\Tmc_2,L_2(n) \cap \NC}$ is the universal model of the 
$\Tmc_2$-type $L_2(n) \cap \NC$ and $\Tmc_2$:
\begin{enumerate}

  \item there is a $\sigmaquery$-role $r$ and an $n'\in\mn{ind}(\Amc_L)$ such that
    $(n,n')\in r^{\Imc_{L,2}}$, but $(n,n')\notin r^{\Imc_{L,1}}$;

  \item there is no $\Qbf$-homomorphism from $\Imc_{\Tmc_2,L_2(n) \cap
    \NC}|^{\mn{con}}_\Qbf$ to 
    $\Imc_{L,1}$ that maps the root of
    $\Imc_{\Tmc_2,L_2(n) \cap \NC}$ to~$n$;

  \item there is a $\sigmaquery$-subtree $\Imc$ of
    $\Imc_{\Tmc_2,L_2(n) \cap \NC}$ satisfying the following two conditions.
	
    \begin{enumerate}

      \item $\Imc\not\rightarrow_\sigmaquery \Imc_{L,1}$;

      \item
	$\Imc\not\rightarrow_\sigmaquery^{\mn{fin}}\Imc_{\Tmc_1,\mn{tp}_{\Imc_{L,1}}(m)}$,
	for all $m$ with $L_0(m)\neq \emptyset$.

    \end{enumerate}

\end{enumerate}
For Condition~3, note that every \Qbf-subtree $\Imc$ of
$\Imc_{\Tmc_2,L_2(n) \cap \NC}$ is of the form
$\Imc_{\Tmc_2,t'}|^{\mn{con}}_\Qbf$ with $\Imc_{\Tmc_2,t'}$ the
universal model of the $\Tmc_2$-type $t'$ satisfied
at the root of $\Imc$ and $\Tmc_2$. For a $\Tmc_2$-type $t$, we use
$R_\sigmaquery(t)$ to denote the set of all types that are realized at
the root of a $\sigmaquery$-subtree in the universal model
$\Imc_{\Tmc_2,t}$ of $t$ and $\Tmc_2$. Then, the \Qbf-subtrees $\Imc$
in Condition~3b are exactly the interpretations
$\Imc_{\Tmc_2,t}|^{\mn{con}}_\Qbf$, $t \in R_\sigmaquery(L_2(n) \cap
\NC)$. Also note that Condition~3b is exactly the question addressed
by Theorem~\ref{cor:bounded_hmph_in_2EXPTIME}.

The essence of Conditions~2 and~3a is to ensure that there is no
\Qbf-homomorphism from $\Imc_{\Tmc_2,t}|^{\mn{con}}_\Qbf$ to
$\Imc_{L,1}$ that maps the root of $\Imc_{\Tmc_2,t}$ to $n$, for some
$\Tmc_2$-type~$t$ and $n \in \mn{ind}(\Amc_L)$ (for Condition~3, one
has to consider all $n \in \mn{ind}(\Amc_L)$). The automaton verifies
this by making sure that one of the following is true:
\begin{itemize}

\item[(i)] there is a concept name from \Qbf that is true at the root of
  $\Imc_{\Tmc_2,t}$ but not at $n$ in $\Imc_{L,1}$;

\item[(ii)] there is a $\Tmc_2$-type $t'$ and a set of
  $\mn{sig}(\Tmc_2)$-roles $\rho$ with
  $t\rightsquigarrow^{\Tmc_2}_{\rho}t'$ such that there is no
  \Qbf-homomorphism from $\Imc_{\Tmc_2,t'}|^{\mn{con}}_\Qbf$ to
  $\Imc_{L,1}$ that maps the root to a $\rho$-neighbor of $n$.



\end{itemize} 
Note that this is sufficient since the subtrees of $\Imc_{\Tmc_2,t}$
rooted at $\rho$-successors of the root are exactly the models
$\Imc_{\Tmc_2,t'}$,
$t\rightsquigarrow^{\Tmc_2}_{\rho}t'$. Condition~(ii) is recursive in
the sense that we are faced with the same conditions that we started
with, only for a different $\Tmc_2$-type and element of
$\mn{ind}(\Amc)$. We need to ensure that the recursion terminates,
which is achieved by assigning appropriate priorities to states.

%
%
The above is implemented as follows. Define
$\Amf_4=(Q_4,\Theta,q_0,\delta_4,\Omega_4)$ where
\begin{align*}
  Q_4={} & \{q_0,q_1\} \cup
  \{ \overline{q}^t_2, \overline{q}^t_3,\overline{q}^t_{3b}\mid t\in
    \mn{tp}(\Tmc_2)\}\cup{}\\
    & \{ \overline{q}^{\rho,t}_2,\overline{q}^{\rho,t,\downarrow}_2 \mid t\in
    \mn{tp}(\Tmc_2), \rho\text{ set of $\mn{sig}(\Tmc_2)$-roles}\},
\end{align*}
and $\Omega_4$ assigns zero to all states, except for states of the
form $q_0$ and  $q^t_2$, $t\in \mn{tp}(\Tmc_2)$, to which it assigns one.  
%
%
For a $\Tmc_2$-type $t$ and set of roles $\rho$, we use
$t|_\sigmaquery$ and $\rho|_\sigmaquery$ to denote the restriction of
$t$ and $\rho$ to (the elements that only use symbols from) $\sigmaquery$.
For each
$\sigma=(L_0,L_1,L_2)$, $\delta_4$ contains the following
transitions:
$$ \begin{array}[h]{@{}r@{~}c@{~}ll@{}}
  \delta_4(q_0,\sigma) &=& \left\{\begin{array}{ll} \Diamond q_0 \vee
q_1\vee \overline{q}_2^{L_2 \cap \NC}\vee
  \bigvee_{t'\in\mn{R}_\sigmaquery(L_2 \cap \NC)} \overline{q}^{t'}_3
    & \text{if }L_0\neq \emptyset \\  \mn{false} & \text{otherwise}
  \end{array} \right. \\[4mm]
  %
  \delta_4(q_1,\sigma) &=& \left\{\begin{array}{ll@{}} \mn{true} &
    \text{if $L_2\setminus L_1$ contains a $\sigmaquery$-role } \\
    \mn{false} & \text{otherwise} \end{array}\right. \\[4mm]
  \delta_4(\overline{q}^t_2,\sigma) &=& \left\{\begin{array}{ll} \mn{true} &
    \text{if }t|_\sigmaquery\not\subseteq L_1 \\
    \bigvee_{t\rightsquigarrow^{\Tmc_2}_{\rho}t'} \overline{q}^{\rho,t'}_2 &
      \text{otherwise} \end{array} \right. \\[4mm]
    \delta_4( \overline{q}^{\rho,t}_2, \sigma ) &=& \left\{\begin{array}{ll} 
      \Box \overline{q}^{\rho,t,\downarrow}_2 & \text{if }\rho^-|_\sigmaquery\not\subseteq L_1
      \\  
      \Box \overline{q}^{\rho,t,\downarrow}_2 \wedge \Diamond^- \overline{q}^t_2 & \text{if
    $\rho^-|_\sigmaquery\subseteq L_1$} \end{array} \right.  \\[5mm] 
  \delta_4( \overline{q}^{\rho,t,\downarrow}_2, \sigma) &=&
  \left\{\begin{array}{ll} \mn{true} & \text{if
    }\rho|_\sigmaquery\not\subseteq L_1 \\ \overline{q}^t_2 & \text{if }\rho|_\sigmaquery \subseteq L_1 \end{array} \right.  \\[4mm] 
  \delta_4(\overline{q}^t_3,\sigma) &=& \Box \overline{q}^t_3 \wedge \Box^- \overline{q}^t_3 \wedge
  \overline{q}^t_2\wedge \overline{q}^t_{3b} \\[1mm]
  \delta_4(\overline{q}^t_{3b},\sigma) &=& \left\{\begin{array}{ll@{}} 
    \mn{true} & \text{if } L_0 = \emptyset\text{ or
    }\Imc_{\Tmc_2,t}|^{\mn{con}}_\sigmaquery\not\rightarrow^\mn{fin}_\sigmaquery
    \Imc_{\Tmc_1,L_1\cap \mn{N_C}}\\ \mn{false} & \text{otherwise} 
  \end{array} \right.
\end{array}$$
Note that states with index $i \in \{1,2,3,3b\}$ are used to enforce
Condition~$i$ from above. Also note that bounded homomorphisms are not
handled directly by the automaton. The condition
$\Imc_{\Tmc_2,t}|^{\mn{con}}_\sigmaquery\not\rightarrow^\mn{fin}_\sigmaquery
\Imc_{\Tmc_1,L_1\cap \mn{N_C}}$ from the last transition is only
needed during the construction of the automaton, and it can be decided
in the required time due to
Theorem~\ref{cor:bounded_hmph_in_2EXPTIME}. The sets
$R_\sigmaquery(t)$
can clearly be computed in single exponential time.

\subsubsection*{Automaton $\Amf_3$.}

Recall that $\Amf_3$ has to accept an input $(T,L)$ iff $\Amc_L$ is
consistent with $\Tmc_2$, and $\Imc_{L,2}$ is $\Imc_{\Tmc_2,\Amc_L}$
restricted to $\mn{ind}(\Amc_L)$. The construction requires a few
preliminaries, in particular a characterization of whether $\Tmc,\Amc
\models A(a)$ in terms of derivation trees. Similar yet slightly
different characterizations have been used before~\cite<e.g.
in>{BLW13}.

\smallskip Let $\Tmc$ be a \HornALCHIF TBox and \Amc an ABox.  A
\emph{derivation tree} for an assertion $A_0(a_0)$ in \Amc w.r.t.\
\Tmc with $A_0 \in \NC$ is a finite $\mn{ind}(\Amc) \times
\NC$-labeled tree $(T,V)$ that satisfies the following conditions:
\begin{enumerate}

\item $V(\varepsilon)=(a_0,A_0)$;

\item if $V(x)=(a,A)$ and neither $A(a) \notin \Amc$ nor $\top
  \sqsubseteq A \in \Tmc$,
  then one of the following holds:
  \begin{enumerate}[(i)]

  \item $x$ has successors $y_1,\dots,y_k$, $k \geq 1$ with
    $V(y_i)=(a,A_i)$ for $1 \leq i \leq k$ and $\Tmc \models A_1
    \sqcap \cdots \sqcap A_k \sqsubseteq A$;

  \item $x$ has a single successor $y$ with $V(y)=(b,B)$ and there
    is an $B \sqsubseteq \forall r.A \in \Tmc$ and an $s(b,a) \in
    \Amc$ such that
    $\Tmc \models s \sqsubseteq r$;

  \item $x$ has a single successor $y$ with $V(y)=(b,B)$ and there
    is a $B \sqsubseteq \exists r . A \in \Tmc$ such that $r(b,a) \in
    \Amc$ and $\mn{func}(r) \in \Tmc$. 

\end{enumerate}

\end{enumerate}
Item~(i) of Point~2 above requires $\Tmc \models A_1\sqcap
\dotsb\sqcap A_n \sqsubseteq A$ instead of $A_1 \sqcap A_2 \sqsubseteq
A \in \Tmc$ to `shortcut' parts of the universal model that are
generated by existential restrictions. In fact, elements of the
universal
model generated by existential restrictions do never appear in a
derivation tree.
%
The main properties of derivation trees are summarized in the
following lemma, proved in Appendix~\ref{appx:derivlem}.
\begin{restatable}{lemma}{lemderivationtrees} \label{lem:derivationtrees}
  Let \Tmc be a \HornALCHIF TBox and \Amc. Then
  \begin{enumerate}

  \item if \Amc is consistent with \Tmc, then $\Tmc,\Amc \models
    A_0(a_0)$ iff there is a derivation tree for $A_0(a_0)$ in \Amc
    w.r.t. \Tmc, for all assertions $A_0(a_0)$;

  \item \Amc is consistent with~\Tmc iff the following are
    satisfied:
  \begin{enumerate}

    \item the ABox $\Amc_a = \{ A(a) \mid \Tmc,\Amc \models A(a) \}$ is
    consistent with \Tmc, for all $a \in \mn{ind}(\Amc)$;

    \item the relation $\{(a,b) \mid s(a,b) \in \Amc\}$ is a partial
      function whenever $\mn{func}(s)\in\Tmc$.

  \end{enumerate}
\end{enumerate}
\end{restatable}
We now construct the automaton $\Amf_3$.  It ensures that when a
$\Theta$-labeled tree $(T,L)$ is accepted, then for all
$n\in\mn{ind}(\Amc_L)$, concept names $A$, and roles $r$:
\begin{enumerate}[(i)]

  \item $A\in L_2(n)$ iff there is a derivation tree for $A(n)$ in
    $\Amc_L$;

  \item for all $n\neq\varepsilon$, $r\in L_2(n)$ iff there is a role
    $s$ such that $s(n\cdot-1,n)\in \Amc_L$ and $\Tmc_2\models s\sqsubseteq r$.

\end{enumerate}
By Point~1 of Lemma~\ref{lem:derivationtrees}, these conditions ensure
that the interpretation $\Imc_{L,2}$ is the universal model of
$\Tmc_2$ and $\Amc$ restricted to $\mn{ind}(\Amc_L)$, in case \Amc is
consistent with $\Tmc_2$. Given this, we can verify consistency of
\Amc with $\Tmc_2$ based on Point~2 of
Lemma~\ref{lem:derivationtrees}, that is, we ensure the following:
\begin{itemize}

  \item[(iii)] the set $L_2(n)\cap\mn{N_C}$ is consistent with
    $\Tmc_2$, for all $n\in\mn{ind}(\Amc)$;

  \item[(iv)] the relation $\{(a,b) \mid s(a,b) \in \Amc_L\}$ is a partial
      function whenever $\mn{func}(s)\in\Tmc_2$.

\end{itemize}
We take $\Amf_3=(Q_3,\Theta,q_0,\delta_3,\Omega_3)$ where
\begin{align*}
  Q_3={} & \{q_0,q_{0}',q_1\} \cup \{q_A,\overline q_A\mid
    A\in\Theta_2\cap \mn{N_C}\}\cup{} \\ & \{q_r, \overline q_r, q_r^\Amc,
      \overline q_r^\Amc, q^f_r, q_{\neg r}\mid r \in
      \Theta_2\setminus\mn{N_C}\}\cup{}\\[0.5mm] & \{q_{r,B},\overline
      q_{r,B}\mid r \in\Theta_2\cap\mn{N_R}, B\in \Theta_2\cap
      \mn{N_C}\}
\end{align*}
and $\Omega_3$ assigns zero to all states, except for states of the
form $q_A$, to which it assigns $1$.  

For Condition~(i), we use states $q_A$ for the ``$\Leftarrow$'' part, and states
$\overline q_A$ for the ``$\Rightarrow$'' part. 
Intuitively, a state $q_A$ assigned to some node $n$ is an obligation
to verify the existence of a derivation tree for $A(n)$. Conversely,
$\overline q_A$ is the obligation that there is {\em no} such
derivation tree.
Conditions~(ii) to~(iv) are rather straightforward to verify. 
The automaton starts with the following transitions for every
$\sigma=(L_0,L_1,L_2)$:
$$ \begin{array}[h]{@{}r@{~}c@{~}ll@{}} \delta_3(q_0,\sigma) &=&
  \begin{cases} 
    \mn{true} & \text{if $L_0=\emptyset$} \\
    q_0' & \text{if $L_0\neq \emptyset$} \end{cases} \\[5mm]
  \delta_3(q_0',\sigma) &=& \begin{cases} \mn{false} & \text{if
      $L_2\cap \mn{N_C}$ inconsistent with $\Tmc_2$} \\
    \Box q_0\wedge\Box q_{1}\wedge\displaystyle \bigwedge_{A\in
      L_2\cap\mn{N_C}} \!\!q_A\wedge \bigwedge_{A\in
	(\Theta_2\cap\mn{N_C})\setminus L_2} \!\!\overline q_A &
	\text{otherwise} \end{cases}\\[7mm]
  \delta_3(q_{1},\sigma)& = & \begin{cases} \mn{true} &\text{if
    $L_0=\emptyset$} \\ \displaystyle\bigwedge_{\mn{func}(r)\in
  \Tmc_2} q^f_{r}~\wedge \bigwedge_{r\in L_2\cap \mn{N_R}} q_r \wedge
  \bigwedge_{r\in (\Theta_2\cap\mn{N_R})\setminus L_2} \overline q_r &
  \text{otherwise} \end{cases} \\[7mm]
  \delta_3(q^f_r,\sigma) &=& \left\{\begin{array}{ll} \Box q_{\neg r}
    & \text{if }r^-\in L_0 \\ \Box_1 q_{\neg r} & \text{if }r^-\notin
    L_0\end{array} \right.  \\
  \delta_3(q_{\neg r},\sigma) &=& \left\{\begin{array}{ll} \mn{true} &
    \text{if }r\notin L_0 \\ \mn{false} & \text{otherwise}\end{array}
  \right. \\
  \delta_3(q_r,\sigma) &=& \left\{\begin{array}{ll} \mn{true} &
    \text{if there is an } s\notin L_0 \text{ with } \Tmc_2 \models s
    \sqsubseteq r\\ \mn{false} & \text{otherwise}\end{array} \right.
  \\
  \delta_3(\overline{q}_r,\sigma) &=& \left\{\begin{array}{ll}
    \mn{true} & \text{if there is no } s\notin L_0 \text{ with }
    \Tmc_2 \models s \sqsubseteq r\\ \mn{false} &
    \text{otherwise}\end{array} \right.
\end{array}$$
For states $q_A$, we implement the conditions of derivation trees as
transitions.  Finiteness of the derivation tree is ensured by the
priority assigned to these states. The relevant transitions are as
follows:
$$\begin{array}[h]{@{}r@{~}c@{~}ll@{}}
  \delta_3(q_A,\sigma) &=& \mn{false} & \text{if $L_0=\emptyset$}
  \\[1mm]
  \delta_3(q_A,\sigma) &=& \mn{true} & \text{if $A\in L_0$} \\[1mm]
  \delta_3(q_A,\sigma) &=& 
      \bigvee_{\Tmc_2 \models A_1\sqcap \dots \sqcap A_n \sqsubseteq
      A} \big(q_{A_1}\wedge\dots\wedge q_{A_n}\big) \vee{} & \text{if
      $L_0\neq \emptyset$ and $A\notin L_0$} \\[1mm]
      %
      %
      && \bigvee_{B\sqsubseteq \forall r.A\in \Tmc_2,
      \Tmc_2\models s\sqsubseteq r} \big((q_{s}^\Amc\wedge
      \Diamond^- q_B) \vee \Diamond q_{s^-,B}\big) \vee{}\\[1mm]
       && \bigvee_{\substack{B\sqsubseteq \exists r.A \in
	\Tmc_2,\mn{func}(r) \in \Tmc_2}} \big((q_s^\Amc\wedge
	\Diamond^- q_B)\vee \Diamond q_{s^-,B}\big)  %
      \\[1mm]
    \delta_3(q^\Amc_r,\sigma) &=& \left\{\begin{array} {ll} \mn{true} & \text{if $r\in L_0$} \\
      \mn{false} & \text{otherwise} \end{array}\right. \\[1mm]
    \delta_3(q_{s,B},\sigma) &=& q_s^\Amc\wedge q_B
\end{array}$$
%
%
%
%
%
%
%
%
The transitions for $\overline q_A$ are obtained by dualizing the ones
for $q_A$.  More precisely, for every $q$ of the form $q_A$,
$q_r^\Amc$, and $q_{s,B}$, we define $\delta_3(\overline
q,\sigma)=\overline{\delta_3(q,\sigma)}$, where $\overline{\varphi}$
is obtained from $\varphi$ by exchanging $\wedge$ with $\vee$,
$\Diamond$ with $\Box$, $\Diamond^-$ with $\Box^-$, and $\mn{true}$
with $\mn{false}$, and replacing every state $p$ with $\overline p$.


  
     




\section{Tree-Shaped CQs and Deductive Conservative Extensions}
\label{sec:lower_bounds}

We consider the \onetCQ versions of entailment, inseparability, and
conservative extensions in \HornALCHIF (and fragments) as well as
their deductive companions in \ELHIFbot (and fragments). The results
are summed up be the following theorem.
\begin{theorem}
  \label{thm:complexity2}
  The following problems are \TwoExpTime-complete:
  \begin{enumerate}

  \item In \HornALCHIF and any of its fragments that contains \ELI or
    Horn-\ALC: $(\sigmaabox,\sigmaquery)$-\onetCQ entailment,
    $(\sigmaabox,\sigmaquery)$-\onetCQ inseparability, and
    $(\sigmaabox,\sigmaquery)$-\onetCQ conservative extensions; this holds
    even when $\sigmaabox=\sigmaquery$;

\item In \ELHIFbot and any of its fragments that contains \ELI: $\sig$-deductive
  entailment, $\sig$-deductive inseparability, and $\sig$-deductive
  conservative extensions.

 \end{enumerate}
 Moreover, all these problems can be 
  solved in time $2^{2^{p(|\Tmc_2| \mn{log}|\Tmc_1|)}}$, where $p$ is a 
  polynomial. 
\end{theorem}
It is worth remembering that in \EL, which is the fragment of \ELI that does not
admit inverse roles, the reasoning problems mentioned in Point~2 of
Theorem~\ref{thm:complexity2} are \ExpTime-complete~\cite{LW10}.  We
find it remarkable that adding inverse roles causes a jump by one
exponential. This has so far only been observed for reasoning problems
that involve (non-tree shaped) conjunctive queries
\cite<e.g.\ in>{BHLW16}, but this is not the case for the problems in Point~2.

For Point~1, the lower bounds have been established
by~\citeauthor{DBLP:journals/ai/BotoevaLRWZ19}~\citeyear{DBLP:journals/ai/BotoevaLRWZ19}. The upper
bound is proved almost exactly as in
Section~\ref{sec:automata_construction}. By
Theorem~\ref{thm:homs_tCQ}, the only difference is that the automaton
$\Amf_4$ now needs to accept a $\Theta$-labeled tree $(T,L)$ iff
$\Imc_{L,2}\not\preceq_\Sbf \Imc_{L,1}$. To achieve this, $\Amf_4$ 
ensures that for some $n\in\mn{ind}(\Amc_L)$, one of the following
conditionds hold:
\begin{enumerate}[1.']

  \item[1.] there is a $\sigmaquery$-role $r$ and an $n'\in\mn{ind}(\Amc_L)$ such that
    $(n,n')\in r^{\Imc_{L,2}}$, but $(n,n')\notin r^{\Imc_{L,1}}$;

  \item[2.$'$] $\Imc_{\Tmc_2,L_2(n) \cap \NC} \not\preceq_\sigmaquery
    \Imc_{L,1}$ via a \Qbf-simulation that relates the root of
    $\Imc_{\Tmc_2,L_2(n) \cap \NC}$ to $n$.  \end{enumerate}
This can be achieved by modifying the construction of $\Amf_4$ from
Section~\ref{sec:automata_construction} as follows: drop all states
$\overline{q}^t_3$, $\overline{q}^t_{3b}$, and
$\overline{q}^{\rho,t}_2,\overline{q}^{\rho,t,\downarrow}_2$ with
$|\rho|>1$ (and all according transitions) and replace the transitions
for $q_0$ and $\overline{q}^t_2$ as follows:
$$ \begin{array}[h]{r@{~}c@{~}ll}
  \delta_4(q_0,\sigma) &=& \left\{\begin{array}{ll} \Diamond q_0 \vee
    q_1\vee \overline{q}_2^{L_2 \cap \NC} & \text{if }L_0\neq
    \emptyset \\  \mn{false} & \text{otherwise} \end{array} \right.
  \\[5mm]
  \delta_4(\overline{q}_2^t,\sigma) &=& \left\{\begin{array}{ll}
    \mn{true} & \text{if }t|_\sigmaquery\not\subseteq L_1 \\
    \displaystyle\bigvee_{t\rightsquigarrow^{\Tmc_2}_{\rho}t'}\,\bigvee_{r\in\rho}
    \overline{q}_2^{\{r\},t'} & \text{otherwise.} \end{array} \right. 
\end{array}$$
It thus remains to prove Point~2 of Theorem~\ref{thm:complexity2}.
The upper bound is a consequence of Point~1,
Proposition~\ref{lem:deductive_versus_query_entailment}, and the fact
that \onetCQ evaluation in \ELHIFbot can be done in exponential
time~\cite{EGOS08}. We
turn to the lower bounds.
\begin{theorem}
  In \ELI, deciding $(\Sbf,\Sbf)$-\onetCQ conservative extensions,
  \Sbf-deductive conservative extensions, deductive $\sig$-entailment,
  and deductive $\sig$-inseparability are \TwoExpTime-hard. This even
  holds when $\Sbf = \mn{sig}(\Tmc_1)$.
\end{theorem}
%
%
\begin{proof} We start with \onetCQ conservative extensions (not
  assuming $\Sbf = \mn{sig}(\Tmc_1)$) and reduce from the following
  problem which has been shown to be \TwoExpTime-hard~\cite[implicit in Lemma~8]{GuJuSa-IJCAI18}.

\begin{description}

  \item[\textbf{Input}:] \ELI TBox \Tmc, signature $\sig$, concept names $A,B$ 

  \item[\textbf{Question}:] $\Imc_{\Tmc,\{A(a)\}}\to_\sig \Imc_{\Tmc,\{B(a)\}}$?

\end{description}
For the reduction, let $\sig,\Tmc,A,B$ be an input. Define two
TBoxes $\Tmc_1$ and $\Tmc_2$ and a signature $\sig'$ by taking
\begin{align*}
  \Tmc_1 & = \Tmc \cup \{Y\sqsubseteq \exists s.B\}, \\
  \Tmc_2 & = \Tmc_1\cup \{Y\sqsubseteq \exists s.A\}, \\
  \sig' & = \sig\cup \{Y,s\},
\end{align*}
for some fresh concept name $Y$ and a fresh role name $s$. By
Theorem~\ref{thm:homs_tCQ}, it suffices to establish the following
claim.

\smallskip\noindent{\bf Claim.} $\Imc_{\Tmc,\{A(a)\}}\to_\sig
\Imc_{\Tmc,\{B(a)\}}$ iff $\Imc_{\Tmc_2,\Amc}\preceq_{\sig'}
\Imc_{\Tmc_1,\Amc}$ for all tree-shaped $\Sbf'$-ABoxes \Amc.

\par\medskip\noindent
\textbf{{\boldmath ``$\Rightarrow$''.}}~
Assume $\Imc_{\Tmc,\{A(a)\}}\to_\sig
\Imc_{\Tmc,\{B(a)\}}$ and let \Amc be a tree-shaped $\Sbf'$-ABox. By
construction of $\Tmc_1$, the universal model $\Imc_{\Tmc_1,\Amc}$ can
be obtained from $\Imc_{\Tmc,\Amc}$ by adding to every $c$ with
$Y(c)\in \Amc$ an $s$-successor to a fresh isomorphic copy of
$\Imc_{\Tmc,\{B(a)\}}$. Similarly, the universal model
$\Imc_{\Tmc_2,\Amc}$ can be obtained from $\Imc_{\Tmc,\Amc}$ by adding
to every $c$ with $Y(c)\in \Amc$ an $s$-successor to a copy of
$\Imc_{\Tmc,\{B(a)\}}$ and another $s$-successor to a copy of
$\Imc_{\Tmc,\{A(a)\}}$.  The required simulation can be easily
constructed based on the existing homomorphism
$\Imc_{\Tmc,\{A(a)\}}\to_\sig \Imc_{\Tmc,\{B(a)\}}$.

\par\medskip\noindent
\textbf{{\boldmath ``$\Leftarrow$''.}}~
Suppose that
$\Imc_{\Tmc_2,\Amc}\preceq_{\sig'} \Imc_{\Tmc_1,\Amc}$ for all
tree-shaped $\Sbf'$-ABoxes \Amc. Then, in particular,
$\Imc_{\Tmc_2,\{Y(a)\}}\preceq_{\sig'} \Imc_{\Tmc_1,\{Y(a)\}}$. We can
now read off an $\sig$-homomorphism from the $\sig'$-simulation
because $\Imc_{\Tmc_2,\{Y(a)\}}$ and $\Imc_{\Tmc_1,\{Y(a)\}}$ are
tree-shaped.

\medskip

This finishes the proof of the claim and shows that \onetCQ
conservative extensions are \TwoExpTime-hard in \ELI. It remains to
note that Proposition~\ref{lem:deductive_versus_query_entailment_aux}
provides a reduction from \onetCQ entailment over \ELI TBoxes to
deductive conservative extensions in \ELHIFbot (importantly, we
trivially have $\Tmc_1\models^\bot_\sig\Tmc_2$, no role hierarchies,
and no functionality assertions).

\medskip

For the case of $\Sbf = \mn{sig}(\Tmc_1)$, we do a closer inspection
of the \TwoExpTime-hardness proof of the mentioned homomorphism
problem by~\citeauthor{GuJuSa-IJCAI18}~\citeyear{GuJuSa-IJCAI18}. The constructed TBox $\Tmc$ is in
fact a union of two disjoint TBoxes $\Tmc=\Tmc_1\cup\Tmc_2$ which
additionally satisfy:
\begin{itemize}
    
  \item[$(\ast)$] 
    $\Imc_{\Tmc_1\cup\Tmc_2,\{A(a)\}}=\Imc_{\Tmc_2,\{A(a)\}}$ for all
    $A\notin\mn{sig}(\Tmc_1)$, and
    $\Imc_{\Tmc_1\cup\Tmc_2,\Amc}=\Imc_{\Tmc_1,\Amc}$, for all
    $\mn{sig}(\Tmc_1)$-ABoxes~\Amc, 

\end{itemize}
and the used signature is actually $\mn{sig}(\Tmc_1)$. Thus, the
problem of deciding $\Imc_{\Tmc_2,\{A(a)\}}\to_{\mn{sig}(\Tmc_1)}
\Imc_{\Tmc_1,\{B(a)\}}$, for $B\in\mn{sig}(\Tmc_1)$ and
$A\notin\mn{sig}(\Tmc_1)$ is \TwoExpTime-hard for TBoxes
$\Tmc_1,\Tmc_2$ satisfying~$(\ast)$. For the reduction, let us fix
TBoxes $\Tmc_1'$ and $\Tmc_2'$ by taking
\begin{align*}
  \Tmc_1' & = \Tmc_1 \cup \{Y\sqsubseteq \exists s.B\}, \\
  \Tmc_2' & = \Tmc_1\cup \Tmc_2 \cup \{Y\sqsubseteq \exists s.B,
  Y\sqsubseteq \exists s.A\},\\
  \sig & = \mn{sig}(\Tmc_1') = \mn{sig}(\Tmc_1)\cup \{Y,s\}
\end{align*}
for fresh symbols $Y,s$. Based on~$(\ast)$, it is not hard to verify
that the following claim can be proved analogously to the claim above, thus establishing
hardness by Theorem~\ref{thm:homs_tCQ}.

\smallskip\noindent{\bf Claim.} $\Imc_{\Tmc_2,\{A(a)\}}\to_{\mn{sig}(\Tmc_1)} \Imc_{\Tmc_1,\{B(a)\}}$
iff $\Imc_{\Tmc_2',\Amc}\preceq_\sig\Imc_{\Tmc_1',\Amc}$,
for all tree-shaped $\sig$-ABoxes. 
\end{proof}

\section{Conclusion}
\label{sec:concl}

We have shown that query conservative extensions of TBoxes are
\TwoExpTime-complete for all DLs between \ELI and Horn-\ALCHIF as
well as between Horn-\ALC and Horn-\ALCHIF, and that deductive
conservative extensions are \TwoExpTime-complete for all DLs between
\ELI and $\ELHIFbot$. This gives a fairly complete picture of the
complexity and decidability of conservative extensions in Horn DLs
with inverse roles.

An interesting problem left open is the decidability and complexity of
deductive conservative extensions in Horn-\ALCHIF. A natural approach
would be to come up with a characterization of the expressive power in
terms of a `(bi)simulation-type' indistinguishability relation on
models, to then use that as the foundation for a characterization of
deductive conservative extensions similar to what has been done for
expressive DLs~\cite{DBLP:conf/ijcai/LutzW11} and \EL~\cite{DBLP:conf/kr/LutzSW12}, and to finally
develop a decision procedure based on tree automata. A first step has
recently been made
by~\citeauthor{DBLP:conf/lics/JungPWZ19}~\citeyear{DBLP:conf/lics/JungPWZ19}
who presented a
(bi)simulation-type relation for the fragment Horn-\ALC of
Horn-\ALCHIF. Unfortunately, the proposed relation is
much more complicated than standard (bi)simulations and while it can
be used to characterize deductive conservative extensions, it is
far from clear how such a characterization can lead to a decision
procedure.

It would also be interesting to add transitive roles to the picture,
that is, to transition from Horn-\ALCHIF to Horn-$\mathcal{SHIF}$.  In
fact, we are not aware of any results on conservative extensions or
related notions that concern description logics with transitive roles.

Finally, it would be interesting to consider both query and deductive
conservative extensions in the Datalog$^\pm$ family of ontology
languages (aka existential rules) such as frontier-guarded existential
rules \cite{DBLP:conf/ijcai/BagetMRT11}. In fact, we are not aware of
any (un)decidability results regarding conservative extensions in such
languages. The increased existential power of Datalog$^\pm$ languages
brings in serious additional technical challenges. In this context, it
is interesting to remark that deductive conservative extensions are
undecidable in the guarded fragment, which might be seen as an
extension of relevant Datalog$^\pm$ languages
\cite{DBLP:conf/icalp/JungLM0W17}.

\section*{Acknowledgments}
This work was partially supported by DFG project SCHN 1234/3
and ERC Consolidator Grant 647289 CODA.

\appendix

\section{Proofs for Section~\ref{sec:prelims}}

This section contains the proofs omitted from Sections~\ref{sec:query_entailment}
and~\ref{sec:dCEs}.

\subsection{Proof of Lemma~\ref{lem:inconsistent_ABoxes_aux}}
\label{appx:inconsistent_ABoxes}


\leminconsistentABoxesaux*

\noindent
\begin{proof}
  We prove both implications via contraposition.

  \par\medskip\noindent
  \textbf{{\boldmath ``$\Rightarrow$''.}}~
  Assume (1) and (2) are both false,
  i.e., $\Tmc_1$ is not $(\sigmaabox,\sigmaquery)$-universal
  and either (a)~$\Tmc_1 \not\models_{\sigmaabox,\sigmaquery}^{\text{CQ}} \Tmc_2$
  or (b)~$\Tmc_1 \not\models_{\sigmaabox}^{\bot} \Tmc_2$.
  In case~(a), $\Tmc_1$ trivially does not $(\sigmaabox,\sigmaquery)$-CQ entail $\Tmc_2$ with inconsistent ABoxes.
  In case~(b), consider a witness $\sigmaabox$-ABox $\Amc$.
  Since $\Tmc_1$ is not $(\sigmaabox, \sigmaquery)$-universal,
  there is an $\sigmaabox$-ABox $\Amc'$, a $\sigmaquery$-CQ $q(\xbf)$
  and a tuple $\abf \subseteq \mn{ind}(\Amc')$ with $|\abf| = |\xbf|$
  such that $\Tmc_1,\Amc' \not\models q(\abf)$.
  We assume w.l.o.g.\ that $\Amc$ and $\Amc'$ use distinct sets of individuals.
  We set $\Amc'' = \Amc \cup \Amc'$ and have:
  \begin{itemize}
    \item
      $\Tmc_2,\Amc'' \models q(\abf)$ because
      $\Amc$ is inconsistent with $\Tmc_2$ and so is $\Amc''$;
    \item
      $\Tmc_1,\Amc'' \not\models q(\abf)$:~
      let \Jmc be the disjoint union of the universal model $\Imc_{\Tmc_1,\Amc}$
      and the model \Imc witnessing $\Tmc_1,\Amc' \not\models q(\abf)$.
      Clearly $\Jmc$ is a model of $\Tmc_1$ and $\Amc''$, but $\Jmc \not\models q(\abf)$.
  \end{itemize}
  Hence 
  $\Tmc_1$ does not $(\sigmaabox,\sigmaquery)$-CQ entail $\Tmc_2$ with inconsistent ABoxes,
  as desired.

  \par\medskip\noindent
  \textbf{{\boldmath ``$\Leftarrow$''.}}~
  Assume
  $\Tmc_1$ does not $(\sigmaabox,\sigmaquery)$-CQ entail $\Tmc_2$ with inconsistent ABoxes
  and consider a witness $(\Amc,q,\abf)$.
  Then it is immediate that (2) does not hold.
  Furthermore,
  if $\Amc$ is consistent with both $\Tmc_1$ and $\Tmc_2$, then
  $\Tmc_1 \not\models_{\sigmaabox,\sigmaquery}^{\text{CQ}} \Tmc_2$.
  Otherwise $\Amc$ must be inconsistent with $\Tmc_2$ but consistent with $\Tmc_1$;
  hence $\Tmc_1 \not\models_{\sigmaabox}^{\bot} \Tmc_2$.
  Therefore (1) does not hold either.
\end{proof}
%
%

\par\medskip\noindent
%
%

\subsection{Characterization of Inconsistency Entailment}
\label{appx:inconsistentcy_entailment}


We show that inconsistency entailment can be reduced to CQ
entailment. We write $\Tmc_1 \models_\sigmaabox^{\textup{fork}} \Tmc_2$
if for all $\sigmaabox$-ABoxes $\Amc = \{r(a,b),r(a,c)\}$: if $\Amc$
is inconsistent with $\Tmc_2$, then also with $\Tmc_1$. Note that
$\Tmc_1 \models_\sigmaabox^{\textup{fork}} \Tmc_2$ can be decided by
evaluating the \onetCQ $A(a)$ on all $\sigmaabox$-ABoxes of the form
$\Amc = \{r(a,b),r(a,c)\}$ for both $\Tmc_i$, where $A$ is a concept
name that does not occur in $\Tmc_1 \cup \Tmc_2$. Thus, $\Tmc_1
\models_\sigmaabox^{\textup{fork}} \Tmc_2$ is basically a shortcut for
polynomially many CQ entailment tests.
%
%
%
\begin{lemma}
  \label{lem:inconsistency_entailment}
  Let $\sigmaabox$ be a signature and let $\Tmc_1$ and $\Tmc_2$ be $\HornALCHIF$ TBoxes.
  Furthermore, let $A$ be a fresh concept name
  and 
  %
  %
  $\Tmc_i^A$ be obtained from $\Tmc_i$
  by replacing each occurrence of $\bot$ with $A$
  and adding the axioms $A \sqsubseteq \forall s.A$ and $A \sqsubseteq
  \forall s^-.A$
  for every role $s$ occurring in $\Tmc_i$, for $i=1,2$.
%
  Then the following are equivalent.
  \begin{enumerate}
    \item[(1)]
      $\Tmc_1 \models_{\sigmaabox}^{\bot} \Tmc_2$;
    \item[(2)]
      $\Tmc_1^A \models_{\sigmaabox,\{A\}}^{\textup{CQ}} \Tmc_2^A$
      ~and~ $\Tmc_1 \models_\sigmaabox^{\textup{fork}} \Tmc_2$.
  \end{enumerate}
\end{lemma}

%
\par\medskip\noindent
\begin{proof}
  We prove both implications via contraposition.

  \par\smallskip\noindent
  \textbf{{\boldmath (1) $\Rightarrow$ (2).}}~
  Assume $\Tmc_1^A \not\models_{\sigmaabox,\{A\}}^{\textup{CQ}} \Tmc_2^A$
  or $\Tmc_1 \not\models_\sigmaabox^{\textup{fork}} \Tmc_2$.
  In case $\Tmc_1 \not\models_\sigmaabox^{\textup{fork}} \Tmc_2$,
  every witness ABox
  is a witness for $\Tmc_1 \not\models_{\sigmaabox}^{\bot} \Tmc_2$ too.

  In case $\Tmc_1^A \not\models_{\sigmaabox,\{A\}}^{\textup{CQ}} \Tmc_2^A$ is violated, 
  consider a witness $(\Amc,q,\abf)$.
  Since $A$ is the only symbol allowed in $q$, all atoms of $q$
  have the form $A(z)$ for arbitrary variables $z$.
  If $q$ consists of several atoms, then it is disconnected
  and we can omit all but one atom from $q$
  and still have a witness
  (see also proof of Proposition~\ref{lem:tree_shaped_witnesses},
  Property~\ref{it:tree-q-connected}).
  Hence we can assume w.l.o.g.\ that $q$ is of the form (i) $q(x) = A(x)$ or (ii) $q() = \exists y\,A(y)$
  and, furthermore, that $\Amc$ and thus the universal models
  $\Imc_{\Tmc_i,\Amc}$ are connected.
  (Due to the ``propagation'' of $A$ in the $\Tmc_i$,
  we can even assume that $q$ is of the form (i) only,
  but that does not matter in the following argumentation.)
  We now have:
  \begin{itemize}
    \item
      $\Amc$ is inconsistent with $\Tmc_2$:
      \par\smallskip\noindent
      Assume to the contrary that
      $\Amc$ is consistent with $\Tmc_2$
      and consider the universal model $\Imc_{\Tmc_2,\Amc}$
      for $\Tmc_2$ and \Amc (Section~\ref{sec:universal_model}).
      Clearly, 
      for all domain elements $d$ of $\Imc_{\Tmc_2,\Amc}$,
      we have $\Tmc_2 \not\models \bigsqcap\mn{tp}_{\Imc_{\Tmc_2,\Amc}}(d) \sqsubseteq \bot$.
      Since $A$ is fresh and by the definition of $\Tmc_2^A$
      we get $\Tmc_2^A \not\models \bigsqcap\mn{tp}_{\Imc_{\Tmc_2,\Amc}}(d) \sqsubseteq A$.
      Now Lemma~\ref{lem:universal_model}~(\ref{it:univ_model_is_a_model}) for $\Tmc_2^A$ implies that 
      $\Imc_{\Tmc_2,\Amc}$ is a model of $\Tmc_2$ and $\Amc$;
      hence $\Imc_{\Tmc_2,\Amc}$ satisfies all axioms in $\Tmc_2^A$
      that have been taken over from $\Tmc_2$ without modification,
      i.e., all axioms that are not of the form $B \sqsubseteq A$.
      But axioms of the latter form are also satisfied
      because $\Tmc_2^A \not\models \bigsqcap\mn{tp}_{\Imc_{\Tmc_2,\Amc}}(d) \sqsubseteq A$
      for every domain element $d$. Hence $\Imc_{\Tmc_2,\Amc}$ is a model of $\Tmc_2^A$ and $\Amc$.
      Now, since $\Imc_{\Tmc_2,\Amc}$ has no $A$-instance,
      we cannot have $\Tmc_2^A,\Amc \models q(\abf)$ for any $\{A\}$-query $q$;
      contradicting the assumption that $(\Amc,q,\abf)$ is a witness.
      \par\smallskip\noindent
    \item
      $\Amc$ is consistent with $\Tmc_1$:
      \par\smallskip\noindent
      Since $(\Amc,q,\abf)$ is a witness,
      we have $\Imc_{\Tmc_1^A,\Amc} \not\models q(\abf)$
      by Lemma~\ref{lem:universal_model}~(\ref{it:univ_model_queries}).
      Due to the additional axioms in the definition of $\Tmc_1^A$,
      which ``propagate'' $A$ into every domain element of the connected (see above)
      universal model $\Imc_{\Tmc_1^A,\Amc}$,
      we have $\Tmc_1^A \not\models \bigsqcap\mn{tp}_{\Imc_{\Tmc_1^A,\Amc}}(d) \sqsubseteq A$
      for all domain elements $d$.
      Since $A$ is fresh, we have 
      $\Tmc_1 \not\models \bigsqcap\mn{tp}_{\Imc_{\Tmc_1^A,\Amc}}(d) \sqsubseteq \bot$.
      With the same reasoning as above,
      we obtain that $\Imc_{\Tmc_1^A,\Amc}$ is a model of $\Tmc_1$ and $\Amc$;
      hence $\Amc$ is consistent with $\Tmc_1$.
  \end{itemize}
  Consequently $\Tmc_1 \not\models_{\sigmaabox}^{\bot} \Tmc_2$, as desired.

%
  \par\medskip\noindent
  \textbf{{\boldmath (2) $\Rightarrow$ (1).}}~
  Assume $\Tmc_1 \not\models_{\sigmaabox}^{\bot} \Tmc_2$,
  i.e., there is an $\sigmaabox$-ABox \Amc that is
  is inconsistent with $\Tmc_2$ but consistent with $\Tmc_1$.
  We need to show that
  $\Tmc_1^A \not\models_{\sigmaabox,\{A\}}^{\textup{CQ}} \Tmc_2^A$
  or $\Tmc_1 \not\models_\sigmaabox^{\textup{fork}} \Tmc_2$.
  
  From $\Amc$ being inconsistent with $\Tmc_2$,
  we first conclude that one of the following two properties must hold.
  \begin{itemize}
    \item[(i)]
      There is some $d \in B^{\Imc_{\Tmc_2,\Amc}}$
      with $B \sqsubseteq \bot \in \Tmc_2$;
    \item[(ii)]
      $\Amc$ contains a ``fork'' $\Amc^- = \{r(a,b),r(a,c)\}$
      such that $\Amc^-$ is inconsistent with $\Tmc_2$.
  \end{itemize}
  Indeed,
  if neither~(i) nor~(ii) holds, then we have $\Imc_{\Tmc_2,\Amc} \models (\Tmc_2,\Amc)$,
  contradicting the inconsistency of $\Amc$ with $\Tmc_2$:
  First, $\Imc_{\Tmc_2,\Amc} \models \Amc$
  follows directly from the construction of $\Imc_{\Tmc_2,\Amc}$.
  Second, $\Imc_{\Tmc_2,\Amc} \models \Tmc_2$
  can be shown analogously to the (omitted) standard proof of
  Lemma~\ref{lem:universal_model}~(\ref{it:univ_model_is_a_model}),
  via a case distinction over the axioms in $\Tmc_2$,
  using ``not~(i)'' and ``not~(ii)'' instead of the
  assumption that $\Amc$ is consistent with $\Tmc_2$.

  Now first assume that (ii)~holds.
  Since $\Amc$ is consistent with $\Tmc_1$,
  so is $\Amc^-$.
  Hence $\Tmc_1 \not\models_\sigmaabox^{\textup{fork}} \Tmc_2$.

  In case (ii)~does not hold, (i)~must hold.
  To show that $\Tmc_1^A \not\models_{\sigmaabox,\{A\}}^{\textup{CQ}} \Tmc_2^A$,
  consider the CQ $q = A(x)$ and some $a \in \mn{ind}(\Amc)$
  to which the element $d$ from~(i) is connected in $\Imc_{\Tmc_2,\Amc}$,
  i.e., if $d \in \mn{ind}(\Amc)$, then choose $a=d$;
  otherwise choose $a$ such that $d$ is in the subtree
  $\Imc_{\Tmc_2,\Amc}|_a$.
  We then have:
  \begin{itemize}
    \item
      $\Amc$ is consistent with $\Tmc_2^A$:
      \par\smallskip\noindent
      Since $\Tmc_2^A$ does not contain $\bot$
      and $\Amc$ does not contain forks as in~(ii),
      \Amc is consistent with $\Tmc_2^A$ is consistent, as witnessed
      by the universal model $\Imc_{\Tmc_2^A,\Amc}$
      (we again refer to the standard proof of Lemma~\ref{lem:universal_model}~(\ref{it:univ_model_is_a_model});
      except that the FA case in the ABox part of $\Imc_{\Tmc_2^A,\Amc}$ is now due to ``not~(ii)'').
      \par\smallskip\noindent
    \item
      $\Amc$ is consistent with $\Tmc_1^A$:
      \par\smallskip\noindent
      It is not difficult to see that
      $\Imc_{\Tmc_1,\Amc}$ is a model of $\Tmc_1^A$ and $\Amc$:
      by Lemma~\ref{lem:universal_model}~(1),
      $\Imc_{\Tmc_1,\Amc}$ is a model of $\Tmc_1$ and $\Amc$;
      in particular,
      $\Imc_{\Tmc_1,\Amc}$ satisfies all axioms in $\Tmc_1^A$
      that $\Tmc_1^A$ shares with $\Tmc_1$.
      The modified axioms $B \sqsubseteq A$ with $B \sqsubseteq \bot \in \Tmc_1$
      are satisfied, too, because $\Imc_{\Tmc_1,\Amc}$ cannot have any $B$-instances.
      Finally, the additional propagation axioms are satisfied because
      $\Imc_{\Tmc_1,\Amc}$ has no $A$-instance as $A$ is fresh.
      \par\smallskip\noindent
    \item
      $\Tmc_2^A,\Amc \models q(a)$:
      \par\smallskip\noindent
      Due to~(i), we have $\Imc_{\Tmc_2,\Amc} \models \exists y\,B(y)$
      for some $B \sqsubseteq \bot \in \Tmc_2$.
      Hence $\Imc_{\Tmc_2^A,\Amc} \models \exists y\,B(y)$,
      which follows from the construction of both universal models
      (in fact the only difference between $\Imc_{\Tmc_2^A,\Amc}$ and $\Imc_{\Tmc_2,\Amc}$
      is that some domain elements of $\Imc_{\Tmc_2^A,\Amc}$ may be $A$-instances).
      Hence $\Imc_{\Tmc_2^A,\Amc}$ has a $B$-instance
      in the subtree $\Imc_{\Tmc_2,\Amc}|_a$
      and thus, by construction, an $A$-instance.
      By the ``propagation'' of $A$ in $\Tmc_2^A$,
      we have that $a$ is an instance of $A$ in $\Imc_{\Tmc_2^A,\Amc}$;
      hence $\Imc_{\Tmc_2^A,\Amc} \models A(a) = q$.
      \par\smallskip\noindent
    \item
      $\Tmc_1^A,\Amc \not\models q(a)$:
      \par\smallskip\noindent
      Follows from $\Imc_{\Tmc_1,\Amc}$ being a model of $\Tmc_1^A$ and $\Amc$ (as shown above)
      and $\Imc_{\Tmc_1,\Amc} \not\models q(a)$ (given the lack of $A$-instances).
  \end{itemize}
\end{proof}
%

\subsection{Proof of Lemma~\ref{lem:deductive_versus_query_entailment_aux}}
\label{appx:deductive_versus_query_entailment}



\myRest*

\noindent
\begin{proof}
%
%
%
  We prove both implications via contraposition.

  \par\medskip\noindent
  \textbf{{\boldmath ``$\Leftarrow$''.}}~
  We assume that
  $\Tmc_1 \not\models_{\sig}^{\textup{TBox}} \Tmc_2$.
  In case this is witnessed by an $\sig$-FA $\mn{func}(r)$,
  we immediately get a witness $\sig$-ABox = $\{r(a,b),r(a,c)\}$
  for $\Tmc_1 \not\models_\sig^\bot \Tmc_2$ and are done.

  Otherwise, $\Tmc_1$ contains all $\sig$-FAs from $\Tmc_2$,
  and there is a witness $\sig$-CI $C \sqsubseteq D$
  (witness RIs are excluded by the assumption
  $\Tmc_1 \models_{\sig,\sig}^{\textup{RI}}\Tmc_2$).
  Since $\ELIbot$ concepts that contain $\bot$ are equivalent to $\bot$,
  the left-hand side $C$ cannot contain $\bot$ (i.e., is an \ELI concept)
  and, if $D$ does, then $C \sqsubseteq \bot$ is a witness.
  We show that such witnesses
  give rise to either 
  a witness $\Amc_C$ for $\Tmc_1 \not\models_{\sig}^{\bot} \Tmc_2$
  or a witness $(\Amc_C,q_D,a)$
  for $\Tmc_1 \not\models_{\sig,\sig}^{\textup{\onetCQ}} \Tmc_2$
  with $q_D(x)$ a \emph{\onetCQ}.

  We first consider the case that there is a witness
  $C \sqsubseteq \bot$ with $C$ an \ELI concept.
  We can construct from $C$ in the obvious way
  a tree-shaped $\sig$-ABox $\Amc_C$ and root $a$:
  \Amc reflects the tree structure of $C$;
  however, to respect the $\sig$-FAs in $\Tmc_1$
  (and thus those in $\Tmc_2$),
  we need to merge the subtrees of all nodes
  that are $r$-neighbors of the same node, whenever $\mn{func}(r) \in \Tmc_1$.
  Consider the universal model $\Imc_{\Tmc_2,\Amc_C}$%
  \footnote{%
    The assumption that \Amc is consistent with \Tmc
    is not needed for the construction of $\Imc_{\Tmc,\Amc}$,
    only for the proof of Lemma~\ref{lem:universal_model}\,(\ref{it:univ_model_is_a_model}).%
  }
  and observe that $a \in C^{\Imc_{\Tmc_2,\Amc_C}}$ from the construction of $\Imc_{\Tmc_2,\Amc_C}$.
  Since $\Tmc_2 \models C \sqsubseteq \bot$,
  we have that $\Imc_{\Tmc_2,\Amc_C}$ is not a model of $\Tmc_2$.
  Hence, by the contrapositive of Lemma~\ref{lem:universal_model}~(\ref{it:univ_model_is_a_model}),
  $\Amc_C$ is inconsistent with $\Tmc_2$.
  On the other hand, since $\Tmc_1 \not\models C \sqsubseteq \bot$,
  there is a model $\Imc \models \Tmc_1$ and an instance $d \in C^\Imc$.
  We can turn $\Imc$ into a model of $\Amc_C$ by interpreting the ABox individuals accordingly
  (``partial'' unraveling might be necessary to ensure that the standard name assumption is respected),
  witnessing the consistency of $\Amc_C$ with $\Tmc_1$.
  We thus have $\Tmc_1 \not\models_\sig^\bot \Tmc_2$
  and are done.

  In the second case, \emph{all} witnesses $C \sqsubseteq D$ consist
  solely of \ELI concepts $C,D$.
  We construct the same ABox $\Amc_C$ with root $a$ from $C$ and
  transform $D$
  into an $\sig$-\onetCQ $q_D(x)$ with a single answer variable
  that represents the tree shape of $D$.
  Now $(\Amc_C,q_D,a)$ is a witness to
  $\Tmc_1 \not\models_{\sig,\sig}^{\textup{\onetCQ}} \Tmc_2$
  for the following reasons.
  \begin{itemize}
    \item
      $\Amc_C$ is consistent with $\Tmc_1$:
      a model can be obtained in the obvious way
      from the model witnessing $\Tmc_1 \not\models C \sqsubseteq D$
      (possibly involving ``partial'' unraveling as above).
    \item
      $\Amc_C$ is consistent with $\Tmc_2$:
      since $C \sqsubseteq \bot$ is not a witness
      to $\Tmc_1 \not\models_{\sig}^{\textup{TBox}} \Tmc_2$,
      there must be a model $\Imc \models \Tmc_2$
      with $d \in C^\Imc$. We claim that we can turn $\Imc$
      into a model of $\Amc_C$ by interpreting the ABox
      individuals without violating the standard name assumption.
      If we assume to the contrary that this is not possible,
      then there are subconcepts $C_1,\dots,C_n$ of $C$
      corresponding to subtrees that have been merged in the construction of $\Amc_C$,
      such that $\Tmc_2 \models C_1 \sqcap \dots \sqcap C_n \sqsubseteq \bot$.
      However, $\Tmc_1 \not\models C_1 \sqcap \dots \sqcap C_n \sqsubseteq \bot$
      because $\Amc_C$ is consistent with $\Tmc_1$, as shown previously.
      Hence $C_1 \sqcap \dots \sqcap C_n \sqsubseteq \bot$
      would be a witness to $\Tmc_1 \not\models_{\sig}^{\textup{TBox}} \Tmc_2$,
      which we have ruled out -- a contradiction.
    \item
      $\Tmc_2,\Amc_C \models q_D(a)$,
      witnessed by $\Imc_{\Tmc_2,\Amc_C}$,
      together with 
      $a \in C^{\Imc_{\Tmc_2,\Amc_C}}$
      and $\Tmc_2 \models C \sqsubseteq D$.
    \item
      $\Tmc_1,\Amc_C \not\models q_D(a)$:
      take a model \Imc witnessing $\Tmc_1 \not\models C \sqsubseteq D$
      and an element $d \in C^\Imc \setminus D^\Imc$.
      As in the previous case, we can turn $\Imc$ into a model
      $\Jmc$ of $\Amc_C$ by interpreting the ABox individuals
      (again involving unraveling if necessary),
      obtaining $\Jmc \not\models q_D(a)$.
  \end{itemize}

  \par\smallskip\noindent
  \textbf{{\boldmath ``$\Rightarrow$''.}}~
  Assume $\Tmc_1 \not\models_{\sig,\sig}^{\textup{\onetCQ}} \Tmc_2$
  or $\Tmc_1 \not\models_\sig^\bot \Tmc_2$.

  In case $\Tmc_1 \not\models_\sig^\bot \Tmc_2$,
  consider a witness $\sig$-Box \Amc and assume w.l.o.g.\
  that \Amc is tree-shaped. Let $a \in \mn{ind}(\Amc)$ be its root.
  We can assume that $\Tmc_1$ contains all $\sig$-FAs from $\Tmc_2$
  (otherwise $\Tmc_1 \not\models_{\sig}^{\textup{TBox}} \Tmc_2$ and we are done).
  We turn \Amc into an $\sig$-\ELI concept $C_\Amc$
  in the obvious way.
  Then $C_\Amc \sqsubseteq \bot$ is a witness to 
  $\Tmc_1 \not\models_{\sig}^{\textup{TBox}} \Tmc_2$:
  \begin{itemize}
    \item
      $\Tmc_2 \models C_\Amc \sqsubseteq \bot$
      because, if there were a model $\Imc$ of $\Tmc_2$
      with $d \in C_\Amc^\Imc$,
      we could turn it into a model of $\Tmc_2$ and $\Amc$
      by interpreting the ABox individuals accordingly
      (possibly involving partial unraveling as above),
      which would contradict the assumption that \Amc
      is a witness to $\Tmc_1 \not\models_\sig^\bot \Tmc_2$.
    \item
      $\Tmc_1 \not\models C_\Amc \sqsubseteq \bot$,
      witnessed by $\Imc_{\Tmc_1,\Amc}$.
  \end{itemize}
  In case $\Tmc_1 \not\models_{\sig,\sig}^{\textup{\onetCQ}} \Tmc_2$,
  by Proposition~\ref{lem:tree_shaped_witnesses_tCQs}
  there is a witness $(\Amc,q,a)$ with 
  $\Amc$ tree-shaped and
  $q$ an $\sig$-\onetCQ.
  We construct $C_\Amc$ as above and another $\sig$-\ELI concept $D_q$
  from $q$ in the obvious way.
  It can be shown analogously to the previous case that $C_\Amc \sqsubseteq D_q$
  is a witness to $\Tmc_1 \not\models_{\sig}^{\textup{TBox}} \Tmc_2$.
\end{proof}
%

\section{Proofs for Section~\ref{sec:automata}}
\label{appx:automata}

We first state the semantics of 2ATA$_c$ precisely. 
Let $(T,L)$ be a $\Theta$-labeled tree and 
$\Amf=(Q,\Theta,q_0,\delta,\Omega)$ a \ata.  A {\em run of \Amf over 
  $(T,L)$} is a $T\times Q$-labeled tree $(T_r,r)$ such that 
$\varepsilon\in T_r$, $r(\varepsilon)=(\varepsilon,q_0)$, and for all 
$y\in T_r$ with $r(y)=(x,q)$ and $\delta(q,V(x))=\theta$, there is an 
assignment $v$ of truth values to the transition atoms in $\theta$
such that $v$ satisfies $\theta$ and:
    \begin{itemize}

      \item if $v(q')=1$, then $r(y')=(x,q')$ for some successor 
        $y'$ of $y$ in $T_r$;

      \item if $v(\Diamond^- q')=1$, then $x \neq \varepsilon$ and 
        $r(y')=(x\cdot -1,q')$ for some successor $y'$ of $y$ in $T_r$;

      \item if $v(\Box^- q')=1$, then $x=\varepsilon$ or 
        $r(y')=(x\cdot -1,q')$ for some successor $y'$ of $y$ in $T_r$;

      \item if $v(\Diamond_n q')=1$, then there are 
        pairwise different $i_1,\ldots,i_n$ such that, for each $j$,
        there is some successor $y'$ of $y$ in $T_r$ with 
        $r(y')=(x\cdot i_j,q')$;

      \item if $v(\Box_n q')=1$, then for all but 
        $n$ successors $x'$ of $x$, there is a 
        successor $y'$ of $y$ in $T_r$ with 
        $r(y')=(x',q')$. 


\end{itemize}
Let $\gamma=i_0i_1\cdots$ be an infinite path in $T_r$ and denote, for 
all $j\geq 0$, with $q_j$ the state such that $r(i_j)=(x,q_j)$. The 
path $\gamma$ is {\em accepting} if the largest number $m$ such that 
$\Omega(q_j)=m$ for infinitely many $j$ is even.  A run $(T_r,r)$ is 
accepting, if all infinite paths in $T_r$ are accepting. 
\Amf accepts a tree if \Amf has an accepting run over it. 

\subsection{Proof of Theorem~\ref{thm:nonemptiness}}
\label{appx:emptiness}


The proof is by reduction to the emptiness
problem of standard two-way alternating tree automata on trees of some
fixed outdegree~\cite{Vardi98}. We need to introduce strategy trees
similar to~\cite[Section 4]{Vardi98}. A {\em strategy tree for $\Amf$}
is a tree $(T,\tau)$ where $\tau$ labels every node in $T$ with a
subset $\tau(x)\subseteq 2^{Q\times \mathbb{N}\cup\{-1\}\times Q}$,
that is, with a graph with nodes from $Q$ and edges labeled with
natural numbers or $-1$. Intuitively, $(q,i,p)\in\tau(x)$ expresses
that, if we reached node $x$ in state $q$, then we should send a copy
of the automaton in state $p$ to $x\cdot i$.  For each label $\zeta$,
we define $\mn{state}(\zeta)=\{q\mid (q,i,q')\in \zeta\}$, that is,
the set of sources in the graph $\zeta$. A strategy tree is \emph{on
an input tree $(T',L)$} if $T=T'$, $q_0\in
\mn{state}(\tau(\varepsilon))$, and for every $x\in T$, the following
conditions are satisfied: 

\begin{enumerate}[(i)]

  \item if $(q,i,p)\in \tau(x)$, then $x\cdot i\in T$; 

  \item if $(q,i,p)\in \tau(x)$, then $p\in \mn{state}(\tau(x\cdot
    i))$;

  \item if $q\in \mn{state}(\tau(x))$, then the truth assignment
    $v_{q,x}$ defined below satisfies $\delta(q,L(x))$:
    \begin{itemize}

      \item $v_{q,x}(p)=1$ iff $(q,0,p)\in \tau(x)$;

      \item $v_{q,x}(\Diamond^- p) = 1$ iff $(q,-1,p)\in
	\tau(x)$;

      \item $v_{q,x}(\Box^- p) = 1$ iff $x=\varepsilon$ or $(q,-1,p)\in
	\tau(x)$;

      \item $v_{q,x}(\Diamond_n p) = 1$ iff $(q,i,p)\in \tau(x)$ for
	$n$ pairwise distinct $i\geq 1$;

      \item $v(\Box_n p) = 1$ iff for all but at most $n$ values
	$i\geq 1$ with $x\cdot i\in T$, we have $(q,i,p)\in \tau(x)$.

    \end{itemize}

\end{enumerate}
A {\em path $\beta$} in a strategy tree $(T,\tau)$ is a sequence
$\beta=(u_1,q_1)(u_2,q_2)\cdots$ of pairs from $T\times Q$ such that
for all $i>0$, there is some $c_i$ such that $(q_i,c_i,q_{i+1})\in
\tau(u_i)$ and $u_{i+1}=u_i\cdot c_i$. Thus, $\beta$ is obtained by
moves prescribed in the strategy tree.  We say that $\beta$ is
\emph{accepting} if the largest number $m$ such that $\Omega(q_i)=m$, for
infinitely many $i$, is even. A strategy tree $(T,\tau)$ is \emph{accepting}
if all infinite paths in $(T,\tau)$ are accepting.
\begin{lemma}
  A \ata accepts an input tree iff there is an accepting strategy tree
  on the input tree.
\end{lemma}
\begin{proof}
  The ``$\Leftarrow$'' direction is immediate: just read off an accepting run
  from the accepting strategy tree.

  For the ``$\Rightarrow$'' direction, we observe that acceptance of an
  input tree can be defined in terms of a parity game between Player 1
  (trying to show that the input is accepted) and Player 2 (trying to
  challenge that). The initial configuration is $(\varepsilon,q_0)$
  and Player 1 begins. Consider a configuration $(x,q)$. Player 1
  chooses a satisfying truth assignment $v$ of $\delta(q,L(x))$.
  Player~2 chooses an atom $\alpha$ with $v_{q,x}(\alpha)=1$ and
  determines the next configuration as follows:
  \begin{itemize}

    \item if $\alpha=p$, then the next configuration is $(x,p)$;

    \item if $\alpha=\Diamond^- p$, then the next configuration
      is $(x\cdot -1,p)$ unless $x=\varepsilon$; in this case,
      Player~1 loses immediately;

    \item if $\alpha=\Box^- p$, then the next configuration is $(x\cdot
      -1,p)$ unless $x=\varepsilon$; in this case, Player~2 loses
      immediately;

    \item if $\alpha=\Diamond_n p$, then Player~1 selects pairwise
      distinct $i_1,\ldots,i_n$ with $x\cdot i_j\in T$, for all $j$
      (and loses if she cannot); Player~2 then chooses some $i_j$ and
      the next configuration is $(x\cdot i_j,p)$;

    \item if $\alpha = \Box_n p$, then Player~1 selects $n$ values
      $i_1,\ldots,i_n$; Player~2 then chooses some
      $\ell\notin\{i_1,\ldots,i_n\}$ such that $x\cdot \ell\in T$ (and
      loses if he cannot) and the next configuration is $(x\cdot
      \ell,p)$.
      
  \end{itemize}
  Player 1 wins an infinite play $(x_0,q_0)(x_1,q_1)\cdots$ if the
  largest number $m$ such that $\Omega(q_i)=m$, for infinitely many
  $i$, is even. It is not difficult to see that Player 1 has a winning
  strategy on an input tree iff \Amf accepts the input tree.

  Observe now that the defined game is a parity game and thus Player 1
  has a winning strategy iff she is has a {\em memoryless} winning
  strategy~\cite{EmersonJ91}. It remains to observe that a memoryless
  winning strategy is nothing else than an accepting strategy tree.
\end{proof}
\begin{lemma} \label{lem:bounded-outdegree}
  If $L(\Amf)\neq\emptyset$, then there is some $(T,L)\in L(\Amf)$
  such that $T$ has outdegree at most $n\cdot C$, where $n$ is the
  number of states in $\Amf$ and $C$ is the largest number in (some
  transition $\Diamond_m p$ or $\Box_m p$ in) $\delta$.
\end{lemma}
\begin{proof}
Let $(T,L) \in L(\Amf)$ and $\tau$ an accepting strategy tree on
$T$, and let $C$ be the largest number appearing in $\delta$.  We
inductively construct a tree $(T',L')$ with $T'\subseteq T$ and $L'$
the restriction of $L$ to $T'$ and an accepting strategy tree $\tau'$
on $(T',L')$. For the induction base, we start with $T' =
\{\varepsilon\}$ and $\tau'$ the empty mapping.  For the inductive
step, assume that $\tau'(x)$ is still undefined for some $x\in T'$,
and proceed as follows:
\begin{enumerate}

  \item For every $(q,i,p)\in \tau(x)$ with $i\in\{-1,0\}$, add
    $(q,i,p)\in \tau'(x)$;

  \item for every $p\in Q$, define $N_p = \{i\geq 1\mid
    (q,i,p)\in\tau(x), x\cdot i\in T\}$ and let $N_p'\subseteq N_p$ be
    a subset of $N_p$ with precisely $\min(C,|N_p|)$ elements. Then:

    \begin{enumerate}

      \item for all $i\in N_p'$, add $x\cdot i\in T'$;

      \item for all $(q,i,p)\in \tau(x)$ with $i\in N_p'$, add
	$(q,i,p)\in \tau'(x)$;

      \item for all $q\in \mn{state}(x)$ and $i\in N_p'$, add
	$(q,i,p)\in \tau'(x)$.

    \end{enumerate}

\end{enumerate}
By Step~2 above, $T'$ has outdegree bounded by $|Q|\cdot C$. It
remains to show that $\tau'$ is an accepting strategy tree on~$T'$.
Observe first that, by construction,
$q_0\in\mn{state}(\tau'(\varepsilon))$.

We verify Conditions~(i)--(iii) of a strategy tree being \emph{on an input
tree}. Condition~(i) follows directly from the construction.  For~(ii),
assume that $(q,i,p)\in\tau'(x)$. By construction, there is some $q'$
with $(q',i,p)\in\tau(x)$, and, by Condition~(ii)
$p\in\mn{state}(\tau(x\cdot i))$. Hence, there is some
$(p,j,p')\in\mn{state}(\tau(x\cdot i))$. By construction, there is
also some $(p,j',p')\in\mn{state}(\tau'(x\cdot i)$, thus
$p\in\mn{state}(x\cdot i)$. For Condition~(iii), take any $x\in T'$
and $q\in \mn{state}(\tau'(x))$. As $q\in\mn{state}(\tau(x))$, we know
that the truth assignment $v_{q,x}$ defined for $\tau$ in
Condition~(iii) satisfies $\delta(q,L(x))$. We show that for all
transitions $\alpha$ with $v_{q,x}(\alpha)=1$, we also have
$v'_{q,x}(\alpha)=1$, where $v_{q,x}'$ is the truth assignment defined
for $\tau'$. By Step~1 of the construction, this is true for all
$\alpha$ of the shape $p$, $\Diamond^- p$, and $\Box^-p$. Let now be
$\alpha=\Diamond_k p$, that is, there are $k$ pairwise distinct $i\geq
1$ such that $(q,i,p)\in\tau(x)$. By the choice of $C$, we have
$|N_P'|\geq k$. By Step~2(c), we know that there are $k$ pairwise
distinct $i$ such that $(q,i,p)\in\tau'(x)$, hence
$v_{q,x}'(\alpha)=1$.  Consider now $\alpha=\Box_k p$, that is, for
all but at most $k$ values $i\geq 1$ with $x\cdot i\in T$, we have
$(q,i,p)\in\tau(x)$.  By Step~2(b), this remains true for $\tau'$,
hence $v_{q,x}'(\alpha)=1$.

We finally argue that $\tau'$ is also accepting.  Let
$\beta=(u_1,q_1)(u_2,q_2)\cdots$ be an infinite path in $(T',\tau')$.
We construct an infinite path
$\beta'=(u_1',q_1)(u_2',q_2)(u_3',q_3)\cdots$ in $(T,\tau)$ as
follows:
\begin{itemize}

  \item[--] $u_1'=u_1$;

  \item[--] Let $u_{i+1}=u_i\cdot \ell$ for some $\ell$ with
    $(q_i,\ell,q_{i+1})\in\tau'(x)$. If $\ell\in\{0,1\}$, we have
    $(q_i,\ell,q_{i+1})\in\tau(x)$, by Step~1. We set
    $u'_{i+1}=u_i'\cdot \ell$.  If $\ell\geq 0$ then, by Step~2(c),
    there is some $\ell'$ with $(q_i,\ell',q_{i+1})\in \tau(x)$ and
    $x\cdot\ell'\in T'$. Set $u_{i+1}'=u_i'\cdot\ell'$.

\end{itemize}
Since every infinite path in $(T,\tau)$ is accepting, so is $\beta'$,
and thus $\beta$.
\end{proof}
We now reduce the emptiness problem of 2ATA$_c$ to the
emptiness of alternating automata running on trees of fixed
outdegree~\cite{Vardi98}, which we recall here. A tree $T$ is
\emph{$k$-ary} if every node has exactly $k$. A \emph{two-way
  alternating tree automaton over $k$-ary trees (2ATA$^k$)} that are
$\Theta$-labeled is a tuple $\Amc=(Q, \Theta, q_0, \delta, \Omega)$
where $Q$ is a finite set of \emph{states}, $\Theta$ is the
\emph{input alphabet}, $q_0 \in Q$ is an \emph{initial state},
$\delta$ is the \emph{transition function}, and $\Omega:Q\to\Nbbm$ is
a \emph{priority function}. The transition function maps a state $q$
and some input letter $\theta$ to a \emph{transition condition}
$\delta(q,\theta)$, which is a positive Boolean formula over the truth
constants $\mn{true}$, $\mn{false}$, and transitions of the form
$(i,q)\in [k]\times Q$ where $[k]=\{-1,0,\ldots,k\}$.  A \emph{run} of
$\Amc$ on a $\Theta$-labeled tree $(T,L)$ is a $T \times Q$-labeled
tree $(T_r, r)$ such that
\begin{enumerate}

  \item $r(\varepsilon)=(\varepsilon,q_0)$; 

  \item for all $x \in T_r$ with $r(x) =(w,q)$ and $\delta(q,\tau(w))=
    \vp$, there is a (possibly empty) set $\mathcal S = \{(m_1,q_1), \ldots,
    (m_n,q_n)\} \subseteq [k] \times Q$ such that $\mathcal S$
    satisfies~$\vp$ and for $1 \leq i \leq n$, we have $x\cdot i
    \in T_r$, $w \cdot m_i$ is defined,
    and $\tau_r(x \cdot i)= (w\cdot m_i, q_i)$.  

\end{enumerate}
Accepting runs and accepted trees are defined as for \ata{}s. It is
well-known that emptiness for 2ATA$^k$s can be checked in exponential
time~\cite{Vardi98}, more precisely: 

\begin{theorem} \label{thm:vardi-emptiness}
  The emptiness problem for 2ATA$^k$s can be solved in time single
  exponential in the number of states and the maximal priority, and
  polynomial in all other inputs.
\end{theorem}

We are now ready to prove Theorem~\ref{thm:nonemptiness}.

\nonemptiness*
 
\noindent\begin{proof}
  Let $\Amf=(Q,\Theta,q_0,\delta,\Omega)$ be an \ata with $n$ states
  and $C$ the largest number in~$\delta$. We devise a
  2ATA$^k$ $\Amf'=(Q',\Theta',q_0',\delta',\Omega)$ with $k=n\cdot C$,
  that is, the bound from Lemma~\ref{lem:bounded-outdegree}, such that
  $L(\Amf)$ is empty iff $L(\Amf')$ is empty. Set
  $Q'=Q\cup\{q_0',q_1,q_r,q_\bot\}$ and
  $\Theta'=(\Theta\cup\{d_\bot\})\times\{0,1\}$. The extended alphabet
  and the extra states are used to simulate transitions of the form
  $\Box^-p$ (using the second component in the alphabet and state
  $q_r$) and to allow for input trees of outdegree less than $k$
  (using the additional symbol $d_\bot)$.

  We obtain $\delta'$ from $\delta$ by replacing $q$ with $(0,q)$,
  $\Diamond^- q$ with $(-1,q)$ and $\Box^-q$ with
  $(0,q_r)\vee(-1,q)$. Moreover, we replace

  \begin{itemize}

    \item[--] $\Diamond_n q$ with $\bigvee_{X\in \binom{\{1,\ldots
      k\}}{n}} \bigwedge_{i\in X} (i,q)$; 

    \item[--] $\Box_n q$ with $\bigvee_{X\in \binom{\{1,\ldots k\}}
    {n}} \bigwedge_{i\in \{1,\ldots,N\}\setminus X} (i,q)$,

\end{itemize}
where, as usual, $\binom{M}{m}$ denotes the set of all $m$-elementary
subsets of a set $M$. The transition function $\delta'$ for $\Amf'$ is
$\delta$ extended with the following transitions:
%
%
%
%
%
 %
  $$ \begin{array}[h]{rcll}
    \delta'(q_0',(\theta,b)) &=&  \left\{\begin{array}{ll} \mn{false}
      & \text{if $b=0$ or $\theta=d_\bot$,} \\
      q_0\wedge \bigwedge_{i=1}^k (i,q_1) & \text{otherwise,}
    \end{array}\right.\\[4mm]
    \delta'(q_1,(\theta,b)) &=&  \left\{\begin{array}{ll}
      \bigwedge_{i=1}^k (i,q_1) & \text{if $b=0$ and $\theta\neq
      d_\bot$}, \\
      \mn{true} & \text{if $\theta=d_\bot$,} \\
      \mn{false} & \text{otherwise},
    \end{array}\right.\\[4mm]
    \delta'(q_r,(\theta,b)) &=&  \left\{\begin{array}{ll} \mn{true} &
      \text{if $b=1$}, \\
    \mn{false}& \text{otherwise},
    \end{array}\right.  \\
    \delta(q,(d_\bot,b)) &=& \mn{false} \hspace{2cm}\text{for all
    $q\in Q$}.
  \end{array}$$
By Lemma~\ref{lem:bounded-outdegree}, we have that $L(\Amf)$ is empty
iff $L(\Amf')$ is empty. Note that this is not a polynomial time
reduction to emptiness of 2ATA$^k$s due to the translations for
$\Diamond_n$ and $\Box_n$.  It should be clear, however, that these
translations can be \emph{represented} in polynomial size, e.g., in the
way they are given above. It remains to observe that, for the
algorithm from Theorem~\ref{thm:vardi-emptiness} to work in the
desired time, it suffices
to be able to perform the following task in polynomial time:
\begin{itemize}

  \item Given a set $\Smc\subseteq [k]\times Q$, a state $q\in Q$, and
    $a\in \Theta$, check whether $\Smc$ satisfies $\delta(q,a)$.

\end{itemize}
This is certainly possible for the mentioned representation, which
concludes the proof of the theorem. 
\end{proof}

\subsection{Proof of Lemma~\ref{lem:derivationtrees}}
\label{appx:derivlem}

\lemderivationtrees*

\noindent
\begin{proof}\ We start with Point~1.  The ``$\Leftarrow$'' direction is
  straightforward to prove by induction on the depth of the derivation
  tree. We thus concentrate on the ``$\Rightarrow$'' direction. Thus assume
  that $\Tmc,\Amc \models A_0(a_0)$.  We construct a sequence of
  ABoxes by `chasing' \Amc with the TBox \Tmc, that is, by
  exhaustively and fairly applying the following rules:
  \begin{enumerate}

    \item If $A_1(a),\dots,A_k(a) \in \Amc$ and $\Tmc \models A_1
      \sqcap \cdots \sqcap A_k \sqsubseteq A$, then add $A(a)$ to
      \Amc;

    \item If $r(a,b), A(a) \in \Amc$ and 
      $A\sqsubseteq \forall r.B\in \Tmc$, then add $B(b)$ to \Amc;

    \item If $r(a,b),B(a) \in \Amc$ and $B \sqsubseteq \exists r . A,
      \mn{func}(r) \in \Tmc$, then add $A(b)$ to \Amc.
      
  \end{enumerate}
  Let $\Amc=\Amc_0,\Amc_1,\dots,\Amc_k$ be the emerging (finite!)
  sequence of ABoxes.  We establish the following central claim.
  
  \medskip\noindent \textbf{Claim.} There is a model \Imc of $\Amc_k$ and \Tmc such
  that for all $a \in \mn{ind}(\Amc)$ and concept names $A$,
  $a \in A^\Imc$ implies $A(a) \in \Amc_k$.

  \medskip \noindent To construct the claimed $\Imc$, we start with
  $\Imc_0$ defined as follows. 
  \begin{align*}
    \Delta^{\Imc_0} &= \mn{ind}(\Amc) \\
    A^{\Imc_0} & = \{a\mid A(a)\in \Amc_k\}\\
    r^{\Imc_0} & = \{(a,b)\mid s(a,b)\in\Amc, \Tmc\models s\sqsubseteq
    r\}
  \end{align*}
  Denote with $\Amc_a$ the set $\{A(a)\in \Amc_k\mid A\in\NC\}$.  Now,
  obtain $\Imc$ from $\Imc_0$ by proceeding as follows for every $a\in
  \mn{ind}(\Amc)$, $A\sqsubseteq \exists r.B\in \Tmc$, and $a\in
  A^\Imc$. Let $\Imc'$ be the subtree rooted at a $\rho$-successor $b$
  of $a$ in $\Imc_{\Tmc,\Amc_a}$ with $b\in B^\Imc$ and $r\in \rho$.
  If $\mn{func}(r)\notin \Tmc$ or there is no $a'$ with $r(a,a')\in
  \Amc_k$, then add $\Imc'$ as a $\rho$-successor of $a$ in $\Imc_0$.
  Clearly, $\Imc_0$ and thus $\Imc$ is a model of $\Amc_k$.  Based on
  the construction of $\Amc_k$ and \Imc, it can also be verified that
  $\Imc$ is a model of \Tmc. This finishes the proof of the claim.

  \medskip We are now ready to finish the proof of Point~1. From $\Tmc,\Amc
  \models A_0(a_0)$ and the claim, we obtain $A_0(a_0) \in
  \Amc_k$. Exploiting that the three rules used to construct
  $\Amc_0,\dots,\Amc_k$ are in one-to-one correspondence with
  Conditions~(i) to~(iii) from the definition of derivation trees, it
  is easy to show by induction on $i$ that for all $i \in
  \{1,\dots,k\}$ and all $A(a) \in \Amc_i$, there is a derivation tree
  for $A(a)$ in \Amc w.r.t.\ \Tmc. In particular, this gives the
  desired derivation tree for $A_0(a_0)$.

\medskip

Now for Point~2. The ``$\Rightarrow$'' direction is immediate, so we
concentrate on ``$\Leftarrow$''. Thus assume that Conditions~(a) and~(b) are
satisfied. By~(a) and
Lemma~\ref{lem:universal_model}~\eqref{it:univ_model_is_a_model}, the universal model
$\Imc_a=\Imc_{\Tmc,\Amc_a}$ is a model of $\Tmc$ and $\Amc_a$. In
particular, $\Imc_a$ is weakly tree-shaped with root $a$ and
satisfies, for all concept names $B\in \NC$:
\begin{itemize}

  \item[$(\ast)$] $a\in B^{\Imc_a}$ iff $\Tmc,\Amc_a \models B(a)$.

\end{itemize}
 
  
  We construct a model $\Imc$ of \Amc and \Tmc as follows. Start with
  $\Imc_0$ by taking 
  \begin{align*}
    \Delta^{\Imc_0} &= \mn{ind}(\Amc) \\
    A^{\Imc_0} & = \{a\mid \text{$A(a)$ has a derivation tree in \Amc
    w.r.t.~\Tmc}\}\\
    r^{\Imc_0} & = \{(a,b)\mid s(a,b)\in\Amc, \Tmc\models s\sqsubseteq
    r\}
  \end{align*}
  and then adding, for every $a\in \mn{ind}(\Amc)$ and
  $\rho$-successor $b$ of $a$ in $\Imc_a$ such that for all roles
  $r\in\rho$ with $\mn{func}(r)\in\Tmc$ we have
  $r(a,a')\notin\Amc$ for all $a'\in\mn{ind}(\Amc)$, the subinterpretation of $\Imc_a$
  rooted at $b$ as a $\rho$-successor of $a$.
    
  Clearly, $\Imc_0$ and thus $\Imc$ is a model of \Amc. 
  Based on~$(\ast)$ and the assumptions, it is also straightforward to show
  that $\Imc$ is a model \Tmc. By construction, all elements from
  $\Delta^\Imc\setminus\mn{ind}(\Amc)$ already satisfy~\Tmc. 
  For $a\in\mn{ind}(\Amc)$, we distinguish cases on the shape of
  concept inclusions: 
  \begin{itemize}

    \item If $\top \sqsubseteq A\in \Tmc$, then trivially $A(a)$ has a
      derivation tree in \Amc w.r.t.~\Tmc, hence $a\in A^{\Imc}$.

    \item Suppose $A\sqsubseteq \bot\in \Tmc$ and $a\in A^\Imc$.
      By definition of $\Imc$, $A(a)$ has a derivation tree in
      $\Amc$ w.r.t.~\Tmc, and thus $\Tmc,\Amc\models A(a)$. But then the
      ABox $\Amc_a$ is not consistent with $\Tmc$, a contradiction.

    \item If $A_1\sqcap A_2\sqsubseteq B\in \Tmc$ and $a\in (A_1\sqcap
      A_2)^\Imc$, then by construction, $A_1(a)$ and $A_2(a)$ have
      derivation trees in \Amc w.r.t.~\Tmc. Thus also $B(a)$ has a
      derivation tree, and hence $a\in B^\Imc$.

    \item If $A\sqsubseteq \exists r.B\in \Tmc$ and $a\in A^\Imc$,
      then $A(a)$ has a derivation tree in \Amc w.r.t.~\Tmc. Thus,
      $\Tmc,\Amc\models A(a)$, and $a$ has a $\rho$-successor
      $b$ in $\Imc_a$ with $r\in \rho$ and $b\in B^\Imc$. By
      construction, $a$ has the $\rho$-successor $b$ in \Imc unless
      $\mn{func}(r)\in \Tmc$ and $r(a,a')\in \Amc$ for some
      $a'\in\mn{ind}(\Amc)$. But then, $B(a')$ has a derivation tree
      and thus $a'\in B^\Imc$ and $a\in (\exists r.B)^\Imc$.

    \item If $A\sqsubseteq \forall r.B\in \Tmc$ and $a\in (\exists
      r^-.A)^\Imc$, there is some $b$ with $b\in A^\Imc$ and $(b,a)\in
      r^\Imc$.  
      In case $b\in \mn{\Amc}$, $B(a)$ has a derivation tree
      by definition.  If $b\notin\mn{ind}(\Amc)$, we get $a\in
      B^{\Imc_a}$.  By~$(\ast)$, we know that $\Tmc,\Amc_a\models
      B(a)$ and thus there $B(a)$ has a derivation tree. In both
      cases, the construction yields $a\in B^\Imc$. 

%

    \item It remains to note that $\Imc$ satisfies the functionality
      assertions from \Tmc because of assumption~(b).

  \end{itemize}
\end{proof}

\bibliographystyle{theapa}

\end{document}